\documentclass{article}

\usepackage{microtype}
\usepackage{graphicx}
\usepackage{subfigure}
\usepackage{booktabs} 

\usepackage[OT1]{fontenc}
\usepackage[colorlinks,
linkcolor=red,
anchorcolor=blue,
citecolor=blue
]{hyperref}
\usepackage{smile}
\usepackage{enumitem}
\usepackage{fullpage}
\usepackage{tikz}

\def\cR{r}
\def\cT{{\mathbb T}}
\usepackage{authblk}
\newcommand*\samethanks[1][\value{footnote}]{\footnotemark[#1]}

\title{Single-Timescale Actor-Critic Provably Finds Globally Optimal Policy}
\author{Zuyue Fu\thanks{Department of Industrial Engineering and Management Sciences, Northwestern University}~~~~ Zhuoran Yang\thanks{Department of Operations Research and Financial Engineering, Princeton University}~~~~ Zhaoran Wang\samethanks[1]}

\date{\vspace{-5ex}}

\begin{document}

\maketitle


\begin{abstract}
We study the global convergence and global  optimality of actor-critic, one of the most popular families of reinforcement learning algorithms.      
While
 most existing works on actor-critic employ bi-level or two-timescale  updates, 
 we focus on the more practical single-timescale setting, where the actor and  critic are updated simultaneously. 
Specifically, in each iteration,  the critic update is obtained by  applying the Bellman evaluation operator only once  while the actor is updated in the policy gradient direction computed using the critic.
Moreover, we consider two function approximation settings where both the actor and critic  are represented by linear or deep neural networks. 
For both cases, 
we prove that the actor sequence  converges to a globally optimal policy at a  sublinear $O(K^{-1/2})$ rate, where $K$ is the number of iterations.  
To the best of our knowledge, we establish the rate of convergence and global optimality of
single-timescale  actor-critic  with linear  function approximation   for the first time.
Moreover, under the broader  scope of policy optimization with nonlinear function approximation, we  prove that actor-critic with deep neural network finds the globally optimal policy at a sublinear rate for the first time. 
\end{abstract}

\section{Introduction}

In reinforcement learning (RL) \citep{sutton1998introduction}, the agent aims to make sequential decisions that maximize the expected total reward through interacting with the environment and learning from the experiences, where the environment is modeled as a Markov Decision Process (MDP)  \citep{puterman2014markov}. 
To learn a policy that achieves the highest possible total reward in expectation, 
 the actor-critic method \citep{konda2000actor} is among the  most commonly used algorithms. 
  In actor-critic, 
  the actor refers to the policy and the critic corresponds to the value function that characterizes the performance of the actor. 
  This method directly optimizes the expected total return over the policy class by iteratively improving the actor, where the update direction is determined by the critic. 
In particular, recently,   actor-critic   combined with  deep neural networks  \citep{lecun2015deep}  achieves tremendous empirical successes in solving large-scale RL tasks, such as the game of Go \citep{silver2017mastering}, StarCraft \citep{vinyals2019alphastar}, Dota \citep{OpenAI_dota}, Rubik's cube \citep{agostinelli2019solving, akkaya2019solving},  and autonomous driving \citep{sallab2017deep}. See \cite{li2017deep} for a detailed survey of the recent developments of deep reinforcement learning.

Despite these great empirical successes of actor-critic, there is  still an evident chasm between  theory and practice.
Specifically, to establish convergence guarantees for actor-critic,   most existing works either focus   on the bi-level setting or the two-timescale setting, which are seldom adopted in practice.  
In particular,
under the bi-level setting \citep{yang2019global, wang2019neural, agarwal2019optimality, fu2019actor, liu2019neural,
abbasi2019politex, abbasi2019exploration, cai2019provably, hao2020provably, mei2020global, bhandari2020note},  the actor is updated only after the critic solves the policy evaluation sub-problem completely, which is equivalent to applying the Bellman evaluation operator to the previous critic for infinite times.
Consequently, 
actor-critic under the bi-level  setting is a double-loop iterative algorithm where the inner loop is allocated for solving the  policy evaluation  sub-problem of the critic. 
In terms of theoretical analysis, 
such a double-loop structure decouples the analysis for the actor and critic. 
For the actor, the problem is essentially reduced to analyzing the convergence of  a variant of the policy gradient method  \citep{sutton2000policy, kakade2002natural} where the error of the  gradient estimate depends on the policy evaluation error of the critic.  
Besides,  under the two-timescale setting \citep{borkar1997actor, konda2000actor, xu2020non, wu2020finite, hong2020two}, the actor 
and the critic are updated simultaneously, but with disparate stepsizes. 
More concretely, 
the stepsize of the actor is set to be  much smaller than that of the critic, with the ratio between these stepsizes converging to  zero. 
In an asymptotic sense, such a separation between stepsizes 
ensures that 
the critic completely solves its policy evaluation sub-problem asymptotically. 
In other words, 
such a two-timescale scheme 
 results in a separation between actor and critic in an asymptotic sense,  which leads to asymptotically unbiased  
 policy gradient  estimates.  
 In sum, in terms of convergence analysis,  the existing theory of actor-critic hinges on 
 decoupling the analysis for critic and actor, which is ensured via focusing on the bi-level or two-timescale settings. 

However,   
  most practical implementations of actor-critic are under the single-timescale setting \citep{peters2008natural, schulman2015trust, mnih2016asynchronous, schulman2017proximal, haarnoja2018soft},   where the actor and critic are simultaneously updated, and particularly, the actor is updated without the critic
  reaching an approximate solution to the policy evaluation  sub-problem. 
Meanwhile, in comparison with the  two-timescale setting, 
the actor is equipped with a much larger stepsize in the 
the single-timescale setting such that the asymptotic separation between the analysis of  actor and critic is no longer valid. 

Furthermore, when it comes to function approximation, most existing works only analyze the convergence of   actor-critic   with either linear function approximation \citep{xu2020non, wu2020finite,hong2020two}, or shallow-neural-network parameterization \citep{wang2019neural, liu2019neural}.
In contrast,  practically  used actor-critic methods  such as 
asynchronous advantage actor-critic  \citep{mnih2016asynchronous} and soft actor-critic  \citep{haarnoja2018soft} oftentimes represent both the  actor and critic   using deep neural networks. 
  
  Thus, the following question is left open:
  \begin{center}
 \emph{ Does single-timescale actor-critic provably find  a  globally optimal policy under the function approximation setting, especially when deep neural networks are employed?}
  \end{center}

To answer such a question, we make the first attempt to 
investigate the convergence and global optimality of single-timescale 
 actor-critic with linear and neural network function approximation.
 In particular, 
 we
 focus on  the family of energy-based policies
 and  
 aim to find the optimal policy within this class.
 Here we 
 represent both the energy function 
 and the critic as linear or deep neural network functions. 
In our actor-critic algorithm, 
the actor update follows
proximal policy optimization (PPO)   \citep{schulman2017proximal} 
and the critic update is obtained by 
applying the Bellman evaluation operator only once to the current critic iterate. 
As a result, the actor is updated before the critic solves the policy evaluation sub-problem. 
Such a coupled updating structure persists even when  the number  of iterations goes to infinity, which implies that the update direction of the actor is always biased compared with the policy gradient direction. 
This brings an additional challenge that is absent 
  in the bi-level and the two-timescale settings, where the actor and critic are decoupled asymptotically.
  
  To tackle such a challenge, 
  our analysis 
     captures 
     the joint effect of 
     actor and critic updates on
     the objective function, 
 dubbed as 
    the   ``double contraction'' phenomenon, 
    which 
    plays a pivotal role 
    for the success of single-timescale actor-critic. 
    Specifically, 
    thanks to the discount factor of the MDP, the Bellman evaluation operator is contractive,
    which implies that, after each  update, the critic makes  noticeable 
    progress by moving towards  the value function associated with the current actor. 
     As a result, although we use a biased estimate 
    of the policy gradient, 
    thanks to the contraction brought by the discount factor, 
    the accumulative effect of the biases is controlled.
    Such a phenomenon 
    enables us to characterize the progress of each iteration of joint   actor and critic update, and thus yields the   convergence to the globally optimal policy.
    In particular, 
    for both the linear and neural settings, we prove that,  single-timescale actor-critic finds a  $O( K^{-1/2})$-globally optimal policy after $K$  iterations. 
    To the 
      best of our knowledge, we seem to  establish  the first theoretical guarantee  of  global convergence and global  optimality for  actor-critic   with function approximation in the single-timescale setting. 
      Moreover, 
      under the broader scope of policy optimization with nonlinear function approximation, our work  seems  to 
      prove convergence and  optimality guarantees for actor-critic with deep neural network  for the first time.

\vskip5pt
\noindent\textbf{Contribution.}
Our contribution is two-fold.  First, in the single-timescale setting with linear function approximation,  we prove that, after $K$ iterations of actor and critic updates, actor-critic returns a policy that is at most $O(K^{-1/2})$ inferior to the globally optimal policy. 
Second, 
when both the actor and critic are represented by deep neural networks, we prove a similar $O(K^{-1/2})$ rate of  convergence to the globally optimal policy when the architecture of the neural networks are properly chosen. 


\vskip5pt
\noindent\textbf{Related Work.}
Our work extends the line of works on the convergence of actor-critic under the  function approximation setting. In particular,   actor-critic   is first introduced in \cite{sutton2000policy, konda2000actor}. Later, \cite{kakade2002natural, peters2008reinforcement} propose the natural actor-critic method which updates the policy   via the  natural gradient   \citep{amari1998natural} direction. The convergence of (natural) actor-critic  with linear function approximation are studied in \cite{bhatnagar2008incremental, bhatnagar2009natural, bhatnagar2010actor, castro2010convergent, maei2018convergent}. 
 However,  these works only characterize the  asymptotic convergence of actor-critic   and their proofs  all resort to tools from stochastic approximation  via ordinary differential equations \citep{borkar2008stochastic}. As a result,  these works  only show that actor-critic with linear function approximation converges to the set of stable equilibria   of  a set of ordinary differential equations.  Recently, \cite{zhang2019global} propose a variant of actor-critic where 
   Monte-Carlo
   sampling is used to ensure the critic and the policy gradient estimates are unbiased. 
    Although 
    they incorporate nonlinear function 
    approximation in the actor, 
     they only establish  finite-time convergence result   to a stationary point  
    of the expected total reward.
    Moreover, due to having an inner loop for solving the policy evaluation sub-problem, they focus on the bi-level setting. 
    Moreover, under the two-timescale setting, 
 \cite{wu2020finite, xu2020non} 
 show that actor-critic with linear function approximation 
 finds an $\varepsilon$-stationary point with  
   $\tilde O(\varepsilon^{-5/2})$ samples, where $\varepsilon$ measures the squared norm of the policy gradient. 
 All of these results establish the convergence of actor-critic, without characterizing the optimality of the policy obtained by actor-critic. 

In terms of the global optimality of actor-critic, 
 \cite{fazel2018global, malik2018derivative, tu2018gap, yang2019global, bu2019lqr, fu2019actor} show that policy gradient and bi-level actor-critic methods converge  to the globally optimal policies
under the   linear-quadratic setting, 
  where the state transitions follow  a linear dynamical system and the reward function is  quadratic.
 For general MDPs,   
 \cite{bhandari2019global} recently prove the global optimality of vanilla policy gradient 
 under the assumption that the families of policies and value functions are both convex. 
 In addition, our work is also related to \cite{liu2019neural} and \cite{wang2019neural}, where they establish  the global optimality of proximal policy optimization  and (natural) actor-critic,  respectively, where both the actor and critic are parameterized by two-layer neural networks. 
Our work is also related to  \cite{agarwal2019optimality,abbasi2019politex, abbasi2019exploration, cai2019provably, hao2020provably, mei2020global, bhandari2020note}, which focus on characterizing the optimality of natural policy gradient  in tabular and/or linear settings. 
 However,   these   aforementioned works all  focus on  bi-level actor-critic, where the actor is updated only after the critic solves the policy evaluation sub-problem to an approximate optimum.   
 Besides, these works consider linear   or two-layer neural network   function approximations whereas we focus on the setting with deep neural networks.
Furthermore, under the  two-timescale setting, 
   \cite{xu2020non, hong2020two} prove that linear  actor-critic  requires a sample complexity of $\tilde O(\varepsilon^{-4})$ for obtaining an  $\varepsilon$-globally optimal policy.  
   In comparison, our $O(K^{-1/2})$ convergence for single-timescale actor-critic can be translated into a similar $\tilde O(\varepsilon^{-4})$ sample complexity directly. Moreover, when reusing the data, our result leads to an improved $\tilde O(\varepsilon^{-2})$ sample complexity. 
In addition, our work is also related to  \cite{geist2019theory}, 
 which proposes a variant of policy iteration algorithm with Bregman divergence regularization. 
 Without considering an explicit form of function approximation, their algorithm is 
 shown to converge to the globally optimal policy
at a similar $O(K^{-1/2})$   rate, where $K$ is the number of policy updates. 
In contrast, our method is   single-timescale actor-critic with linear or deep neural network function approximation, which enjoys both global convergence and global optimality.  Meanwhile, our proof is based on a finite-sample analysis, which involves dealing with the algorithmic  errors that track the performance of actor and critic updates as well as the statistical error due to having finite data. 

Our work is also related to the literature on deep neural networks. Previous works \citep{daniely2017sgd, jacot2018neural, wu2018sgd, allen2018learning, allen2018convergence, du2018gradient, zou2018stochastic, chizat2018note, jacot2018neural, li2018learning, cao2019bounds, cao2019generalization, arora2019fine, lee2019wide, gao2019convergence} analyze the computational and statistical rates of supervised learning methods with overparameterized neural networks. 
In contrast, our work employs overparameterized deep neural networks in actor-critic for solving RL tasks, which is significantly more challenging than supervised learning  due to the interplay between the actor and the critic.

\vskip5pt
\noindent\textbf{Roadmap.}
In \S\ref{sec:background}, we introduce the background of discounted MDP and actor-critic method. Then in \S\ref{sec:ac},  we introduce the two actor-critic methods, where the actors and critics are parameterized using linear functions and deep neural networks. The theoretical results are presented in \S\ref{sec:theory}. 

\vskip5pt
\noindent\textbf{Notation.}  
We denote by $[n]$ the set $\{1, 2, \ldots, n\}$.  
For any measure $\nu$ and $1\leq p\leq \infty$, we denote by $\|f\|_{\nu,p} = (\int_\cX |f(x)|^p\ud \nu)^{1/p}$ and $\|f\|_{p} = (\int_\cX |f(x)|^p\ud \mu)^{1/p}$, where $\mu$ is the Lebesgue measure.


\section{Background}\label{sec:background}

In this section, we introduce the background on discounted Markov decision processes (MDPs) and actor-critic methods. 

\subsection{Discounted MDP} 
A discounted MDP is defined by a tuple $(\cS, \cA, P, \zeta, \cR, \gamma)$. Here $\cS$ and $\cA$ are the state and action spaces, respectively, $P\colon \cS\times \cS\times \cA\to [0,1]$ is the Markov transition kernel, $\zeta\colon \cS\to [0,1]$ is the initial state distribution, $\cR\colon \cS\times\cA\to \RR$ is the deterministic reward function, and $\gamma\in[0,1)$ is the discount factor. A policy $\pi(a\given s)$ measures the probability of taking the action $a$ at the state $s$. We focus on a family of parameterized policies defined as follows,
\#\label{eq:def-pi-set}
\Pi = \{\pi_\theta(\cdot\given s) \in \cP(\cA) \colon s\in \cS  \}, 
\#
where $\cP(\cA)$ is the probability simplex on the action space $\cA$ and $\theta$ is the parameter of the policy $\pi_\theta$. 
For any state-action pair $(s,a)\in \cS\times \cA$, we define the action-value function as follows,
\#\label{eq:def-qfunc}
 Q^\pi(s, a) = (1 - \gamma)\cdot\EE_{\pi} \Bigl[\sum^\infty_{t = 0}\gamma^t \cdot \cR(s_t, a_t)  \Biggiven s_0 = s, a_0 = a\Bigr],
\#
where $s_{t+1}\sim P(\cdot\given s_t, a_t)$ and $a_{t+1}\sim\pi(\cdot\given s_{t+1})$ for any $t\geq 0$.  We use $\EE_\pi[\cdot]$ to denote that the actions follow the policy $\pi$, which further affect the transition of the states. 
We aim to find an optimal policy $\pi^*$ such that $Q^{\pi^*}(s,a) \geq Q^{\pi}(s,a)$ for any policy $\pi$ and state-action pair $(s,a) \in \cS\times \cA$. That is to say, such an optimal policy $\pi^*$ attains a higher expected total reward than any other policy $\pi$, regardless of the initial state-action pair $(s, a)$. For notational convenience, we denote by $Q^*(s,a) = Q^{\pi^*}(s,a)$ for any $(s,a)\in\cS\times\cA$ hereafter. 

Meanwhile, we denote by $\nu_\pi(s)$ and $\rho_\pi(s, a) = \nu_\pi(s)\cdot \pi(a\given s)$ the  stationary state distribution and stationary state-action  distribution of the policy $\pi$, respectively, for any $(s,a)\in\cS\times\cA$.  
Correspondingly, we denote by $\nu^*(s)$ and $\rho^*(s,a)$ the  stationary state distribution and  stationary state-action distribution of the optimal policy $\pi^*$, respectively, for any $(s,a)\in \cS\times\cA$. For ease of presentation, given any functions $g_1\colon \cS\to \RR$ and $g_2\colon \cS\times \cA\to \RR$, we define two operators $\PP$ and $\PP^{\pi}$ as follows, 
\#\label{eq:def-ops}
& [\PP g_1](s,a) = \EE[ g_1(s_1) \given s_0 = s, a_0 = a] ,\quad  [\PP^{\pi} g_2] (s,a) = \EE_{\pi}[g_2(s_1, a_1)\given s_0 = s, a_0 = a],
\#
where $s_1\sim P(\cdot\given s_0, a_0)$ and $a_1\sim \pi(\cdot\given s_1)$. Intuitively, given the current state-action pair $(s_0, a_0)$, the operator $\PP$ pushes the agent to its next state $s_1$ following the Markov transition kernel $P(\cdot\given s_0, a_0)$, while the operator $\PP^\pi$ pushes the agent to its next state-action pair $(s_1, a_1)$ following the Markov transition kernel $P(\cdot\given s_0, a_0)$ and policy $\pi(\cdot\given s_1)$. These operators also relate to the Bellman evaluation operator $\cT^\pi$, which is defined for any function $g\colon \cS\times \cA\to \RR$ as follows, 
\#\label{eq:def-bellman-op}
\cT^\pi g = (1-\gamma)\cdot \cR + \gamma \cdot \PP^\pi g.
\#
The Bellman evaluation operator $\cT^\pi$ is used to characterize the actor-critic method in the following section. By the definition in \eqref{eq:def-qfunc}, it is straightforward to verify that the action-value function $Q^\pi$ is the fixed point of the Bellman evaluation operator $\cT^\pi$ defined in \eqref{eq:def-bellman-op}, that is, $Q^\pi = \cT^\pi Q^\pi$ for any policy $\pi$.  For notational convenience, we let  $\PP^\ell$ denote the $\ell$-fold composition  
$
  \underbrace{\PP\PP \cdots \PP}_{\ell}. 
$
Such notation is also adopted for other linear operators such as 
 $\PP^\pi$ and $\cT^\pi$.

\subsection{Actor-Critic Method}  

To obtain an optimal policy $\pi^*$, the actor-critic method \citep{konda2000actor} aims to maximize the expected total reward as a function of the policy, which is equivalent to solving the following maximization problem, 
\#\label{eq:bu-opt}
\max_{\pi \in \Pi } J(\pi) = \EE_{s\sim \zeta, a\sim \pi(\cdot\given s)} \bigl[Q^\pi(s,a) \bigr ], 
\#
where $\zeta$ is the initial state distribution, $Q^\pi$ is the action-value function defined in \eqref{eq:def-qfunc}, and the family of parameterized polices $\Pi$ is defined in \eqref{eq:def-pi-set}. The actor-critic method solves the maximization problem in \eqref{eq:bu-opt} via first-order optimization using an estimator of the policy gradient $\nabla_\theta J(\pi)$. Here $\theta$ is the parameter of the policy $\pi$. In detail, by the policy gradient theorem \citep{sutton2000policy}, we have
\#\label{eq:bu-pg}
\nabla_\theta J(\pi) = \EE_{(s,a)\sim \varrho_\pi} \bigl[ Q^\pi(s,a)\cdot \nabla_\theta \log \pi(a\given s) \bigr]. 
\#
Here $\varrho_\pi$ is the state-action visitation measure of the policy $\pi$, which is defined as $\varrho_\pi(s,a) = (1-\gamma)\cdot \sum_{t = 0}^\infty \gamma^t \cdot \Pr[ s_t = s, a_t = a ]$. Based on the closed form of the policy gradient in \eqref{eq:bu-pg}, the actor-critic method consists of the following two parts: (i) the  critic update, where a policy evaluation algorithm is invoked to estimate the action-value function $Q^\pi$, e.g., by applying the Bellman evaluation operator $\cT^\pi$ to the current estimator of $Q^\pi$, and (ii) the actor update, where a policy improvement algorithm, e.g., the policy gradient method, is invoked using the updated estimator of $Q^\pi$. 

In this paper, we consider the following variant of the actor-critic method, 
\#\label{eq:ac-algo}
& \pi_{k+1} \gets \argmax_{\pi\in\Pi} \EE_{\nu_{\pi_k}} \bigl [ \la Q_{k}(s, \cdot), \pi(\cdot \given s) \ra - \beta \cdot \kl \bigl(\pi(\cdot \given s) \,\|\, \pi_{k}(\cdot \given s) \bigr) \bigr], \notag \\
& Q_{k+1}(s,a) \gets \EE_{\pi_{k+1}}\bigl[(1-\gamma) \cdot \cR(s_0, a_0)+ \gamma \cdot Q_{k}(s_1, a_1)\biggiven s_0 = s, a_0 = a \bigr],
\#
for any $(s,a) \in \cS\times \cA$, where $s_1\sim P(\cdot\given s_0,a_0)$, $a_1\sim \pi_{k+1}(\cdot\given s_1)$, and we write $\EE_{\nu_{\pi_k}}[\cdot] = \EE_{s\sim \nu_{\pi_k}}[\cdot]$ for notational convenience. Here $\Pi$ is defined in \eqref{eq:def-pi-set} and $\kl(\pi(\cdot \given s) \,\|\, \pi_{k}(\cdot \given s))$ is the Kullback-Leibler (KL) divergence between $\pi(\cdot \given s)$ and $\pi_k(\cdot \given s)$, which is defined for any $s\in \cS$ as follows, 
\$
\kl\bigl(\pi(\cdot \given s) \,\|\, \pi_{k}(\cdot \given s)\bigr) = \sum_{a\in\cA}\log \Bigl(\frac{\pi(a \given s)}{\pi_k(a \given s)}\Bigr) \cdot \pi(a \given s). 
\$
 In \eqref{eq:ac-algo}, the actor update uses the proximal policy optimization (PPO) method \citep{schulman2017proximal}, while the critic update applies the Bellman evaluation operator $\cT^{\pi_{k+1}}$ defined in \eqref{eq:def-bellman-op}  to $Q_{k}$ only once, which is the current estimator of the action-value function. 
{Furthermore, we remark that the updates in \eqref{eq:ac-algo} provide a general framework in the following two aspects. 
First, 
the critic update can be extended to letting  $Q_{k+1} \leftarrow (\cT^{\pi_{k+1}})^\tau Q_{k}$  for any fixed  $\tau \geq 1$, which  corresponds to updating the value function via  $\tau$-step rollouts following $\pi_{k+1}$. Here we only focus on the case with $\tau = 1$ for simplicity. Our theory can be easily modified for any fixed $\tau$. 
Moreover, the KL divergence used in the actor step can also be replaced by other Bregman divergences between
probability distributions over $\cA$. 
Second, the actor and critic updates in \eqref{eq:ac-algo} 
is a general template  that admits both on- and off-policy evaluation methods and various function approximators  in  the actor and critic. 
In the next section, we present an incarnation of   \eqref{eq:ac-algo} with on-policy sampling and linear  and neural network  function approximation.
}

{Furthermore, for} analyzing the  actor-critic method, most existing works \citep{yang2019global, wang2019neural, agarwal2019optimality, fu2019actor, liu2019neural} rely on (approximately) obtaining $Q^{\pi_{k+1}}$ at each iteration, which is equivalent to applying the Bellman evaluation operator $\cT^{\pi_{k+1}}$ infinite times to   $Q_k$. This is usually achieved by minimizing the mean-squared Bellman error $\|Q - \cT^{\pi_{k+1}} Q \|_{\rho_{\pi_{k+1}},2}^2$  using stochastic semi-gradient descent, e.g., as in the temporal-difference method \citep{sutton1988learning},  to update the critic for sufficiently many iterations.   The unique global minimizer of  the mean-squared Bellman error gives the action-value function  $Q^{\pi_{k+1}}$, which is used in the  actor update.  
Meanwhile,  the two-timescale setting is also considered in  existing works \citep{borkar1997actor, konda2000actor,   xu2019two, xu2020non, wu2020finite, hong2020two},  which require  the actor to be updated more slowly than the critic in an asymptotic sense.  Such a requirement is usually satisfied by  forcing  the ratio between the stepsizes of the actor and critic updates to go to zero asymptotically.


In comparison with the setting with bi-level updates, we consider the single-timescale actor and critic updates in \eqref{eq:ac-algo}, where the critic involves only one step of  update, that is, applying the Bellman evaluation operator $\cT^\pi$ to   $Q_k$ only once. 
Meanwhile, in comparison with the two-timescale setting, where the actor and critic are updated simultaneously but with the ratio between their stepsizes asymptotically going to zero, the single-timescale setting is able to  achieve a faster rate of convergence by allowing the actor to be updated with a  larger stepsize, while updating the critic simultaneously. 
In particular, such a single-timescale setting  better captures a broader range of practical algorithms \citep{peters2008natural, schulman2015trust, mnih2016asynchronous, schulman2017proximal, haarnoja2018soft}, where the stepsize of the actor is not asymptotically zero. 
In \S\ref{sec:ac}, we discuss the implementation of  the updates in \eqref{eq:ac-algo} for different schemes of function approximation.  
In \S\ref{sec:theory}, we compare the rates of convergence between the two-timescale and single-timescale settings.




\section{Algorithms} \label{sec:ac}

We consider two settings, where the actor and critic are parameterized using linear functions and deep neural networks, respectively. 
We consider the energy-based policy $\pi_\theta(a\given s) \propto \exp( \tau^{-1} f_\theta(s,a) )$, where the energy function $f_\theta(s,a)$ is parameterized with the parameter $\theta$. 
Also, for the (estimated) action-value function, we consider the parameterization $Q_{\omega}(s, a)$ for any $(s,a)\in\cS\times\cA$, where $\omega$ is the parameter. 
For such parameterizations of the actor and critic, the updates in \eqref{eq:ac-algo} have the following forms.

\vskip5pt
\noindent\textbf{Actor Update.} 
The following proposition gives the closed form of $\pi_{k+1}$ in \eqref{eq:ac-algo}. 

\begin{proposition}\label{prop:energy-based}
Let $\pi_{\theta_{k}}(a\given s)\propto \exp(\tau_k^{-1} f_{\theta_k}(s,a))$ be an energy-based policy and 
\$
\tilde \pi_{{k+1}} = \argmax_{\pi} \EE_{\nu_k} \bigl[  \la Q_{\omega_k}(s, \cdot), \pi(\cdot \given s)   \ra  - \beta \cdot \kl\bigl(\pi(\cdot \given s) \,\|\, \pi_{\theta_{k}}(\cdot \given s)\bigr) \bigr].
\$ 
Then $\tilde \pi_{k+1}$ has the following closed form, 
\$
\tilde \pi_{{k+1}}(a\given s) \propto \exp \bigl( \beta^{-1} Q_{\omega_k}(s,a) + \tau_k^{-1} f_{\theta_k}(s,a)  \bigr),
\$
for any $(s,a)\in\cS\times\cA$, where $\nu_k = \nu_{\pi_{\theta_k}}$ is the stationary state distribution of $\pi_{\theta_k}$. 
\end{proposition}
\begin{proof}
See \S\ref{prf:prop:energy-based} for a detailed proof. 
\end{proof}

Motivated by Proposition \ref{prop:energy-based}, to implement the actor update in \eqref{eq:ac-algo}, we update the actor parameter $\theta$ by solving the following minimization problem, 
\#\label{eq:actor-mse}
\theta_{k+1}\gets \argmin_{\theta} \EE_{\rho_k}\bigl[ & \bigl( f_\theta(s,a) - \tau_{k+1} \cdot \bigl( \beta^{-1} Q_{\omega_k}(s,a)  + \tau_k^{-1} f_{\theta_k}(s,a)  \bigr)  \bigr)^2 \bigr],
\#
where $\rho_k = \rho_{\pi_{\theta_k}}$ is the stationary state-action  distribution of $\pi_{\theta_k}$.  

\vskip5pt
\noindent\textbf{Critic Update.} To implement the critic update in \eqref{eq:ac-algo}, we update the critic parameter $\omega$ by solving the following minimization problem,
\#\label{eq:critic-update-p}
\omega_{k+1}\gets \argmin_{\omega } \EE_{\rho_{k+1}} \bigl[ & \bigl([Q_{\omega} - (1-\gamma) \cdot \cR - \gamma \cdot \PP^{\pi_{\theta_{k+1}}} Q_{\omega_k} ](s,a) \bigr)^2 \bigr],
\#
where $\rho_{k+1} =   \rho_{\pi_{\theta_{k+1}}}$ is the stationary state-action  distribution of $\pi_{\theta_{k+1}}$ and the operator $\PP^\pi$ is defined in \eqref{eq:def-ops}.   

\subsection{Linear Function Approximation} \label{sec:lac}
In this section, we consider linear function approximation.  
More specifically, we parameterize the action-value function using $Q_\omega(s,a) = \omega^\top \varphi(s,a)$ and the energy function of the energy-based policy $\pi_\theta$ using $f_\theta(s,a) = \theta^\top \varphi(s,a)$. Here $\varphi(s,a) \in \RR^{d}$ is the feature vector, where $d>0$ is the dimension. 
Without loss of generality, we assume that $\|\varphi(s,a)\|_2 \leq 1$ for any $(s,a) \in \cS\times \cA$, which can be achieved by normalization.

\vskip5pt
\noindent {\bf Actor Update.}
The minimization problem in \eqref{eq:actor-mse} admits the following closed-form solution, 
\#\label{eq:theta-update}
\theta_{k+1} = \tau_{k+1}\cdot  (\beta^{-1}\omega_k + \tau_k^{-1} \theta_k ), 
\#
which corresponds to a step of the natural policy gradient method \citep{kakade2002natural}.

\vskip5pt
\noindent {\bf Critic Update.}  
The minimization problem in \eqref{eq:critic-update-p} admits the following closed-form solution, 
\#\label{eq:pop-omega}
\tilde \omega_{k+1} =  \bigl( \EE_{\rho_{k+1}} [ \varphi(s,a) \varphi(s,a)^\top  ] \bigr)^{-1} \cdot \EE_{\rho_{k+1}} \bigl [   [ (1-\gamma) \cdot \cR + \gamma \cdot \PP^{\pi_{\theta_{k+1}}} Q_{\omega_k}    ](s,a) \cdot \varphi(s,a) \bigr]. 
\#
Since the closed-form solution $\tilde \omega_{k+1}$ in \eqref{eq:pop-omega} involves the expectation over the stationary state-action distribution $\rho_{k+1}$ of $\pi_{\theta_{k+1}}$, we use data to approximate such an expectation.  
More specifically, we sample $\{(s_{\ell,1}, a_{\ell,1})\}_{\ell \in [N]}$ and $\{(s_{\ell,2}, a_{\ell,2}, r_{\ell,2}, s_{\ell,2}', a_{\ell,2}')\}_{\ell \in [N]}$ such that $(s_{\ell,1}, a_{\ell,1}) \sim \rho_{k+1}$, $(s_{\ell,2}, a_{\ell,2}) \sim \rho_{k+1}$, $r_{\ell,2} = \cR(s_{\ell,2}, a_{\ell,2})$, $s_{\ell,2}' \sim P(\cdot \given s_{\ell,2}, a_{\ell,2} )$, and $a_{\ell,2}' \sim \pi_{\theta_{k+1}}(\cdot \given s_{\ell,2}')$, where $N$ is the sample size. 
We approximate $\tilde \omega_{k+1}$ using $\omega_{k+1}$, which is defined as follows, 
\#\label{eq:omega-update}
\omega_{k+1} =  \Gamma_R  \Bigl\{\Bigl( & \sum_{\ell = 1}^{N} \varphi(s_{\ell,1}, a_{\ell,1}) \varphi(s_{\ell,1}, a_{\ell,1})^\top  \Bigr)^{-1} \\
& \cdot \sum_{\ell = 1}^{N} \bigl( (1-\gamma) \cdot r_{\ell,2} + \gamma \cdot  Q_{\omega_k}(s_{\ell,2}', a_{\ell,2}')   \bigr) \cdot \varphi(s_{\ell,2}, a_{\ell,2})  \Bigr\}.  \notag
\#
Here $\Gamma_R$ is the projection operator,  which projects the parameter onto the centered ball with radius $R$ in $\RR^d$.   Such a projection operator stabilizes the algorithm
\citep{konda2000actor, bhatnagar2009natural}.    
It is worth mentioning that one may also view the update in \eqref{eq:omega-update} as one step of the least-squares temporal difference method \citep{bradtke1996linear}, {which can be modified for the  off-policy setting   \citep{antos2007value, yu2010convergence, liu2018breaking,nachum2019dualdice, xie2019towards, zhang2020gendice, uehara2019minimax, nachum2020reinforcement}.  
Such a modification allows the data points in \eqref{eq:omega-update} to be reused in   the  subsequent iterations, which further improves the sample complexity. 
Specifically, let $\rho_{\textrm{bhv}} \in \cP(\cS\times \cA) $ be the stationary state-action distribution induced by a behavioral policy $\pi_{\textrm{bhv}}$. 
We replace the actor and critic updates in \eqref{eq:actor-mse} and \eqref{eq:critic-update-p} by 
\#
\theta_{k+1}  &\gets \argmin_{\theta} \EE_{\rho_{\textrm{bhv}} }\bigl[     \bigl( f_\theta(s,a) - \tau_{k+1} \cdot \bigl( \beta^{-1} Q_{\omega_k}(s,a)  + \tau_k^{-1} f_{\theta_k}(s,a)  \bigr)  \bigr)^2 \bigr], \label{eq:offpolicy_actor}\\
\omega_{k+1} &\gets \argmin_{\omega }  \EE_{\rho_{\textrm{bhv} } } \bigl[   \bigl([Q_{\omega} - (1-\gamma) \cdot \cR - \gamma \cdot \PP^{\pi_{\theta_{k+1}}} Q_{\omega_k} ](s,a) \bigr)^2 \bigr],\label{eq:offpolicy_critic}
\# 
respectively.
With linear function approximation, the actor update in \eqref{eq:offpolicy_actor} is reduced to \eqref{eq:theta-update},   while the  critic update in \eqref{eq:offpolicy_critic} admits a closed form solution 
\$
\tilde \omega_{k+1} =  \bigl( \EE_{\rho_{\textrm{bhv} }} [ \varphi(s,a) \varphi(s,a)^\top  ] \bigr)^{-1} \cdot \EE_{\rho_{\textrm{bhv} }} \bigl [   [ (1-\gamma) \cdot \cR + \gamma \cdot \PP^{\pi_{\theta_{k+1}}} Q_{\omega_k}    ](s,a) \cdot \varphi(s,a) \bigr], 
\$
which can be well approximated using state-action pairs drawn from $\rho_{\textrm{bhv}}$. 
See \S\ref{sec:theory} for a detailed discussion.}

Finally, by assembling the updates in \eqref{eq:theta-update} and \eqref{eq:omega-update}, we present the linear actor-critic method in Algorithm \ref{algo:l-ac}, which is deferred to \S\ref{sec:aux-algo-paper} of the appendix. 

\subsection{Deep Neural Network Approximation}\label{sec:nac}

In this section, we consider deep neural network approximation.  We first formally define deep neural networks.  Then we introduce the actor-critic method under such a parameterization. 


A deep neural network (DNN) $u_\theta(x)$ with the input $x\in\RR^d$, depth $H$, and width $m$ is defined as
\#\label{eq:def-nn-form}
& x^{(0)} = x, \quad x^{(h)} = \frac{1}{\sqrt{m}}\cdot \sigma( W_h^\top x^{(h-1)} ),   {\rm ~for~}h \in [H], \quad u_\theta(x) = b^\top x^{(H)}. 
\#
Here $\sigma\colon \RR^m \to \RR^m$ is the rectified linear unit (ReLU) activation function, which is define as $\sigma(y) = ( \max\{0, y_1\}, \ldots, \max\{0, y_m\} )^\top$ for any $y = (y_1, \ldots, y_m)^\top\in \RR^m$. Also, we have $b\in\{-1,1\}^{m}$, $W_1\in\RR^{d\times m}$, and $W_h\in\RR^{m\times m}$ for $2\leq h\leq H$. Meanwhile, we denote the parameter of the DNN $u_\theta$ as $\theta = (\vec(W_1)^\top, \ldots, \vec(W_H)^\top)^\top \in\RR^{m_{\rm all}}$ with $m_{\rm all} = md + (H-1)m^2$. We call $\{W_{h}\}_{h\in[H]}$ the weight matrices of $\theta$. Without loss of generality, we normalize the input $x$ such that $\|x\|_2 = 1$. 

We initialize the DNN such that each entry of $W_h$ follows the standard Gaussian distribution $\mathcal N(0,1)$ for any $h\in[H]$, while each entry of $b$ follows the uniform distribution ${\rm Unif}(\{-1,1\})$. Without loss of generality, we fix $b$ during training and only optimize $\{W_h\}_{h\in[H]}$. We denote the initialization of the parameter $\theta$ as $\theta_0 = (\vec(W_1^0)^\top, \ldots, \vec(W_H^0)^\top)^\top$.  Meanwhile, we restrict $\theta$ within the ball $\cB(\theta_0, R)$ during training,  which is defined as follows, 
\#\label{eq:def-proj-set}
& \cB(\theta_0, R) = \bigl\{\theta\in\RR^{m_{\rm all}} \colon \|W_{h} - W_{h}^0\|_\F \leq R,  ~{\rm for~}h\in[H] \bigr\}. 
\#
Here $\{W_{h}\}_{h\in[H]}$ and $\{W_{h}^0\}_{h\in[H]}$ are the weight matrices of  $\theta$ and $\theta_0$, respectively. By \eqref{eq:def-proj-set}, we have $\|\theta-\theta_0\|_2\leq R\sqrt{H}$ for any $\theta\in\cB(\theta_0, R)$. Now, we define the family of DNNs as
\#\label{eq:def-dnn-class}
\cU(m, H, R) = \bigl\{u_\theta \colon \theta\in \cB(\theta_0, R)\bigr\}, 
\#
where  $u_\theta$ is a DNN with depth $H$ and width $m$. 

We parameterize the action-value function using $Q_\omega(s,a)\in\cU(m_{\rm c}, H_{\rm c}, R_{\rm c})$  and the energy function of the energy-based policy $\pi_\theta$ using $f_\theta(s,a)\in\cU(m_{\rm a}, H_{\rm a}, R_{\rm a})$.  Here $\cU(m_{\rm c}, H_{\rm c}, R_{\rm c})$ and $\cU(m_{\rm a}, H_{\rm a}, R_{\rm a})$ are the families of DNNs defined in \eqref{eq:def-dnn-class}. 
Hereafter we assume that the energy function $f_\theta$ and the action-value function $Q_\omega$ share the same architecture and initialization, i.e., $m_{\rm a} = m_{\rm c}$, $H_{\rm a} = H_{\rm c}$, $R_{\rm a} = R_{\rm c}$, and $\theta_0 = \omega_0$. Such shared architecture and initialization of the DNNs ensure that the parameterizations of the policy   and the action-value function  are approximately compatible. See \cite{sutton2000policy, konda2000actor, kakade2002natural, peters2008natural, wang2019neural} for a detailed discussion.

\vskip5pt
\noindent{\bf Actor Update.}
To solve \eqref{eq:actor-mse}, we   use   projected stochastic gradient descent, whose $n$-th iteration has the following form, 
\$
\theta(n+1)\gets & \Gamma_{\cB(\theta_0, R_{\rm a})}\bigl( \theta(n) - \alpha\cdot \bigl( f_{\theta(n)}(s,a) - \tau_{k+1}\cdot \bigl( \beta^{-1} Q_{\omega_k}(s,a) + \tau_k^{-1} f_{\theta_k}(s,a)  \bigr)  \bigr) \cdot \nabla_\theta f_{\theta(n)}(s,a)\bigr). 
\$
Here $\Gamma_{\cB(\theta_0, R_{\rm a})}$ is the projection operator, which projects the parameter onto the ball   $\cB(\theta_0, R_{\rm a})$ defined in \eqref{eq:def-proj-set}. The state-action pair $(s,a)$ is sampled from the stationary state-action distribution $\rho_k$. We summarize the update in Algorithm \ref{algo:actor}, which is deferred to \S\ref{sec:aux-algo-paper} of the appendix.

\vskip5pt
\noindent{\bf Critic Update.}
 To solve \eqref{eq:critic-update-p}, we   apply  projected stochastic gradient descent. More specifically, at the $n$-th iteration of   projected stochastic gradient descent, we sample a tuple $(s, a, r, s', a')$, where $(s, a)\sim \rho_{k+1}$, $r = \cR(s, a)$, $s'\sim P(\cdot\given s, a)$, and $a'\sim\pi_{\theta_{k+1}}(\cdot\given s')$. 
We define the residual at the $n$-th iteration as $\delta(n) = Q_{\omega(n)}(s, a) - (1-\gamma) \cdot r - \gamma \cdot Q_{\omega_k}(s', a')$. 
Then   the $n$-th iteration of projected stochastic gradient descent has the following form, 
\$
\omega(n+1)\gets \Gamma_{\cB(\omega_0, R_{\rm c})} \bigl(\omega(n) - \eta\cdot \delta(n)\cdot \nabla_\omega Q_{\omega(n)}(s, a) \bigr). 
\$
Here $\Gamma_{\cB(\omega_0, R_{\rm c})}$ is the projection operator, which projects the parameter onto the ball   $\cB(\omega_0, R_{\rm c})$ defined in \eqref{eq:def-proj-set}. 
We summarize the update in Algorithm \ref{algo:critic}, which is deferred to \S\ref{sec:aux-algo-paper} of the appendix. 

By assembling Algorithms \ref{algo:actor} and \ref{algo:critic}, we present the deep neural actor-critic method in Algorithm  \ref{algo:n-ac}, which is deferred to \S\ref{sec:aux-algo-paper} of the appendix. 

{Finally, we remark that the  off-policy  actor and critic updates given in 
\eqref{eq:offpolicy_actor} and \eqref{eq:offpolicy_critic}
can also incorporate deep neural network approximation with a slight modification, which enables data reuse in the algorithm.
}



\section{Theoretical Results}\label{sec:theory}

In this section, we upper bound the regret of the linear actor-critic method.   We defer the analysis of the deep neural  actor-critic method to \S\ref{sec:theory-nn} of the appendix. 
Hereafter we assume that  $|\cR(s,a)| \leq \cR_{\max}$ for any $(s,a) \in \cS\times \cA$, where $\cR_{\max}$ is a positive absolute constant. 
First, we impose the following assumptions.  Recall that $\rho^*$ is the stationary state-action distribution of $\pi^*$, while $\rho_k$ is the stationary state-action distribution of $\pi_{\theta_k}$.
{Moreover, let $\rho \in \cP(\cS\times \cA)$ be a state-action distribution with respect to which we aim to characterize the performance of the actor-critic algorithm. 
Specifically, after $K+1$ actor updates, we are interest in   upper 
bounding the following regret 
\#\label{eq:define_regret}
\EE \Bigl [ \sum_{k=0}^K \bigl ( \| Q^*  - Q^{\pi_{\theta_{k+1}}} \|_{\rho, 1} \bigr ) \Bigr ]  = \EE \Bigl [\sum_{k = 0}^K \bigl( Q^*(s,a) - Q^{\pi_{\theta_{k+1}}}(s,a) \bigr) \Bigr]  ,
\#
where the expectation is taken with respect to $\{\theta_{k}\}_{k\in [K+1]}$ and $(s,a) \sim \rho$. 
Here we allow $\rho$ to be any fixed distribution for generality, which might be different from $\rho^*$. 
}

\begin{assumption} [Concentrability Coefficient] \label{assum:cc-lin}
The following statements hold. 
\begin{itemize}
\item[(i)\namedlabel{assum:1-lin}{(i)}]  There exists a positive absolute constant $\phi^*$ such that $\phi_k^* \leq \phi^*$  for any $k\geq 1$, where $\phi_k^* = \|{\ud\rho^*}/{\ud\rho_{k}}\|_{\rho_k, 2}$. 

\item [(ii)\namedlabel{assum:2-lin}{(ii)}]  {We assume that
for any  $k \geq 1$ and a sequence of policies $\{\pi_i\}_{i\geq 1}$, the $k$-step future-state-action distribution $\rho \PP^{\pi_1} \cdots \PP^{\pi_k}$ is absolutely continuous with respect to $\rho^*$, where $\rho$ is the same as the one in \eqref{eq:define_regret}} Also, it holds for such $\rho$ that  
\$
C_{\rho, \rho^*} = (1-\gamma)^2 \sum_{k = 1}^\infty k^2\gamma^{k} \cdot c(k) < \infty, 
\$
where $c(k) = \sup_{\{\pi_i\}_{i\in[k]}} \| {\ud (\rho \PP^{\pi_1} \cdots \PP^{\pi_k})}/{\ud \rho^*}  \|_{\rho^*, \infty}$. 
\end{itemize}
\end{assumption}

In Assumption \ref{assum:cc-lin}, 
$C_{\rho, \rho^*}$ is known as the discounted-average concentrability coefficient of the future-state-action distributions. Similar assumptions are commonly imposed in the   literature \citep{szepesvari2005finite, munos2008finite, antos2008fitted, antos2008learning, scherrer2013performance, scherrer2015approximate, 
farahmand2016regularized,yang2019theoretical, geist2019theory, chen2019information}.

\begin{assumption}[Zero Approximation Error] \label{assum:error}
It holds for any $\omega, \theta \in \cB(0, R)$ that
\$
\inf_{\bar \omega\in \cB(0, R)} \EE_{\rho_{\pi_\theta}} \bigl[  \bigl( [  \cT^{\pi_\theta} Q_\omega - \bar \omega^\top \varphi    ](s,a) \bigr)^2  \bigr] = 0, 
\$
where $\cT^{\pi_\theta}$ is defined in \eqref{eq:def-bellman-op}. 
\end{assumption}

Assumption \ref{assum:error} states that the Bellman evaluation operator   maps a linear function to a linear function.  Such an assumption only aims to simplify the presentation of our results. If the approximation error is nonzero, we  only need to incorporate an additional bias term into the rate of convergence. 

\begin{assumption}[Well-Conditioned Feature]\label{assum:matrix}
The minimum singular value of the matrix $\EE_{\rho_{k}} [ \varphi(s,a) \varphi(s,a)^\top  ]$ is uniformly lower bounded by a positive absolute constant $\sigma^*$ for any $k \geq 1$. 
\end{assumption}

Assumption \ref{assum:matrix} ensures that the minimization problem in \eqref{eq:critic-update-p} admits a unique minimizer, which is used in the  critic update. 
 Similar assumptions are commonly imposed in the literature \citep{bhandari2018finite, zou2019finite}.

Under Assumptions \ref{assum:cc-lin}, \ref{assum:error}, and \ref{assum:matrix}, we upper bound the regret of Algorithm \ref{algo:l-ac} in the following theorem.

\begin{theorem}\label{thm:regret-lin}
We assume that Assumptions \ref{assum:cc-lin}, \ref{assum:error}, and \ref{assum:matrix} hold. 
Let $\rho$ be a state-action distribution satisfying \ref{assum:2-lin} of Assumption \ref{assum:cc-lin}. Also, for any confidence parameter $\delta \in (0, 1)$ and  sufficiently large number of iterations $K > 0$,  let $\beta = K^{1/2}$, $N = \Omega(  K   C_{\rho, \rho^*}^2 \cdot  (\phi^* / \sigma^*)^2 \cdot  \log^2 (KN/\delta) )$, and the sequence of policy parameters $\{\theta_k\}_{k\in[K+1]}$ be generated by Algorithm \ref{algo:l-ac}. It holds with probability at least $1 - \delta$ that
\#\label{eq:rhs-rev}
& \EE_\rho \Bigl [\sum_{k = 0}^K \bigl( Q^*(s,a) - Q^{\pi_{\theta_{k+1}}}(s,a) \bigr) \Bigr]  \leq \bigl( 2(1-\gamma)^{-3} \cdot \log |\cA|  + O (1)  \bigr)\cdot K^{1/2}, 
\#
where the expectation is taken with respect $(s,a) \sim \rho$. 
\end{theorem}
\begin{proof}
We sketch the proof in \S\ref{sec:proof}.  
See \S\ref{prf:thm:regret-lin} for a detailed proof. 
\end{proof}

Theorem \ref{thm:regret-lin} establishes an $O (K^{1/2})$ regret of Algorithm \ref{algo:l-ac},   where $K$ is the total number of iterations.  
Here $O(\cdot)$ omits terms involving $(1- \gamma)^{-1}$ and $\log |\cA|$. 
To better understand Theorem \ref{thm:regret-lin}, we consider the ideal setting, where we have access to the action-value function $Q^{\pi}$ of any policy $\pi$.  In such an ideal setting, the critic update is unnecessary.  However, the natural policy gradient method, which only uses the actor update, achieves the same $O(K^{1/2})$ regret \citep{liu2019neural, agarwal2019optimality, cai2019provably}.   In other words, in terms of the iteration complexity, Theorem \ref{thm:regret-lin} shows that in the single-timescale setting, using only one step of the critic update along with one step of the actor update is as efficient as the natural policy gradient method in the ideal setting. 

{
	Furthermore, by the regret bound in  \eqref{eq:rhs-rev},  to obtain an $\varepsilon$-globally optimal policy, 
	it suffices to set 
	$ K \asymp  (1-\gamma)^{-6} \cdot \varepsilon^{-2} \cdot \log ^2 |\cA| $ in Algorithm \ref{algo:l-ac} and 
output a  randomized  policy 
 that is drawn from   $\{\pi_{\theta_{k}} \}_{k=1}^{K+1}$ uniformly.  
	Plugging such a $K$ into 
	$N = \Omega(  K   C_{\rho, \rho^*}^2 (\phi^* / \sigma^*)^2 \cdot  \log^2 (KN/\delta) ) )$,
	we obtain that $N = \tilde O( \varepsilon^{-2})$, where $\tilde O(\cdot)$ omits the logarithmic terms. 
	Thus, to achieve an $\varepsilon$-globally optimal policy, the total sample complexity of Algorithm \ref{algo:l-ac} is $\tilde O(\varepsilon^{-4})$. 
	This matches the sample complexity results established in \cite{xu2020non, hong2020two} for two-timescale actor-critic methods. 
	Meanwhile, 
	notice that here the critic updates are on-policy and we  draw $N$ new data points in each critic update. 
As discussed in  \S\ref{sec:lac},
under the off-policy setting, the critic updates given in \eqref{eq:offpolicy_critic}
can be 
implemented using a fixed dataset sampled from $\rho_{\textrm{bhv}}$, 
the stationary state-action distribution induced by the behavioral policy. 
Under this scenario, 
the total number of data points used by the algorithm is equal to   $N$. 
Moreover, by imposing similar assumptions on $\rho_{\textrm{bhv}}$ as   in  \ref{assum:1-lin} of Assumption \ref{assum:cc-lin} and Assumption \ref{assum:matrix}, 
we can establish a similar $O( K^{1/2})$ regret as in \eqref{eq:rhs-rev} for the off-policy setting. 
As a result, with data reuse, to obtain an $\varepsilon$-globally optimal policy, the sample complexity 
of Algorithm \ref{algo:l-ac} is essentially $\tilde O( \varepsilon^{-2})$, which demonstrates the advantage of our  single-timescale  actor-critic method.
Besides, 
only focusing on the convergence to an $\varepsilon$-stationary point, 
 	\cite{wu2020finite, xu2020non} establish the sample complexity of $\tilde O(\varepsilon^{-5/2})$ for two-timescale actor-critic, 
where  
 	  $\varepsilon$ measures the squared Euclidean norm of the policy gradient. 
 	  In contrast, by adopting the natural policy gradient \citep{kakade2002natural} in actor updates, we achieve convergence to the globally optimal policy. 
	To the best of our knowledge, we establish the rate of convergence and global optimality of the actor-critic method with function approximation in the single-timescale setting for the first time. 

Furthermore, as we will show in Theorem \ref{thm:regret} of  \S\ref{sec:aux-algo-paper}, when both the actor and the critic are represented using overparameterized deep neural networks, we establish a similar $O( (1-\gamma)^{-3}\cdot \log |\cA| \cdot  K^{1/2})$ regret when the architecture of the actor and critic neural networks are properly chosen. 
To our best knowledge, this 
 seems the first  theoretical guarantee     for the  actor-critic method with deep neural network function approximation  in terms of the rate of convergence and  global optimality.
}

\section{Proof Sketch of Theorem \ref{thm:regret-lin}}\label{sec:proof}

In this section, we sketch the proof of Theorem \ref{thm:regret-lin}. Recall that $\rho$ is a state-action distribution satisfying \ref{assum:2-lin} of Assumption \ref{assum:cc-lin}.   
We first   upper bound   $\sum_{k = 0}^K (Q^*(s,a) - Q^{\pi_{\theta_{k+1}}}(s,a))$ for any   $(s,a)\in\cS\times \cA$ in part 1.  Then by further taking the expectation over $\rho$ in part 2,  we conclude the proof of Theorem \ref{thm:regret-lin}. See \S\ref{prf:thm:regret-lin} for a detailed proof.

\vskip5pt
\noindent\textbf{Part 1.} In the sequel, we     upper bound   $\sum_{k = 0}^K ( Q^*(s,a) - Q^{\pi_{\theta_{k+1}}}(s,a) )$ for any   $(s,a)\in\cS\times \cA$. We first decompose $Q^* - Q^{\pi_{\theta_{k+1}}}$ into the following three terms, 
\#\label{eq:Q-bound1-skc}
& \sum_{k=0}^K [Q^* - Q^{\pi_{\theta_{k+1}}}](s,a)   =  \sum_{k=0}^K \bigl[(I - \gamma \PP^{\pi^*})^{-1} ( A_{1,k} + A_{2,k} + A_{3,k}  ) \bigr](s,a), 
\#
the proof of which is deferred to \eqref{eq:Q-bound1-lin} and \eqref{eq:Q-bound2-lin} in \S\ref{prf:thm:regret-lin} of the appendix. 
Here the operator $\PP^{\pi^*}$ is defined in \eqref{eq:def-ops},  $(I - \gamma \PP^{\pi^*})^{-1} = \sum_{i = 0}^\infty (\gamma \PP^{\pi^*})^i$, and  $A_{1,k}$, $A_{2,k}$, and $A_{3,k}$ are defined as follows,
\#
& A_{1,k}(s,a) = [\gamma ( \PP^{\pi^*} - \PP^{\pi_{\theta_{k+1}}}) Q_{\omega_k }](s,a)  , \label{eq:def-as-skc-a1} \\
& A_{2,k}(s,a) = \bigl[\gamma \PP^{\pi^*} (Q^{\pi_{\theta_{k+1}}} - Q_{\omega_k }) \bigr](s,a),  \label{eq:def-as-skc-a2} \\
& A_{3,k}(s,a) = [ \cT^{\pi_{\theta_{k+1}}} Q_{\omega_{k}}  - Q^{\pi_{\theta_{k+1}}} ](s,a) \label{eq:def-as-skc-a3}.  
\#
To understand the intuition behind $A_{1,k}$, $A_{2,k}$, and $A_{3,k}$, we interpret them as follows. 

\vskip5pt\noindent{\bf Interpretation of $A_{1,k}$.}
As defined in \eqref{eq:def-as-skc-a1},  $A_{1,k}$ arises from the actor update and measures the  convergence of the policy $\pi_{\theta_{k+1}}$ towards a globally optimal policy $\pi^*$, which implies the convergence of $\PP^{\pi_{\theta_{k+1}}}$ towards $\PP^{\pi^*}$. 

\vskip5pt\noindent{\bf Interpretation of $A_{3,k}$.}
Note that by \eqref{eq:def-qfunc} and  \eqref{eq:def-bellman-op}, we have  $Q^{\pi_{\theta_{k+1}}} =  \cT^{\pi_{\theta_{k+1}}} Q^{\pi_{\theta_{k+1}}}$ and $\cT^{\pi_{\theta_{k+1}}}$ is a $\gamma$-contraction, which implies that applying the Bellman evaluation operator $\cT^{\pi_{\theta_{k+1}}}$ to any $Q$, e.g., $Q_{\omega_k}$, infinite times yields $Q^{\pi_{\theta_{k+1}}}$. 
As defined in \eqref{eq:def-as-skc-a3}, $A_{3,k}$ measures the error  of tracking the action-value function $Q^{\pi_{\theta_{k+1}}}$ of $\pi_{\theta_{k+1}}$ by applying the Bellman evaluation operator $\cT^{\pi_{\theta_{k+1}}}$ to $Q_{\omega_{k}}$ only once, which arises from the critic update.  Also, as $A_{3,k} = \cT^{\pi_{\theta_{k+1}}} ( Q_{\omega_k} - Q^{\pi_{\theta_{k+1}}} ) $, $A_{3,k}$   measures  the difference between $Q^{\pi_{\theta_k}}$, which is approximated by $Q_{\omega_k}$  as discussed subsequently, and $Q^{\pi_{\theta_{k+1}}}$.  Such a difference can also be viewed as the difference  between $\pi_{\theta_k}$ and $\pi_{\theta_{k+1}}$, which  arises from   the actor update.   Therefore,  the convergence of $A_{3,k}$ to zero implies the   contractions of not only the critic update 
but also  the actor update, which illustrates the ``double contraction'' phenomenon.  
We establish the convergence of $A_{3,k}$ to zero in \eqref{eq:up-bd-a-skc2} subsequently. 

\vskip5pt\noindent{\bf Interpretation of $A_{2,k}$.}
Assuming that  $A_{3,k-1}$ converges to zero, we have $  \cT^{\pi_{\theta_k}} Q_{\omega_{k-1}} \approx Q^{\pi_{\theta_k}}$.  Moreover, assuming that the number of data points $N$ is sufficiently large and ignoring the projection in \eqref{eq:omega-update}, 
we have $\cT^{\pi_{\theta_k}} Q_{\omega_{k-1}} = Q_{\tilde \omega_{k}} \approx Q_{\omega_{k}}$ as $\tilde \omega_{k}$ defined in \eqref{eq:pop-omega}  is an estimator of $\omega_k$. Hence, we have $Q^{\pi_{\theta_k}} \approx Q_{\omega_{k}}$.  Such an approximation  error is characterized  by $\epsilon_k^{\rm c}$ defined in \eqref{eq:error-Q-sketch} subsequently.  
Hence,  $A_{2,k}$ measures the difference between $\pi_{\theta_k}$ and $\pi_{\theta_{k+1}}$ through the difference between $Q^{\pi_{\theta_k}} \approx Q_{\omega_k}$ and $Q^{\pi_{\theta_{k+1}}}$, which relies on the convergence of $A_{3,k-1}$ to zero.

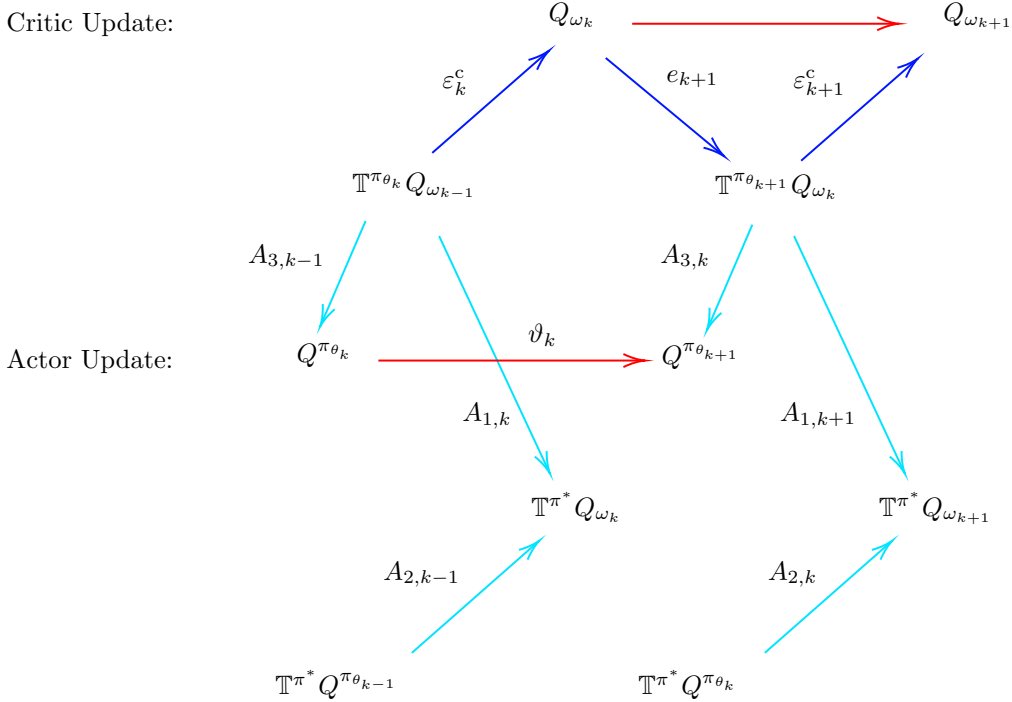
\begin{figure}[h]
\centering

\tikzset{every picture/.style={line width=0.75pt}} 

\begin{tikzpicture}[x=0.75pt,y=0.75pt,yscale=-1,xscale=1]

\draw [color={rgb, 255:red, 0; green, 24; blue, 255 }  ,draw opacity=1 ]   (216.17,105.19) -- (273.6,54.86) ;
\draw [shift={(275.1,53.54)}, rotate = 498.77] [color={rgb, 255:red, 0; green, 24; blue, 255 }  ,draw opacity=1 ][line width=0.75]    (10.93,-3.29) .. controls (6.95,-1.4) and (3.31,-0.3) .. (0,0) .. controls (3.31,0.3) and (6.95,1.4) .. (10.93,3.29)   ;
\draw [color={rgb, 255:red, 0; green, 24; blue, 255 }  ,draw opacity=1 ]   (402.22,109.14) -- (459.65,58.81) ;
\draw [shift={(461.16,57.49)}, rotate = 498.77] [color={rgb, 255:red, 0; green, 24; blue, 255 }  ,draw opacity=1 ][line width=0.75]    (10.93,-3.29) .. controls (6.95,-1.4) and (3.31,-0.3) .. (0,0) .. controls (3.31,0.3) and (6.95,1.4) .. (10.93,3.29)   ;
\draw [color={rgb, 255:red, 0; green, 228; blue, 255 }  ,draw opacity=1 ]   (205.89,357.53) -- (268.77,302.15) ;
\draw [shift={(270.27,300.82)}, rotate = 498.62] [color={rgb, 255:red, 0; green, 228; blue, 255 }  ,draw opacity=1 ][line width=0.75]    (10.93,-3.29) .. controls (6.95,-1.4) and (3.31,-0.3) .. (0,0) .. controls (3.31,0.3) and (6.95,1.4) .. (10.93,3.29)   ;
\draw [color={rgb, 255:red, 0; green, 228; blue, 255 }  ,draw opacity=1 ]   (182.54,139.56) -- (160.42,190.98) ;
\draw [shift={(159.63,192.82)}, rotate = 293.28] [color={rgb, 255:red, 0; green, 228; blue, 255 }  ,draw opacity=1 ][line width=0.75]    (10.93,-3.29) .. controls (6.95,-1.4) and (3.31,-0.3) .. (0,0) .. controls (3.31,0.3) and (6.95,1.4) .. (10.93,3.29)   ;
\draw [color={rgb, 255:red, 0; green, 228; blue, 255 }  ,draw opacity=1 ]   (219.63,147.15) -- (274.65,265.19) ;
\draw [shift={(275.5,267)}, rotate = 245] [color={rgb, 255:red, 0; green, 228; blue, 255 }  ,draw opacity=1 ][line width=0.75]    (10.93,-3.29) .. controls (6.95,-1.4) and (3.31,-0.3) .. (0,0) .. controls (3.31,0.3) and (6.95,1.4) .. (10.93,3.29)   ;
\draw [color={rgb, 255:red, 0; green, 24; blue, 255 }  ,draw opacity=1 ]   (303.67,57.27) -- (361.96,105.72) ;
\draw [shift={(363.5,107)}, rotate = 219.73] [color={rgb, 255:red, 0; green, 24; blue, 255 }  ,draw opacity=1 ][line width=0.75]    (10.93,-3.29) .. controls (6.95,-1.4) and (3.31,-0.3) .. (0,0) .. controls (3.31,0.3) and (6.95,1.4) .. (10.93,3.29)   ;
\draw [color={rgb, 255:red, 0; green, 228; blue, 255 }  ,draw opacity=1 ]   (377.6,141.42) -- (355.47,192.83) ;
\draw [shift={(354.68,194.67)}, rotate = 293.28] [color={rgb, 255:red, 0; green, 228; blue, 255 }  ,draw opacity=1 ][line width=0.75]    (10.93,-3.29) .. controls (6.95,-1.4) and (3.31,-0.3) .. (0,0) .. controls (3.31,0.3) and (6.95,1.4) .. (10.93,3.29)   ;
\draw [color={rgb, 255:red, 255; green, 0; blue, 0 }  ,draw opacity=1 ]   (189,210) -- (321.5,210) ;
\draw [shift={(323.5,210)}, rotate = 180] [color={rgb, 255:red, 255; green, 0; blue, 0 }  ,draw opacity=1 ][line width=0.75]    (10.93,-3.29) .. controls (6.95,-1.4) and (3.31,-0.3) .. (0,0) .. controls (3.31,0.3) and (6.95,1.4) .. (10.93,3.29)   ;
\draw [color={rgb, 255:red, 0; green, 228; blue, 255 }  ,draw opacity=1 ]   (398.63,147.15) -- (453.65,265.19) ;
\draw [shift={(454.5,267)}, rotate = 245] [color={rgb, 255:red, 0; green, 228; blue, 255 }  ,draw opacity=1 ][line width=0.75]    (10.93,-3.29) .. controls (6.95,-1.4) and (3.31,-0.3) .. (0,0) .. controls (3.31,0.3) and (6.95,1.4) .. (10.93,3.29)   ;
\draw [color={rgb, 255:red, 0; green, 228; blue, 255 }  ,draw opacity=1 ]   (383.89,357.53) -- (446.77,302.15) ;
\draw [shift={(448.27,300.82)}, rotate = 498.62] [color={rgb, 255:red, 0; green, 228; blue, 255 }  ,draw opacity=1 ][line width=0.75]    (10.93,-3.29) .. controls (6.95,-1.4) and (3.31,-0.3) .. (0,0) .. controls (3.31,0.3) and (6.95,1.4) .. (10.93,3.29)   ;
\draw [color={rgb, 255:red, 255; green, 0; blue, 0 }  ,draw opacity=1 ]   (317,40) -- (449.5,40) ;
\draw [shift={(451.5,40)}, rotate = 180] [color={rgb, 255:red, 255; green, 0; blue, 0 }  ,draw opacity=1 ][line width=0.75]    (10.93,-3.29) .. controls (6.95,-1.4) and (3.31,-0.3) .. (0,0) .. controls (3.31,0.3) and (6.95,1.4) .. (10.93,3.29)   ;

\draw (273.31,28) node [anchor=north west][inner sep=0.75pt]  [rotate=-359.37]  {$Q_{\omega _{k}}$};
\draw (472.34,28) node [anchor=north west][inner sep=0.75pt]  [rotate=-359.37]  {$Q_{\omega _{k+1}}$};
\draw (357.29,114) node [anchor=north west][inner sep=0.75pt]  [rotate=-359.37]  {$\mathbb{T}^{\pi _{\theta _{k+1}}} Q_{\omega _{k}}$};
\draw (174.24,114) node [anchor=north west][inner sep=0.75pt]  [rotate=-359.37]  {$\mathbb{T}^{\pi _{\theta _{k}}} Q_{\omega _{k-1}}$};
\draw (264.5,275) node [anchor=north west][inner sep=0.75pt]  [rotate=-359.37]  {$\mathbb{T}^{\pi ^{*}} Q_{\omega _{k}}$};
\draw (146.21,200) node [anchor=north west][inner sep=0.75pt]  [rotate=-359.37]  {$Q^{\pi _{\theta _{k}}}$};
\draw (330.25,200) node [anchor=north west][inner sep=0.75pt]  [rotate=-359.37]  {$Q^{\pi _{\theta _{k+1}}}$};
\draw (439.85,275) node [anchor=north west][inner sep=0.75pt]  [rotate=-359.37]  {$\mathbb{T}^{\pi ^{*}} Q_{\omega _{k+1}}$};
\draw (318.5,365) node [anchor=north west][inner sep=0.75pt]  [rotate=-359.37]  {$\mathbb{T}^{\pi ^{*}} Q^{\pi _{\theta _{k}}}$};
\draw (135.48,365) node [anchor=north west][inner sep=0.75pt]  [rotate=-359.37]  {$\mathbb{T}^{\pi ^{*}} Q^{\pi _{\theta _{k-1}}}$};
\draw (219.69,62) node [anchor=north west][inner sep=0.75pt]  [rotate=-359.37]  {$\varepsilon ^{\text{c}}_{k}$};
\draw (121.66,150) node [anchor=north west][inner sep=0.75pt]  [rotate=-359.37]  {$A_{3,k-1}$};
\draw (190,310) node [anchor=north west][inner sep=0.75pt]  [rotate=-359.37]  {$A_{2,k-1}$};
\draw (229.72,230) node [anchor=north west][inner sep=0.75pt]  [rotate=-359.37]  {$A_{1,k}$};
\draw (332.71,62) node [anchor=north west][inner sep=0.75pt]  [rotate=-359.37]  {$e_{k+1}$};
\draw (396.71,62) node [anchor=north west][inner sep=0.75pt]  [rotate=-359.37]  {$\varepsilon ^{\text{c}}_{k+1}$};
\draw (390,230) node [anchor=north west][inner sep=0.75pt]  [rotate=-359.37]  {$A_{1,k+1}$};
\draw (329.68,150) node [anchor=north west][inner sep=0.75pt]  [rotate=-359.37]  {$A_{3,k}$};
\draw (383.88,310) node [anchor=north west][inner sep=0.75pt]  [rotate=-359.37]  {$A_{2,k}$};
\draw (263,190) node [anchor=north west][inner sep=0.75pt]    {$\vartheta _{k}$};
\draw (0,33) node [anchor=north west][inner sep=0.75pt]   [align=left] {Critic Update:};
\draw (0,203) node [anchor=north west][inner sep=0.75pt]   [align=left] {Actor Update:};

\end{tikzpicture}

\caption{Illustration of the relationship  among  $A_{1,k}$, $A_{2,k}$, $A_{3,k}$, $\epsilon_{k+1}^{\rm c}$, $e_{k+1}$, and $\vartheta_k$.
Here $\{\theta_k, \omega_k \}$ and $\{\theta_{k+1}, \omega_{k+1}\}$ are two consecutive iterates of actor-critic. 
The red arrow  from  $Q_{\omega_k}$ to $Q_{\omega_{k+1}}$  represents  the  critic update and the red arrow from $Q^{\pi_{\theta_k}}$ to $Q^{\pi_{\theta_{k+1}}}$ represents the action-value functions associated with the two policies in any actor update. Here  
$\vartheta_k $ given in \eqref{eq:def-deltak-sketch} quantifies the difference between $\pi_{\theta_k}$ and $\pi_{\theta_{k+1}}$ in terms of their KL distances to $\pi^*$. In addition, the cyan arrows represent quantities $A_{1, k}$, $A_{2, k}$, and $A_{3,k}$ introduced in \eqref{eq:def-as-skc-a1}--\eqref{eq:def-as-skc-a3}, 
which are intermediate terms used for analyzing the error $Q^* - Q^{\pi_{k+1}}$. 
Finally, the blue arrows represent $\varepsilon_{k+1}^c$ and $e_{k+1}$ defined in \eqref{eq:error-Q-sketch} and \eqref{eq:def-ek-sketch}, respectively. Here $\varepsilon_{k+1}^c$  corresponds to the statistical error due to having finite data whereas  $e_{k+1}$ essentially quantifies the difference between $\pi_{\theta_k}$ and $\pi_{\theta_{k+1}}$.
 } \label{fig:o}
\end{figure}

\vskip5pt
In the sequel, we upper bound   $A_{1,k}$, $A_{2,k}$, and $A_{3,k}$, respectively. To establish such upper bounds, we define the following quantities, 
\#
& \epsilon^{\rm c}_{k+1}(s, a) = [\cT^{\pi_{\theta_{k+1}}} Q_{\omega_k}-Q_{\omega_{k+1}}](s, a), \label{eq:error-Q-sketch} \\
& e_{k+1}(s, a) =  [Q_{\omega_{k}} - \cT^{\pi_{\theta_{k+1}}} Q_{\omega_{k}}  ](s, a), \label{eq:def-ek-sketch} \\
& \vartheta_k(s) = \kl \bigl(\pi^*(\cdot\given s) \,\|\, \pi_{\theta_{k}}(\cdot\given s) \bigr)  - \kl \bigl(\pi^*(\cdot\given s) \,\|\, \pi_{\theta_{k+1}}(\cdot\given s) \bigr). \label{eq:def-deltak-sketch} 
\#
To understand the intuition behind $\epsilon^{\rm c}_{k+1}$, $e_{k+1}$, and $\vartheta_k$, we interpret them as follows. 

\vskip5pt\noindent{\bf Interpretation of $\epsilon^{\rm c}_{k+1}$.}
Recall that $\tilde \omega_{k+1}$ is defined in \eqref{eq:pop-omega}, which parameterizes $\cT^{\pi_{\theta_{k+1}}} Q_{\omega_k}$ (ignoring the projection in \eqref{eq:omega-update}). 
Here
$\epsilon^{\rm c}_{k+1}$ arises from  approximating $\tilde \omega_{k+1}$ using $\omega_{k+1}$ as an estimator, which is constructed based on   $\omega_k$ and the $N$ data points.
In particular, $\epsilon_{k+1}^{\rm c}$ decreases to zero as   $N\to\infty$, which is used in characterizing $A_{2,k}$ defined in \eqref{eq:def-as-skc-a2}. 

\vskip5pt\noindent{\bf Interpretation of $e_{k+1}$.}
Assuming that $A_{3,k-1}$ defined in \eqref{eq:def-as-skc-a3} and $\epsilon_{k}^{\rm c}$ defined in \eqref{eq:error-Q-sketch} converge to zero, which implies $\cT^{\pi_{\theta_k}} Q_{\omega_{k-1}} \approx Q^{\pi_{\theta_k}}$ and $ \cT^{\pi_{\theta_k}} Q_{\omega_{k-1}} \approx Q_{\omega_{k}} $, respectively, 
we have $Q_{\omega_{k}}  \approx Q^{\pi_{\theta_k}}$. Therefore, as defined in \eqref{eq:def-ek-sketch},  $e_{k+1} = Q_{\omega_{k}} - \cT^{\pi_{\theta_{k+1}}} Q_{\omega_{k}} \approx Q^{\pi_{\theta_k}} - \cT^{\pi_{\theta_{k+1}}}Q^{\pi_{\theta_k}} =  (\cT^{\pi_{\theta_{k}}} - \cT^{\pi_{\theta_{k+1}}}) Q^{\pi_{\theta_k}}$ measures the difference between $\pi_{\theta_k}$ and $\pi_{\theta_{k+1}}$, which implies the difference between $\cT^{\pi_{\theta_k}}$ and $\cT^{\pi_{\theta_{k+1}}}$. We remark that $e_{k+1}$ fully characterizes $A_{3,k}$ defined in \eqref{eq:def-as-skc-a3} as shown in \eqref{eq:up-bd-a-skc} subsequently. 

\vskip5pt\noindent{\bf Interpretation of $\vartheta_k$.}
As defined in \eqref{eq:def-deltak-sketch}, $\vartheta_k$   measures the difference between $\pi_{\theta_k}$ and $\pi_{\theta_{k+1}}$ in terms of their differences with $\pi^*$, which are measured by the corresponding KL-divergences. 
In particular, $\vartheta_k$ is used in characterizing $A_{1,k}$ and $A_{2,k}$ defined in \eqref{eq:def-as-skc-a1} and \eqref{eq:def-as-skc-a2}, respectively.   

\vskip5pt
We remark that $\epsilon_{k+1}^{\rm c}$ measures the statistical error in the critic update, while $\vartheta_k$ measures the optimization error in the actor update.  As discussed above, the convergence of $A_{3,k}$ to zero implies the contraction of both the actor update and the critic update, which illustrates the ``double contraction'' phenomenon.  Meanwhile, since $e_{k+1}$ fully characterizes $A_{3,k}$ as shown in \eqref{eq:up-bd-a-skc} subsequently,  $e_{k+1}$ plays a key role in  the ``double contraction'' phenomenon.  In particular, the convergence of $e_{k+1}$ to zero is established in \eqref{eq:ek-bound-sk} subsequently.   See Figure \ref{fig:o} for an illustration of  these quantities.

 With the quantities defined in \eqref{eq:error-Q-sketch}, \eqref{eq:def-ek-sketch}, and \eqref{eq:def-deltak-sketch}, we upper bound $A_{1,k} $, $A_{2,k} $, and $A_{3,k} $ as follows, 
\#\label{eq:up-bd-a-skc}
& A_{1,k}(s,a) \leq  \gamma \beta \cdot [ \PP \vartheta_k   ] (s,a), \notag \\
& A_{2,k}(s,a ) \leq \bigl[ (\gamma\PP^{\pi^*})^{k+1} (Q^* - Q_{\omega_{0}}) \bigr] (s,a)  + \gamma \beta \cdot \sum_{i = 0}^{k-1} \bigl[ (\gamma \PP^{\pi^*} )^{k-i}  \PP \vartheta_i    \bigr](s,a) \notag \\
& \qquad\qquad\qquad + \sum_{i = 0}^{k-1} \bigl[ (\gamma \PP^{\pi^*} )^{k-i} \epsilon^{\rm c}_{i+1} \bigr] (s,a), \notag \\
& A_{3,k}(s,a) = \bigl[ \gamma \PP^{\pi_{\theta_{k+1}}} (I - \gamma \PP^{\pi_{\theta_{k+1}}})^{-1} e_{k+1} \bigr](s,a), 
\#
the proof of which is deferred to Lemmas \ref{lemma:a1}, \ref{lemma:a2}, and \ref{lemma:a3} in \S\ref{prf:thm:regret-lin} of the appendix, respectively.  
Meanwhile, by recursively expanding \eqref{eq:error-Q-sketch} and \eqref{eq:def-ek-sketch}, we have
\#\label{eq:ek-bound-sk}
e_{k+1}(s,a) & \leq \biggl[ \gamma^k \Bigl( \prod_{s = 1}^k \PP^{\pi_{\theta_s}} \Bigr) e_1 + \sum_{i = 1}^k \gamma^{k-i} \Bigl( \prod_{s = i+1}^k \PP^{\pi_{\theta_s}} \Bigr)       (I - \gamma \PP^{\pi_{\theta_{i}}})\epsilon_i^{\rm c}    \biggr](s,a), 
\#
the proof of which is deferred to Lemma \ref{lemma:bound-ek} in \S\ref{prf:thm:regret-lin} of the appendix. 
By plugging \eqref{eq:ek-bound-sk} into \eqref{eq:up-bd-a-skc}, we have
\#\label{eq:up-bd-a-skc2}
& A_{3,k}(s,a) \leq \biggl[ \gamma \PP^{\pi_{\theta_{k+1}}} (I - \gamma \PP^{\pi_{\theta_{k+1}}})^{-1} \biggl(\gamma^k \Bigl( \prod_{s = 1}^k \PP^{\pi_{\theta_s}} \Bigr) e_1  \\
& \qquad \qquad \qquad\qquad\qquad\qquad\qquad\qquad + \sum_{i = 1}^k \gamma^{k-i} \Bigl( \prod_{s = i+1}^k \PP^{\pi_{\theta_s}} \Bigr)     (I - \gamma \PP^{\pi_{\theta_{i}}})\epsilon_i^{\rm c}   \biggr) \biggr](s,a). \notag
\#
To better understand \eqref{eq:up-bd-a-skc2} and how it relates to the convergence of $A_{3,k}$, $A_{2,k}$, and $A_{1,k}$ to zero, we discuss in the following two steps. 

\vskip5pt\noindent{\bf Step (i).}
We assume $\epsilon_{i}^{\rm c} = 0$, which corresponds to the number of data points $N\to\infty$. Then  \eqref{eq:up-bd-a-skc2} yields $A_{3,k} = O(\gamma^k)$, which implies that $A_{3,k}$ defined in \eqref{eq:def-as-skc-a3} converges to zero driven by the discount factor $\gamma$. As discussed above,  the convergence of $A_{3,k}$ to zero also implies the contraction between $\pi_{\theta_k}$ and $\pi_{\theta_{k+1}}$ of the actor update and the contraction between $Q_{\omega_k}$ and $Q^{\pi_{\theta_k}}$ of  the critic update, which illustrates the ``double contraction'' phenomenon.

\vskip5pt\noindent{\bf Step (ii).}
The convergence of $A_{3,k}$ to zero further ensures that $A_{2,k}$ converges to zero. 
To see this, we further assume $A_{3,k} = 0$, which together with the assumption that $\epsilon^{\rm c}_{k+1} = 0$ implies $ Q^{\pi_{\theta_{k+1}}} = \cT^{\pi_{\theta_{k+1}}} Q_{\omega_{k}} = Q_{\omega_{k+1}}$ by their definitions  in \eqref{eq:def-as-skc-a3} and \eqref{eq:error-Q-sketch}, respectively. 
Then by telescoping the sum of $A_{2,k}$ defined in \eqref{eq:def-as-skc-a2}, which cancels out $Q_{\omega_{k+1}}$ and $Q^{\pi_{\theta_{k+1}}}$, we obtain the convergence of $A_{2,k}$ to zero.  
Meanwhile, telescoping the sum of  $A_{1,k}$ defined in \eqref{eq:def-as-skc-a1} and  the sum of its upper bound in \eqref{eq:up-bd-a-skc} implies that $A_{1,k}$  converges to zero. 

\vskip5pt
Now, by plugging \eqref{eq:up-bd-a-skc} and \eqref{eq:up-bd-a-skc2} into \eqref{eq:Q-bound1-skc}, we establish an upper bound of $\sum_{k = 0}^K ( Q^*(s,a) - Q^{\pi_{\theta_{k+1}}}(s,a) )$ for any   $(s,a)\in\cS\times \cA$, which is deferred to \eqref{eq:f-bd-lin} in \S\ref{prf:thm:regret-lin} of the appendix.  Hence, we conclude the proof in part 1.   See part 1 of \S\ref{prf:thm:regret-lin} for details.

\vskip5pt
\noindent\textbf{Part 2.} Recall that  $\rho$ is a state-action distribution satisfying \ref{assum:2-lin} of Assumption \ref{assum:cc-lin}.  In the sequel, we take the expectation over $\rho$ in \eqref{eq:f-bd-lin}  and upper bound each term. 
We first introduce the following lemma, which upper bounds   $\epsilon^{\rm c}_{k+1}$ defined in \eqref{eq:error-Q-sketch}. 

\begin{lemma}\label{lemma:l-critic}
Under Assumptions \ref{assum:error} and \ref{assum:matrix},  with probability at least $1 - \delta$,  it holds for any $k\in \{0, 1, \ldots, K\}$ that 
\$
\EE_{\rho_{k+1}}\bigl[ \epsilon^{\rm c}_{k+1}(s,a)^2 \bigr] = \EE  \bigl[ \bigl( Q_{\omega_{k+1}}(s,a)  -[ \cT^{\pi_{\theta_{k}}} Q_{\omega_k}] (s,a)  \bigr)^2 \bigr] \leq \frac{ 32 (\cR_{\max} + R)^2 }{ N (\sigma^*)^4} \cdot \log^2 (NK/p+dK/p), 
\$
where the expectation is taken with respect to $(s,a)\sim \rho_{k+1}$. 
\end{lemma}
\begin{proof}
See \S\ref{prf:lemma:l-critic} for a detailed proof.
\end{proof}

On the right-hand side of \eqref{eq:f-bd-lin} in \S\ref{prf:thm:regret-lin} of the appendix, for the terms not involving $\epsilon^{\rm c}_{k+1}$, i.e., $M_1$, $M_2$, and $M_3$ in \eqref{eq:meiyoue},  we take the expectation over $\rho$ and establish their upper bounds in the $\ell_\infty$-norm over $(s,a)$  in Lemma \ref{lemma:meiyou}.   On the other hand, for the terms involving $\epsilon^{\rm c}_{k+1}$, i.e., $M_4$ and $M_5$ in \eqref{eq:youe},  we  take the expectation over $\rho$ and  then  change the measure  from $\rho$ to $\rho_{k+1}$.  By Assumption \ref{assum:cc-lin} and Lemma \ref{lemma:l-critic}, which relies on $\rho_{k+1}$, we establish the upper bounds in Lemma \ref{lemma:you}.  See part 2 of \S\ref{prf:thm:regret-lin} for details. 

 Combining Lemmas \ref{lemma:meiyou} and \ref{lemma:you} yields Theorem \ref{thm:regret-lin}.
See \S\ref{prf:thm:regret-lin} for a detailed proof.

\bibliographystyle{ims}
\bibliography{rl_ref}

\newpage
\onecolumn

\appendix{}

\section{Deep Neural Network Approximation}\label{sec:nac}

In this section, we consider deep neural network approximation.  We first formally define deep neural networks.  Then we introduce the actor-critic method under such a parameterization. 


A deep neural network (DNN) $u_\theta(x)$ with the input $x\in\RR^d$, depth $H$, and width $m$ is defined as
\#\label{eq:def-nn-form}
& x^{(0)} = x, \quad x^{(h)} = \frac{1}{\sqrt{m}}\cdot \sigma( W_h^\top x^{(h-1)} ),   {\rm ~for~}h \in [H], \quad u_\theta(x) = b^\top x^{(H)}. 
\#
Here $\sigma\colon \RR^m \to \RR^m$ is the rectified linear unit (ReLU) activation function, which is define as $\sigma(y) = ( \max\{0, y_1\}, \ldots, \max\{0, y_m\} )^\top$ for any $y = (y_1, \ldots, y_m)^\top\in \RR^m$. Also, we have $b\in\{-1,1\}^{m}$, $W_1\in\RR^{d\times m}$, and $W_h\in\RR^{m\times m}$ for $2\leq h\leq H$. Meanwhile, we denote the parameter of the DNN $u_\theta$ as $\theta = (\vec(W_1)^\top, \ldots, \vec(W_H)^\top)^\top \in\RR^{m_{\rm all}}$ with $m_{\rm all} = md + (H-1)m^2$. We call $\{W_{h}\}_{h\in[H]}$ the weight matrices of $\theta$. Without loss of generality, we normalize the input $x$ such that $\|x\|_2 = 1$. 

We initialize the DNN such that each entry of $W_h$ follows the standard Gaussian distribution $\mathcal N(0,1)$ for any $h\in[H]$, while each entry of $b$ follows the uniform distribution ${\rm Unif}(\{-1,1\})$. Without loss of generality, we fix $b$ during training and only optimize $\{W_h\}_{h\in[H]}$. We denote the initialization of the parameter $\theta$ as $\theta_0 = (\vec(W_1^0)^\top, \ldots, \vec(W_H^0)^\top)^\top$.  Meanwhile, we restrict $\theta$ within the ball $\cB(\theta_0, R)$ during training,  which is defined as follows, 
\#\label{eq:def-proj-set}
& \cB(\theta_0, R) = \bigl\{\theta\in\RR^{m_{\rm all}} \colon \|W_{h} - W_{h}^0\|_\F \leq R,  ~{\rm for~}h\in[H] \bigr\}. 
\#
Here $\{W_{h}\}_{h\in[H]}$ and $\{W_{h}^0\}_{h\in[H]}$ are the weight matrices of  $\theta$ and $\theta_0$, respectively. By \eqref{eq:def-proj-set}, we have $\|\theta-\theta_0\|_2\leq R\sqrt{H}$ for any $\theta\in\cB(\theta_0, R)$. Now, we define the family of DNNs as
\#\label{eq:def-dnn-class}
\cU(m, H, R) = \bigl\{u_\theta \colon \theta\in \cB(\theta_0, R)\bigr\}, 
\#
where  $u_\theta$ is a DNN with depth $H$ and width $m$. 

We parameterize the action-value function using $Q_\omega(s,a)\in\cU(m_{\rm c}, H_{\rm c}, R_{\rm c})$  and the energy function of the energy-based policy $\pi_\theta$ using $f_\theta(s,a)\in\cU(m_{\rm a}, H_{\rm a}, R_{\rm a})$.  Here $\cU(m_{\rm c}, H_{\rm c}, R_{\rm c})$ and $\cU(m_{\rm a}, H_{\rm a}, R_{\rm a})$ are the families of DNNs defined in \eqref{eq:def-dnn-class}. 
Hereafter we assume that the energy function $f_\theta$ and the action-value function $Q_\omega$ share the same architecture and initialization, i.e., $m_{\rm a} = m_{\rm c}$, $H_{\rm a} = H_{\rm c}$, $R_{\rm a} = R_{\rm c}$, and $\theta_0 = \omega_0$. Such shared architecture and initialization of the DNNs ensure that the parameterizations of the policy   and the action-value function  are approximately compatible. See \cite{sutton2000policy, konda2000actor, kakade2002natural, peters2008natural, wang2019neural} for a detailed discussion.

\noindent{\bf Actor Update.}
To solve \eqref{eq:actor-mse}, we   use   projected stochastic gradient descent, whose $n$-th iteration has the following form, 
\$
& \theta(n+1) \notag\\
& \quad \gets   \Gamma_{\cB(\theta_0, R_{\rm a})}\bigl( \theta(n) - \alpha\cdot \bigl( f_{\theta(n)}(s,a) - \tau_{k+1}\cdot \bigl( \beta^{-1} Q_{\omega_k}(s,a) + \tau_k^{-1} f_{\theta_k}(s,a)  \bigr)  \bigr) \cdot \nabla_\theta f_{\theta(n)}(s,a)\bigr). 
\$
Here $\Gamma_{\cB(\theta_0, R_{\rm a})}$ is the projection operator, which projects the parameter onto the ball   $\cB(\theta_0, R_{\rm a})$ defined in \eqref{eq:def-proj-set}. The state-action pair $(s,a)$ is sampled from the stationary state-action distribution $\rho_k$. We summarize the update in Algorithm \ref{algo:actor}, which is deferred to \S\ref{sec:aux-algo-paper} of the appendix.

\noindent{\bf Critic Update.}
 To solve \eqref{eq:critic-update-p}, we   apply  projected stochastic gradient descent. More specifically, at the $n$-th iteration of   projected stochastic gradient descent, we sample a tuple $(s, a, r, s', a')$, where $(s, a)\sim \rho_{k+1}$, $r = \cR(s, a)$, $s'\sim P(\cdot\given s, a)$, and $a'\sim\pi_{\theta_{k+1}}(\cdot\given s')$. 
We define the residual at the $n$-th iteration as $\delta(n) = Q_{\omega(n)}(s, a) - (1-\gamma) \cdot r - \gamma \cdot Q_{\omega_k}(s', a')$. 
Then   the $n$-th iteration of projected stochastic gradient descent has the following form, 
\$
\omega(n+1)\gets \Gamma_{\cB(\omega_0, R_{\rm c})} \bigl(\omega(n) - \eta\cdot \delta(n)\cdot \nabla_\omega Q_{\omega(n)}(s, a) \bigr). 
\$
Here $\Gamma_{\cB(\omega_0, R_{\rm c})}$ is the projection operator, which projects the parameter onto the ball   $\cB(\omega_0, R_{\rm c})$ defined in \eqref{eq:def-proj-set}. 
We summarize the update in Algorithm \ref{algo:critic}, which is deferred to \S\ref{sec:aux-algo-paper} of the appendix. 

By assembling Algorithms \ref{algo:actor} and \ref{algo:critic}, we present the deep neural actor-critic method in Algorithm  \ref{algo:n-ac}, which is deferred to \S\ref{sec:aux-algo-paper} of the appendix. 

Finally, we remark that the  off-policy  actor and critic updates given in 
\eqref{eq:offpolicy_actor} and \eqref{eq:offpolicy_critic}
can also incorporate deep neural network approximation with a slight modification, which enables data reuse in the algorithm.

\section{Details of Algorithms}\label{sec:aux-algo-paper}

In this section, we summarize the algorithms in \S\ref{sec:ac}. 
We first introduce the actor-critic method with linear function approximation in Algorithm \ref{algo:l-ac}.

\begin{algorithm}[h]
  \caption{Linear Actor-Critic Method}
  \label{algo:l-ac}
  \begin{algorithmic}
  \STATE{{\textbf{Input:}} Number of iterations $K$, sample size $N$, temperature parameter $\beta$.}
  \STATE{\textbf{Initialization:} Set $\tau_0 \gets \infty$, and randomly initialize the actor parameter $\theta_0$ and the critic parameter $\omega_0$.  }
  \FOR{$k = 0, 1, 2, \ldots, K$}
    \STATE {\textbf{Actor Update:} Update $\theta_{k+1}$ via \eqref{eq:theta-update} with $\tau_{k+1}^{-1} = (k+1)\cdot \beta^{-1}$. } 
        \STATE {\textbf{Critic Update:} Sample $\{(s_{\ell,1}, a_{\ell,1})\}_{\ell \in [N]}$ and $\{(s_{\ell,2}, a_{\ell,2}, r_{\ell,2}, s_{\ell,2}', a_{\ell,2}')\}_{\ell \in [N]}$ as specified in \S\ref{sec:lac}. Update $\omega_{k+1}$ via \eqref{eq:omega-update}. } 
  \ENDFOR
  \STATE{\textbf{Output:}   $\{\pi_{\theta_{k}}\}_{k \in [K+1]}$, where $\pi_{\theta_{k}} \propto \exp(\tau_k^{-1} f_{\theta_k})$.  }
  \end{algorithmic}
\end{algorithm}

We   introduce the actor-critic method with DNN approximation in Algorithm \ref{algo:n-ac}, which relies on  Algorithms \ref{algo:actor} and \ref{algo:critic} for the actor and critic updates.

\begin{algorithm}[h]
  \caption{Deep Neural Actor-Critic Method}
  \label{algo:n-ac}
  \begin{algorithmic}
  \STATE{{\textbf{Input:}} Number of iterations $K, N_{\rm a}, N_{\rm c}$, stepsizes $\alpha, \eta$, and temperature parameter $\beta$.}
  \STATE{\textbf{Initialization:} Set $\tau_0 \gets \infty$ and initialize DNNs $f_{\theta_0}$ and $Q_{\omega_0}$ as specified in \S\ref{sec:nac}. }
  \FOR{$k = 0, 1, 2, \ldots, K$}
    \STATE {\textbf{Actor Update:} Update $\theta_{k+1}$ via Algorithm \ref{algo:actor} with input $\pi_{\theta_k}$, $\theta_0$,  $Q_{\omega_k}$, $\alpha$, $\beta$,  $\tau_{k+1} = (k+1)^{-1}\cdot \beta$, and $N_{\rm a}$. } 
        \STATE {\textbf{Critic Update:} Update $\omega_{k+1}$ via Algorithm \ref{algo:critic} with input $\pi_{\theta_{k+1}}$, $Q_{\omega_k}$,  $\omega_0$, $\eta$, and $N_{\rm c}$. } 
  \ENDFOR
  \STATE{\textbf{Output:} $\{\pi_{\theta_{k}}\}_{k \in [K+1]}$, where $\pi_{\theta_{k}} \propto \exp(\tau_k^{-1} f_{\theta_k})$.  }
  \end{algorithmic}
\end{algorithm}

\begin{algorithm}[h] 
  \caption{Actor Update for Deep Neural Actor-Critic Method}
  \label{algo:actor}
  \begin{algorithmic}
  \STATE{{\textbf{Input:}} Policy $\pi_{\theta}\propto \exp(\tau^{-1} f_{\theta})$,  initial actor parameter $\theta_0$,  action-value function $Q_{\omega}$, stepsize $\alpha$, temperature parameter $\beta$, temperature $\tilde \tau$, and number of iterations $N_{\rm a}$. }
  \STATE{\textbf{Initialization:} Set $\theta(0)\gets \theta_0$. }
  \FOR{$n = 0, 1, 2, \ldots, N_{\rm a}-1$}
    \STATE {Sample   $(s,a)$ as specified in \S\ref{sec:nac}. }
    \STATE {Set $\theta(n+1) \gets \Gamma_{\cB(\theta_0, R_{\rm a})} ( \theta(n) - \alpha\cdot  ( f_{\theta(n)}(s, a) - \tilde \tau \cdot (\beta^{-1} Q_{\omega}(s, a) + \tau^{-1} f_{\theta}(s, a) )  )\cdot \nabla_\theta f_{\theta(n)}(s, a) )$. }
  \ENDFOR
  \STATE{\textbf{Output:} $\overline \theta = 1/N_{\rm a}\cdot \sum_{n = 1}^{N_{\rm a}} \theta(n)$.  }
  \end{algorithmic}
\end{algorithm}

\begin{algorithm}[H] 
  \caption{Critic Update for Deep Neural Actor-Critic Method}
  \label{algo:critic}
  \begin{algorithmic}
  \STATE{{\textbf{Input:}} Policy $\pi_\theta$, action-value function $Q_{\omega}$,  initial critic parameter $\omega_0$,  stepsize $\eta$, and number of iterations $N_{\rm c}$.} 
  \STATE{\textbf{Initialization:} Set $\omega(0)\gets \omega_0$. }
  \FOR{$n = 0, 1, 2, \ldots, N_{\rm c}-1$}
    \STATE {Sample   $(s, a, r, s', a')$ as specified in \S\ref{sec:nac}. }
    \STATE {Set $\delta(n) \gets Q_{\omega(n)}(s, a) - (1-\gamma) \cdot r - \gamma \cdot Q_{\omega}(s', a') $. }
    \STATE {Set $\omega(n+1)\gets \Gamma_{\cB(\omega_0, R_{\rm c})}(\omega(n) - \eta\cdot \delta(n)\cdot \nabla_\omega Q_{\omega(n)}(s, a))$. }
  \ENDFOR
  \STATE{\textbf{Output:} $\overline \omega = 1/N_{\rm c}\cdot \sum_{n = 1}^{N_{\rm c}} \omega(n)$.  }
  \end{algorithmic}
\end{algorithm}

\section{Convergence Results of Algorithm \ref{algo:n-ac}} \label{sec:theory-nn}
In this section, we upper bound the regret of the deep neural actor-critic method. 
Hereafter we assume that  $|\cR(s,a)| \leq \cR_{\max}$ for any $(s,a) \in \cS\times \cA$, where $\cR_{\max}$ is a positive absolute constant. 
First, we impose the following assumptions in parallel  to Assumption \ref{assum:cc-lin}. Recall that $\rho^*$ is the stationary state-action distribution of $\pi^*$, while $\rho_k$ is the stationary state-action distribution of $\pi_{\theta_k}$.

\begin{assumption} [Concentrability Coefficient] \label{assum:cc}
The following statements hold. 
\begin{itemize}
\item[(i)]  There exists a positive absolute constant $\phi^*$ such that $\phi_k^* \leq \phi^*$  for any $k\geq 1$, where $\phi_k^* = \|{\ud\rho^*}/{\ud\rho_{k}}\|_{\rho_k, 2}$. 

\item [(ii)\namedlabel{assum:2}{(ii)}]  
{For the state-action distribution $\rho$ 
used to define the regret in \eqref{eq:define_regret}, 
we assume that
for any  $k \geq 1$ and a sequence of policies $\{\pi_i\}_{i\geq 1}$, the $k$-step future-state-action distribution $\rho \PP^{\pi_1} \cdots \PP^{\pi_k}$ is absolutely continuous with respect to $\rho^*$. Also, it holds that}
\$
C_{\rho, \rho^*} = (1-\gamma)^2 \sum_{k = 1}^\infty k^3 \gamma^{k} \cdot c(k) < \infty, 
\$
where $c(k) = \sup_{\{\pi_i\}_{i\in[k]}} \| {\ud (\rho \PP^{\pi_1} \cdots \PP^{\pi_k})}/{\ud \rho^*}  \|_{\rho^*, \infty}$. 
\end{itemize}
\end{assumption}


Meanwhile, we impose the following assumption in parallel to Assumption \ref{assum:error}. 

\begin{assumption}[Zero Approximation Error] \label{assum:close-bellman}
For any $Q_\omega \in\cU(m_{\rm c}, H_{\rm c}, R_{\rm c})$ and policy $\pi$, it holds that $\cT^\pi Q_\omega \in \cU(m_{\rm c}, H_{\rm c}, R_{\rm c})$, where $\cT^{\pi}$ is defined in \eqref{eq:def-bellman-op}. 
\end{assumption}

Assumption \ref{assum:close-bellman} states that   $\cU(m_{\rm c}, H_{\rm c}, R_{\rm c})$ is closed under the Bellman evaluation operator $\cT^\pi$, which is commonly imposed in the literature \citep{munos2008finite, antos2008fitted, farahmand2010error, farahmand2016regularized, tosatto2017boosted, yang2019theoretical, liu2019neural}.

We upper bound the regret of the deep neural actor-critic method in Algorithm \ref{algo:n-ac} in the sequel.  
To establish such an upper bound, we first establish the rates of convergence of Algorithms \ref{algo:actor} and \ref{algo:critic} as follows.


\begin{proposition}\label{prop:ac-bd}
For any sufficiently large $N_{\rm a} > 0$, let $m_{\rm a} = \Omega(d^{3/2}R_{\rm a}^{-1} H_{\rm a}^{-3/2}   \log(m_{\rm a}^{1/2}/R_{\rm a})^{3/2})$, $H_{\rm a} = O (N_{\rm a}^{1/4})$, and $R_{\rm a} = O (m_{\rm a}^{1/2} H_{\rm a}^{-6}(\log m_{\rm a})^{-3})$.   We denote by $\overline \theta$ the output of Algorithm \ref{algo:actor} with input $\pi_{\theta}\propto\exp(\tau^{-1}f_{\theta})$, $\theta_0$, $Q_{\omega}$, $\alpha$, $\beta$, $\tilde \tau = (\tau^{-1} + \beta^{-1})^{-1}$, and $N_{\rm a}$. Also, let $\tilde f = \tilde \tau \cdot ( \beta^{-1} Q_{\omega} +  \tau^{-1} f_{\theta})$. With probability at least $1 - \exp(-\Omega(R_{\rm a}^{2/3} m_{\rm a}^{2/3} H_{\rm a} ))$ over the random initialization $\theta_0$, we have
\$
& \EE\bigl[ \bigl( f_{\overline\theta}(s,a) - \tilde f(s,a)  \bigr)^2\bigr]  = O ( R_{\rm a}^2 N_{\rm a}^{-1/2} + R_{\rm a}^{8/3} m_{\rm a}^{-1/6} H_{\rm a}^{7} \log m_{\rm a} ). 
\$
Here the expectation is taken over the randomness of  $\overline \theta$ conditioning on the initialization $\theta_0$  and $(s,a)\sim \rho_{\pi_\theta}$, where $\rho_{\pi_\theta}$ is the stationary state-action distribution of $\pi_\theta$. 
\end{proposition}
\begin{proof}
See \S\ref{prf:prop:ac-bd} for a detailed proof.
\end{proof}

\begin{proposition}\label{prop:acc-bd}
For any sufficiently large $N_{\rm c} > 0$, let $m_{\rm c} = \Omega(d^{3/2}R_{\rm c}^{-1} H_{\rm c}^{-3/2}   \log(m_{\rm c}^{1/2}/R_{\rm c})^{3/2})$, $H_{\rm c} = O (N_{\rm c}^{1/4})$, and $R_{\rm c} = O (m_{\rm c}^{1/2} H_{\rm c}^{-6}(\log m_{\rm c})^{-3})$.  We denote by $\overline \omega$ the output of Algorithm \ref{algo:critic} with input $\pi_\theta$, $Q_\omega$, $\omega_0$, $\eta$, and $N_{\rm c}$. Also, let $\tilde Q = (1-\gamma)\cdot  \cR + \gamma\cdot \PP^{\pi_\theta} Q_\omega$.  With probability at least $1 - \exp(-\Omega(R_{\rm c}^{2/3} m_{\rm c}^{2/3} H_{\rm c} ))$ over the random initialization $\omega_0$, we have
\$
& \EE\bigl[ \bigl(Q_{\bar\omega}(s,a) - \tilde Q(s,a) \bigr)^2 \bigr]  = O ( R_{\rm c}^2 N_{\rm c}^{-1/2} + R_{\rm c}^{8/3} m_{\rm c}^{-1/6} H_{\rm c}^{7} \log m_{\rm c} ).
\$
Here the expectation is taken over the randomness of  $\overline\omega$ conditioning on the initialization $\omega_0$ and $(s,a)\sim \rho_{\pi_\theta}$, where $\rho_{\pi_\theta}$ is the stationary state-action distribution of $\pi_\theta$. 
\end{proposition}
\begin{proof}
See \S\ref{prf:prop:acc-bd} for a detailed proof.
\end{proof}

Propositions \ref{prop:ac-bd} and \ref{prop:acc-bd} characterize the errors that arise from the  actor and  critic updates in Algorithm \ref{algo:n-ac}, respectively.   In particular, if the widths $m_{\rm a}$ and $m_{\rm c}$ of the DNNs $f_{\theta}$ and $Q_{\omega}$ are sufficiently large,    the errors characterized in Propositions \ref{prop:ac-bd} and \ref{prop:acc-bd} decay to zero at the rates of $O (N_{\rm a}^{-1/2})$ and $O (N_{\rm c}^{-1/2})$, respectively.  Propositions \ref{prop:ac-bd} and \ref{prop:acc-bd} act as the key ingredients to upper bounding the regret of the deep neural actor-critic method.

Based on Propositions \ref{prop:ac-bd} and \ref{prop:acc-bd}, we upper bound the regret of Algorithm \ref{algo:n-ac} in the following theorem, which is in parallel to Theorem \ref{thm:regret-lin}.

\begin{theorem}\label{thm:regret}
We assume that Assumptions \ref{assum:cc} and \ref{assum:close-bellman} hold. Let $\rho$ be a state-action distribution satisfying \ref{assum:2} of Assumption \ref{assum:cc}.
 Also, for any sufficiently large $K > 0$, 
 let 
$N_{\rm a} = \Omega( K^{6} C_{\rho, \rho^*}^4   (\phi^* + \psi^* + 1)^4 R_{\rm a}^4 )$, $N_{\rm c} = \Omega( K^6 C_{\rho, \rho^*}^4   \phi^{*4} R_{\rm c}^4 )$, $H_{\rm a} = H_{\rm c} = O(N_{\rm c}^{1/4})$, $R_{\rm a} = R_{\rm c} = O (m_{\rm c}^{1/2} H_{\rm c}^{-6}(\log m_{\rm c})^{-3})$, $m_{\rm a} = m_{\rm c} = \Omega( d^{3/2} K^{6} C_{\rho, \rho^*}^{12}   (\phi^* + \psi^* + 1)^{12} R_{\rm c}^{16} H_{\rm c}^{42}  \log(m_{\rm c}^{1/2}/R_{\rm c})^{3/2}  )$, $\beta = K^{1/2}$, and the sequence $\{\theta_k\}_{k\in[K]}$ be generated by Algorithm \ref{algo:n-ac}. 
With probability at least $1 - 1/K$ over the random initialization $\theta_0$ and $\omega_0$, it holds that
\$
& \EE \Bigl [\sum_{k = 0}^K Q^*(s,a) - Q^{\pi_{\theta_{k+1}}}(s,a) \Bigr]  \leq \bigl(2(1-\gamma)^{-3}\log |\cA| + O (1)\bigr)\cdot K^{1/2}, 
\$
where the expectation is taken over the randomness of $(s,a) \sim \rho$ and $\{\theta_{k+1}\}_{k\in [K]}$ conditioning on the initialization $\theta_0$ and $\omega_0$. 
\end{theorem}
\begin{proof}
See \S\ref{prf:thm:regret} for a detailed proof. 
\end{proof}

When the architecture of the actor and critic neural networks are properly chosen,  Theorem \ref{thm:regret} establishes an $O (K^{1/2})$ regret of Algorithm \ref{algo:n-ac},   where $K$ is the total number of iterations.   Specifically speaking, to establish such a regret upper bound, we need the widths $m_{\rm a}$ and $m_{\rm c}$ of the DNNs $f_\theta$ and $Q_\omega$ to be sufficiently large.  
Meanwhile, to control the errors of actor update and critic update in Algorithm \ref{algo:n-ac}, we also run sufficiently large numbers of iterations  in Algorithms \ref{algo:actor} and \ref{algo:critic}.  

In terms of the total sample complexity, to simplify our discussion, we omit constant and logarithmic terms here.  To obtain an $\varepsilon$-globally optimal policy, it suffices to set $K\asymp \varepsilon^{-2}$ in Algorithm \ref{algo:n-ac}. By plugging such a $K$ into $N_{\rm a} = \Omega( K^{6} C_{\rho, \rho^*}^4   (\phi^* + \psi^* + 1)^4 R_{\rm a}^4 )$  and $N_{\rm c} = \Omega( K^6 C_{\rho, \rho^*}^4   \phi^{*4} R_{\rm c}^4 )$ as required in Theorem \ref{thm:regret}, we have $N_{\text{a}} = \tilde O(\varepsilon^{-12})$ and $N_{\text{c}} = \tilde O(\varepsilon^{-12})$. Thus, to achieve an $\varepsilon$-globally optimal policy, the total sample complexity of Algorithm \ref{algo:n-ac} is $\tilde O(\varepsilon^{-14})$.  With the modification to off-policy setting as in \S\ref{sec:lac}, the total sample complexity of Algorithm \ref{algo:n-ac} is $\tilde O(\varepsilon^{-12})$.  

To the best of our knowledge, we establish the rate of convergence and global optimality of the actor-critic method under single-timescale setting with DNN approximation for the first time.

\section{Proofs of Theorems}

\subsection{Proof of Theorem \ref{thm:regret-lin}}\label{prf:thm:regret-lin}

Recall that $\rho$ is a state-action distribution satisfying \ref{assum:2-lin} of Assumption \ref{assum:cc-lin}.   
We first upper bound $\sum_{k = 0}^K ( Q^*(s,a) - Q^{\pi_{\theta_{k+1}}}(s,a)) $ for any   $(s,a)\in\cS\times \cA$  in part 1.  Then by further taking the expectation over $\rho$ and invoking Lemma  \ref{lemma:l-critic}  in part 2,  we conclude the proof of Theorem \ref{thm:regret-lin}. 

\vskip5pt
\noindent\textbf{Part 1.}  In the sequel, we   upper bound   $\sum_{k = 0}^K (Q^*(s,a) - Q^{\pi_{\theta_{k+1}}}(s,a) )$ for any   $(s,a)\in\cS\times \cA$. By the definition of $Q^*$ in \eqref{eq:def-qfunc}, it holds for any $(s,a)\in\cS\times\cA$ that
\#\label{eq:Q-bound1-lin}
& [Q^* - Q^{\pi_{\theta_{k+1}}}](s,a) \notag \\
& \qquad  =  \sum_{\ell = 0}^{\infty} \bigl[ (1-\gamma)\cdot (\gamma \PP^{\pi^*})^\ell \cR \bigr](s,a) - Q^{\pi_{\theta_{k+1}}}(s,a) \notag \\
& \qquad = \sum_{\ell = 0}^{\infty}\bigl[(1-\gamma)\cdot (\gamma \PP^{\pi^*})^\ell \cR +  (\gamma \PP^{\pi^*})^{\ell+1} Q^{\pi_{\theta_{k+1}}} -  (\gamma \PP^{\pi^*})^{\ell+1} Q^{\pi_{\theta_{k+1}}}  \bigr](s,a) - Q^{\pi_{\theta_{k+1}}}(s,a) \notag \\
& \qquad = \sum_{\ell = 0}^{\infty}\bigl[(1-\gamma)\cdot (\gamma \PP^{\pi^*})^\ell \cR +  (\gamma \PP^{\pi^*})^{\ell+1} Q^{\pi_{\theta_{k+1}}} -  (\gamma \PP^{\pi^*})^{\ell} Q^{\pi_{\theta_{k+1}}}  \bigr](s,a)\notag\\
&\qquad  = \sum_{\ell = 0}^\infty \Bigl[(\gamma \PP^{\pi^*})^\ell  \bigl((1-\gamma)\cdot  \cR + \gamma\cdot \PP^{\pi^*} Q^{\pi_{\theta_{k+1}}}  - Q^{\pi_{\theta_{k+1}}} \bigr)  \Bigr](s,a), 
\#
where $\PP^{\pi^*}$ is defined in \eqref{eq:def-ops}.  We upper bound   $[(1-\gamma)\cdot  \cR + \gamma\cdot \PP^{\pi^*} Q^{\pi_{\theta_{k+1}}}  - Q^{\pi_{\theta_{k+1}}} ](s,a)$ on the RHS of \eqref{eq:Q-bound1-lin} in the sequel.  By calculation, we have
\#\label{eq:Q-bound2-lin}
& \bigl[(1-\gamma)\cdot  \cR + \gamma\cdot \PP^{\pi^*} Q^{\pi_{\theta_{k+1}}}  - Q^{\pi_{\theta_{k+1}}} \bigr](s,a) \notag\\
&\qquad  =  \Bigl[ \bigl( (1-\gamma)\cdot  \cR + \gamma\cdot \PP^{\pi^*} Q^{\pi_{\theta_{k+1}}} \bigr)   -  \bigl( (1-\gamma)\cdot  \cR + \gamma\cdot \PP^{\pi^*} Q_{\omega_k }  \bigr)\Bigr](s,a) \notag\\
&\qquad\qquad  + \Bigl[ \bigl( (1-\gamma)\cdot  \cR + \gamma\cdot \PP^{\pi^*} Q_{\omega_k }  \bigr) - \bigl( (1-\gamma)\cdot  \cR + \gamma\cdot \PP^{\pi_{\theta_{k+1}}} Q_{\omega_k }  \bigr) \Bigr] (s,a) \notag\\
&\qquad \qquad + \Bigl[\bigl( (1-\gamma)\cdot  \cR + \gamma\cdot \PP^{\pi_{\theta_{k+1}}} Q_{\omega_k }  \bigr) - Q^{\pi_{\theta_{k+1}}} \Bigr](s,a)\notag\\
&\qquad  =  A_{1,k}(s,a)  + A_{2,k}(s,a)  + A_{3,k}(s,a), 
\#
where $A_{1,k} $, $A_{2,k} $, and $A_{3,k} $ are defined as follows, 
\#\label{eq:jiaoshaa}
& A_{1,k}(s,a) = \bigl[\gamma( \PP^{\pi^*} - \PP^{\pi_{\theta_{k+1}}})Q_{\omega_k }  \bigr] (s,a), \notag \\
& A_{2,k}(s,a) = \bigl[\gamma \PP^{\pi^*} (Q^{\pi_{\theta_{k+1}}} - Q_{\omega_k }) \bigr](s,a),  \notag \\
& A_{3,k}(s,a) = [ \cT^{\pi_{\theta_{k+1}}} Q_{\omega_{k}}  - Q^{\pi_{\theta_{k+1}}} ](s,a).
\#
Here $\cT^{\pi_{\theta_{k+1}}}$ is defined in \eqref{eq:def-bellman-op}. 
By the following three lemmas, we upper bound   $A_{1,k} $, $A_{2,k} $, and $A_{3,k} $ on the RHS of \eqref{eq:Q-bound2-lin}, respectively.    

\begin{lemma}\label{lemma:a1}
It holds for any $(s,a) \in \cS\times \cA$  that
\$
A_{1,k}(s,a) =  \bigl[\gamma( \PP^{\pi^*} - \PP^{\pi_{\theta_{k+1}}})Q_{\omega_k }   \bigr] (s,a) \leq  \bigl[\gamma \beta \cdot \PP(\vartheta_k + \epsilon^{\rm a}_{k+1})\bigr] (s,a), 
\$
where $\vartheta_k$ and  $\epsilon^{\rm a}_{k+1}$ are defined as follows, 
\#
& \vartheta_k(s) = \kl \bigl (\pi^*(\cdot\given s) \,\|\, \pi_{\theta_{k}}(\cdot\given s) \bigr)  - \kl \bigl(\pi^*(\cdot\given s) \,\|\, \pi_{\theta_{k+1}}(\cdot\given s) \bigr), \label{eq:def-deltak-lin}\\
& \epsilon^{\rm a}_{k+1}(s)=  \bigl\la \log \bigl( \pi_{\theta_{k+1}}(\cdot\,|\, s)/\pi_{\theta_k}(\cdot\,|\, s) \bigr) - \beta^{-1} \cdot Q_{\omega_k}(s,\cdot), \pi^*(\cdot\given s) - \pi_{\theta_{k+1}}(\cdot\given s) \bigr\ra.  \label{eq:def-pi-error-lin}
\#
\end{lemma}
\begin{proof}
See \S\ref{prf:lemma:a1} for a detailed proof. 
\end{proof}
We remark that $\epsilon^{\rm a}_{k+1} = 0$ for any $k$ in the  linear actor-critic method.  Meanwhile, such a term is included in Lemma \ref{lemma:a1} only aiming to generalize to the  deep neural actor-critic method.

\begin{lemma}\label{lemma:a2}
It holds for any $(s,a) \in \cS\times \cA$  that
\$
A_{2,k}(s,a )& \leq \bigl[ (\gamma\PP^{\pi^*})^{k+1} (Q^* - Q_{\omega_{0}})  \bigr] (s,a) + \gamma \beta \cdot \sum_{i = 0}^{k-1} \bigl[ (\gamma \PP^{\pi^*}  )^{k-i}   \PP(\vartheta_i +  \epsilon^{\rm a}_{i+1} )  \bigr](s,a)\\
& \qquad  +  \sum_{i = 0}^{k-1} \bigl[ (\gamma \PP^{\pi^*}  )^{k-i}  \epsilon^{\rm c}_{i+1}  \bigr] (s,a),  \notag
\$
where  $\vartheta_i$ is defined in \eqref{eq:def-deltak-lin} of Lemma \ref{lemma:a1},  $\epsilon_{i+1}^{\rm a}$ is defined in \eqref{eq:def-pi-error-lin} of Lemma \ref{lemma:a1},  and   $\epsilon^{\rm c}_{i+1}$ is defined as follows, 
\#\label{eq:error-Q-lin}
\epsilon^{\rm c}_{i+1}(s, a) = [\cT^{\pi_{\theta_{i+1}}} Q_{\omega_i} - Q_{\omega_{i+1}} ](s, a). 
\#
\end{lemma}
\begin{proof}
See \S\ref{prf:lemma:a2} for a detailed proof. 
\end{proof}
We remark that $\epsilon^{\rm a}_{k+1} = 0$ for any $k$ in the  linear actor-critic method.  Meanwhile, such a term is included in Lemma \ref{lemma:a2} only aiming to generalize to the  deep neural actor-critic method.

\begin{lemma}\label{lemma:a3}
It holds for any $(s,a) \in \cS\times \cA$  that
\$
A_{3,k}(s,a)  = \bigl[ \gamma \PP^{\pi_{\theta_{k+1}}} (I - \gamma \PP^{\pi_{\theta_{k+1}}})^{-1} e_{k+1} \bigr](s,a),
\$
where   $e_{k+1}$   is defined as follows, 
\#\label{eq:def-ek-lin}
e_{k+1}(s, a) =  [Q_{\omega_{k}} - \cT^{\pi_{\theta_{k+1}}} Q_{\omega_{k}}  ](s, a). 
\#   
\end{lemma}
\begin{proof}
See \S\ref{prf:lemma:a3} for a detailed proof. 
\end{proof}

We  upper bound  $e_{k+1}$ in \eqref{eq:def-ek-lin} of Lemma \ref{lemma:a3} using Lemma \ref{lemma:bound-ek} as follows. 

\begin{lemma}\label{lemma:bound-ek}
It holds for any $(s,a)\in \cS\times \cA$ that
\$
e_{k+1}(s,a) \leq \biggl[ \gamma^k \Bigl( \prod_{s = 1}^k \PP^{\pi_{\theta_s}} \Bigr) e_1 + \sum_{i = 1}^k \gamma^{k-i} \Bigl( \prod_{s = i+1}^k \PP^{\pi_{\theta_s}} \Bigr) \bigl(\gamma\beta \PP \epsilon_{i+1}^{\rm b} + (I - \gamma \PP^{\pi_{\theta_{i}}})\epsilon_i^{\rm c} \bigr)  \biggr](s,a). 
\$
where $\epsilon_{i}^{\rm c}(s,a)$ is defined in \eqref{eq:error-Q-lin} of Lemma \ref{lemma:a2} and $\epsilon_{i+1}^{\rm b}(s)$ is defined as follows, 
\#\label{eq:eek-b-lin}
\epsilon_{i+1}^{\rm b}(s) = \bigl\la \log \bigl( \pi_{\theta_{i+1}}(\cdot\,|\, s)/\pi_{\theta_i}(\cdot\,|\, s) \bigr) - \beta^{-1} \cdot Q_{\omega_i}(s,\cdot), \pi_{\theta_i}(\cdot\given s) - \pi_{\theta_{i+1}}(\cdot\given s) \bigr\ra. 
\#
\end{lemma}
\begin{proof}
See \S\ref{prf:lemma:bound-ek} for a detailed proof. 
\end{proof}
We remark that $\epsilon^{\rm b}_{i+1} = 0$ for any $i$ in the  linear actor-critic method.  Meanwhile, such a term is included in Lemma \ref{lemma:bound-ek} only aiming to generalize to the  deep neural actor-critic method.

Combining Lemmas \ref{lemma:a3} and \ref{lemma:bound-ek}, we obtain the following upper bound of $A_{3,k}$, 
\#\label{eq:term3-bound-lin}
A_{3,k}(s,a) & = \bigl[ \gamma \PP^{\pi_{\theta_{k+1}}} (I - \gamma \PP^{\pi_{\theta_{k+1}}})^{-1} e_{k+1} \bigr](s,a)\notag\\
& \leq \biggl[ \gamma \PP^{\pi_{\theta_{k+1}}} (I - \gamma \PP^{\pi_{\theta_{k+1}}})^{-1} \biggl(\gamma^k \Bigl( \prod_{s = 1}^k \PP^{\pi_{\theta_s}} \Bigr) e_1  \\
& \qquad \qquad \qquad\qquad \qquad  \quad + \sum_{i = 1}^k \gamma^{k-i} \Bigl( \prod_{s = i+1}^k \PP^{\pi_{\theta_s}} \Bigr) \bigl(\beta\gamma \PP \epsilon_{i+1}^{\rm b}   + (I - \gamma \PP^{\pi_{\theta_{i}}})\epsilon_i^{\rm c} \bigr) \biggr) \biggr](s,a). \notag
\#
Combining \eqref{eq:Q-bound1-lin}, \eqref{eq:Q-bound2-lin}, Lemma \ref{lemma:a1} and Lemma \ref{lemma:a2}, it holds for any $(s,a)\in\cS\times\cA$ that
\#\label{eq:Q-bound3-lin}
& \sum_{k = 0}^K [Q^* - Q^{\pi_{\theta_{k+1}}} ](s,a) \notag\\
&\qquad  \leq \sum_{k = 0}^K \Bigl[  (I - \gamma \PP^{\pi^*})^{-1}  \bigl ( (\gamma\PP^{\pi^*})^{k+1} (Q^* - Q_{\omega_{0}}) + \sum_{i = 0}^{k} (\gamma \PP^{\pi^*}  )^{k-i}  \gamma \beta \PP (\vartheta_i + \epsilon^{\rm a}_{i+1})\notag\\
& \qquad\qquad\qquad\qquad\qquad\qquad  + \sum_{i = 0}^{k-1}  (\gamma \PP^{\pi^*})^{k-i} \epsilon^{\rm c}_{i+1} +  A_{3,k}\bigr) \Bigr](s,a) \notag\\
&\qquad  = \biggl[ (I - \gamma \PP^{\pi^*})^{-1}  \Bigl (  \sum_{k = 0}^K (\gamma\PP^{\pi^*})^{k+1} (Q^* - Q_{\omega_{0}}) + \sum_{k = 0}^K \sum_{i = 0}^{k} (\gamma \PP^{\pi^*}  )^{k-i}  \gamma \beta \PP \epsilon^{\rm a}_{i+1}\\
&\qquad\qquad\qquad\qquad\qquad   + \sum_{k = 0}^K  \sum_{i = 0}^{k-1}  (\gamma \PP^{\pi^*})^{k-i} \epsilon^{\rm c}_{i+1} + \sum_{k = 0}^K A_{3,k} +  \sum_{k = 0}^K \sum_{i = 0}^{k} (\gamma \PP^{\pi^*}  )^{k-i}   \gamma \beta \PP\vartheta_i  \Bigr)  \biggr](s,a), \notag
\#
where $\vartheta_i$, $\epsilon_{i+1}^{\rm a}$, $\epsilon_{i+1}^{\rm c}$, and $e_{k+1}$ are defined in \eqref{eq:def-deltak-lin} of Lemma \ref{lemma:a1}, \eqref{eq:def-pi-error-lin} of Lemma \ref{lemma:a1}, \eqref{eq:error-Q-lin} of Lemma \ref{lemma:a2}, and \eqref{eq:def-ek-lin} of Lemma \ref{lemma:a3}, respectively.  
We upper bound the last term as follows, 
\#\label{eq:delta-terms-bound-lin}
&\biggl[\sum_{k = 0}^K \sum_{i = 0}^{k} (\gamma \PP^{\pi^*}  )^{k-i}   \gamma \beta \PP\vartheta_i \biggr](s, a) = \biggl[ \sum_{k = 0}^K \sum_{i = 0}^{k} \gamma \beta   (\gamma \PP^{\pi^*}  )^{i}  \PP\vartheta_{k-i}  \biggr](s, a) \notag\\
&\qquad = \biggl[\sum_{i = 0}^K \gamma \beta    (\gamma \PP^{\pi^*})^i  \PP \sum_{k = i}^K   \vartheta_{k-i}\biggr](s, a)  \notag\\
&\qquad = \biggl[\sum_{i = 0}^K \gamma \beta    (\gamma \PP^{\pi^*})^i  \PP \sum_{k = i}^K  \Bigl ( \kl \bigl (\pi^* \,\|\, \pi_{\theta_{k-i}} \bigr)  - \kl \bigl(\pi^* \,\|\, \pi_{\theta_{k-i+1}} \bigr)\Bigr) \biggr](s, a)  \notag\\
&\qquad = \biggl[\sum_{i = 0}^K  \gamma \beta   (\gamma \PP^{\pi^*})^i  \PP   \bigl( \kl(\pi^* \,\|\, \pi_{\theta_{0}})  -  \kl(\pi^* \,\|\, \pi_{\theta_{K-i+1}}) \bigr) \biggr](s,a)\notag\\
&\qquad \leq  \biggl[ \sum_{i = 0}^K  \gamma \beta   (\gamma \PP^{\pi^*})^i  \PP   \kl(\pi^* \,\|\, \pi_{\theta_{0}}) \biggr] (s,a),
\#
where we use the definition of $\vartheta_{k-i}$ in \eqref{eq:def-deltak-lin} of Lemma \ref{lemma:a1} and the non-negativity of the KL divergence in the second equality and the last inequality, respectively. 
By plugging \eqref{eq:term3-bound-lin} and  \eqref{eq:delta-terms-bound-lin} into \eqref{eq:Q-bound3-lin}, we have
\#\label{eq:f-bd-lin}
& \sum_{k = 0}^K [Q^* - Q^{\pi_{\theta_{k+1}}} ](s,a) \notag\\
&\quad  \leq \biggl[ (I - \gamma \PP^{\pi^*})^{-1}  \biggl (  \sum_{k = 0}^K (\gamma\PP^{\pi^*})^{k+1} (Q^* - Q_{\omega_{0}}) + \sum_{k = 0}^K \sum_{i = 0}^{k} (\gamma \PP^{\pi^*}  )^{k-i}  \gamma\beta \PP \epsilon^{\rm a}_{i+1}\\
&\quad\quad + \sum_{k = 0}^K  \sum_{i = 0}^{k-1}  (\gamma \PP^{\pi^*})^{k-i} \epsilon^{\rm c}_{i+1} +  \sum_{k = 0}^K \gamma^{k+1} \PP^{\pi_{\theta_{k+1}}} (I - \gamma \PP^{\pi_{\theta_{k+1}}})^{-1}  \Bigl(\prod_{s = 1}^k  \PP^{\pi_{\theta_s}}\Bigr) e_1  \notag\\
&\quad\quad + \sum_{k = 0}^K  \PP^{\pi_{\theta_{k+1}}} (I - \gamma \PP^{\pi_{\theta_{k+1}}})^{-1} \sum_{\ell = 1}^k \gamma^{k-\ell+1}\Bigl(\prod_{s = \ell+1}^k  \PP^{\pi_{\theta_s}} \Bigr) \bigl(\gamma \beta \PP \epsilon_{\ell+1}^{\rm b} + (I - \gamma \PP^{\pi_{\theta_{\ell}}})\epsilon_{\ell}^{\rm c}\bigr) \biggr)  \biggr](s,a).  \notag\\
&\quad\quad + \sum_{i = 0}^K  (\gamma \PP^{\pi^*})^i  \gamma \beta \PP  \kl(\pi^* \,\|\, \pi_{\theta_{0}})   \notag
\#
We remark that $\epsilon^{\rm a}_{i+1} = \epsilon^{\rm b}_{i+1} = 0$ for any $i$ in the  linear actor-critic method.  Meanwhile, such terms is included in \eqref{eq:f-bd-lin} only aiming to generalize to the  deep neural actor-critic method. 
This concludes the proof in  {part 1}.

\vskip5pt
\noindent\textbf{Part 2.}  Recall that  $\rho$ is a state-action distribution satisfying \ref{assum:2-lin} of Assumption \ref{assum:cc-lin}.  In the sequel, we take the expectation over $\rho$ in \eqref{eq:f-bd-lin}  and upper bound each term.   
Recall that $\epsilon^{\rm a}_{i+1} = \epsilon^{\rm b}_{i+1} = 0$ for any $i$ in the linear actor-critic method.  Hence, we only need to consider terms in \eqref{eq:f-bd-lin} that do not involve $\epsilon^{\rm a}_{i+1}$ or $\epsilon^{\rm b}_{i+1}$.     We first upper bound terms on the RHS of \eqref{eq:f-bd-lin} that do not involve $\epsilon^{\rm c}_{i+1}$. More specifically, for any measure $\rho$ satisfying satisfying \ref{assum:2-lin} of Assumption \ref{assum:cc-lin}, we upper bound the following three terms, 
\#\label{eq:meiyoue}
& M_1 =  \EE_{\rho}\Bigl[ (I - \gamma \PP^{\pi^*})^{-1} \sum_{k = 0}^K (\gamma \PP^{\pi^*})^{k+1}(Q^* - Q_{\omega_0})  \Bigr],  \notag\\
& M_2 =  \EE_{\rho}\Bigl[ (I - \gamma \PP^{\pi^*})^{-1} \sum_{k = 0}^K \gamma^{k+1} \PP^{\pi_{\theta_{k+1}}} (I - \gamma \PP^{\pi_{\theta_{k+1}}})^{-1}  \Bigl(\prod_{s = 1}^k  \PP^{\pi_{\theta_s}}\Bigr) e_1 \Bigr], \notag \\
& M_3 =  \EE_{\rho}\Bigl[ (I - \gamma \PP^{\pi^*})^{-1} \sum_{i = 0}^K (\gamma \PP^{\pi^*})^{i}\gamma \beta \PP \kl (\pi^* \,\|\, \pi_{\theta_0})  \Bigr]. 
\#
We  upper bound $M_1$, $M_2$, and $M_3$ in the following lemma. 
\begin{lemma}\label{lemma:meiyou}
It holds that 
\$
& |M_1|  \leq 4 (1 - \gamma)^{-2} \cdot (\cR_{\max} + R), \qquad |M_2| \leq  (1 - \gamma)^{-3} \cdot (2R + \cR_{\max}), \notag\\
&  |M_3| \leq (1 - \gamma)^{-2} \cdot \log |\cA|\cdot K^{1/2} , 
\$
where $M_1$, $M_2$, and $M_3$ are defined in \eqref{eq:meiyoue}. 
\end{lemma}
\begin{proof}
See \S\ref{prf:lemma:meiyou} for a detailed proof. 
\end{proof}

Now, we  upper bound terms on the RHS of \eqref{eq:f-bd-lin} that   involve $\epsilon^{\rm c}_{i+1}$. More specifically, for any measure $\rho$ satisfying \ref{assum:2-lin} of Assumption \ref{assum:cc-lin}, we upper bound the following two terms,
\#\label{eq:youe}
& M_4 = \EE_{\rho}\Bigl[ (I - \gamma \PP^{\pi^*})^{-1} \sum_{k = 0}^K \sum_{i = 0}^k (\gamma \PP^{\pi^*})^{k-i}  \epsilon_{i+1}^{\rm c}  \Bigr],\\
& M_5 = \EE_{\rho}\biggl[ (I - \gamma \PP^{\pi^*})^{-1} \sum_{k = 0}^K  \PP^{\pi_{\theta_{k+1}}} (I - \gamma \PP^{\pi_{\theta_{k+1}}})^{-1} \sum_{\ell = 1}^k \gamma^{k-\ell+1}\Bigl(\prod_{s = \ell+1}^k  \PP^{\pi_{\theta_s}} \Bigr) (I- \gamma \PP^{\pi_{\theta_\ell}}) \epsilon_\ell^{\rm c}\biggr].  \notag
\#
We  upper bound $M_4$ and $M_5$ in the following lemma. 
\begin{lemma}\label{lemma:you}
It holds that 
\$
& |M_4|    \leq 3 K C_{\rho, \rho^*}  \cdot \varepsilon_Q , \qquad  |M_5|  \leq  K   C_{\rho, \rho^*} \cdot \varepsilon_Q.
 \$
where $M_4$ and $M_5$ are defined in \eqref{eq:youe}. 
\end{lemma}
\begin{proof}
See \S\ref{prf:lemma:you} for a detailed proof. 
\end{proof}

Now, by plugging Lemmas \ref{lemma:meiyou} and \ref{lemma:you} into \eqref{eq:f-bd-lin},   we have
\#\label{eq:f-bd-aa-lin}
& \EE_{\rho} \Bigl [\sum_{k = 0}^K Q^*(s,a) - Q^{\pi_{\theta_{k+1}}}(s,a) \Bigr]  \notag\\
& \qquad \leq 2(1 - \gamma)^{-3} \cdot \log |\cA|\cdot  K^{1/2}  +   4 K C_{\rho, \rho^*}   \cdot \varepsilon_Q + O(1). 
\#
Meanwhile, by changing measure from $\rho^*$ to $\rho_{k+1}$, it holds for any $k$ that
\#\label{eq:wocao1}
\EE_{\rho^*}[|\epsilon_{k+1}^{\rm c}|] \leq \sqrt{\EE_{\rho_{k+1}}\bigl[(\epsilon^{\rm c}_{k+1}(s,a))^2 \bigr]} \cdot \phi_{k+1}^*, 
\#
where $\phi_{k+1}^*$ is defined in Assumption \ref{assum:cc-lin}. 
Also, by Lemma \ref{lemma:l-critic},  with probability at least $1 - \delta$, it holds for any $k \in \{0,1,\ldots, K\}$ that
\#\label{eq:wocao2}
\sqrt{\EE_{\rho_{k+1}}\bigl[(\epsilon^{\rm c}_{k+1}(s,a))^2 \bigr] } =  O\bigl(  1 / (\sqrt N \sigma^*) \cdot \log (KN/\delta) \bigr).
\#
Now, by plugging \eqref{eq:wocao2} into \eqref{eq:wocao1}, combining the definition of $\varepsilon_Q = \max_k \EE_{\rho^*}[|\epsilon_{k+1}^{\rm c}(s,a)|]$,  it holds with probability at least $1 - \delta$ that
\#\label{eq:laji-bd-lin}
& \varepsilon_Q = O\bigl(  \phi^* / (\sqrt N \sigma^*) \cdot \log (KN/\delta) \bigr ). 
\#
Combining \eqref{eq:f-bd-aa-lin}, \eqref{eq:laji-bd-lin}, and the choices of parameters stated in the theorem  that
\$
N = \Omega \bigl( K   C_{\rho, \rho^*}^2 (\phi^* / \sigma^*)^2 \cdot  \log^2 (KN/\delta)  \bigr), 
\$
we have
\$
\EE_{\rho} \Bigl [\sum_{k = 0}^K Q^*(s,a) - Q^{\pi_{\theta_{k+1}}}(s,a) \Bigr] \leq \bigl(2(1-\gamma)^{-3} \log |\cA| + O(1)\bigr)\cdot K^{1/2},
\$
which concludes the proof of Theorem \ref{thm:regret-lin}.

\subsection{Proof of Theorem \ref{thm:regret}}\label{prf:thm:regret}

We follow the proof of Theorem \ref{thm:regret-lin} in \S\ref{prf:thm:regret-lin}.  
 Following similar arguments when deriving \eqref{eq:f-bd-lin} in \S\ref{prf:thm:regret-lin}, we have 
\#\label{eq:f-bd}
& \sum_{k = 0}^K [Q^* - Q^{\pi_{\theta_{k+1}}} ](s,a) \notag\\
&\quad  \leq \biggl[ (I - \gamma \PP^{\pi^*})^{-1} \cdot \biggl (  \sum_{k = 0}^K (\gamma\PP^{\pi^*})^{k+1} (Q^* - Q_{\omega_{0}}) + \sum_{k = 0}^K \sum_{i = 0}^{k} (\gamma \PP^{\pi^*}  )^{k-i} \cdot \gamma\beta \PP \epsilon^{\rm a}_{i+1}\\
&\quad\quad + \sum_{k = 0}^K  \sum_{i = 0}^{k-1}  (\gamma \PP^{\pi^*})^{k-i} \epsilon^{\rm c}_{i+1} + \sum_{i = 0}^K  (\gamma \PP^{\pi^*})^i \cdot \gamma \beta \PP\cdot  \kl(\pi^* \,\|\, \pi_{\theta_{0}}) \notag\\
&\quad\quad +  \sum_{k = 0}^K \gamma^{k+1} \PP^{\pi_{\theta_{k+1}}} (I - \gamma \PP^{\pi_{\theta_{k+1}}})^{-1}  \Bigl(\prod_{s = 1}^k  \PP^{\pi_{\theta_s}}\Bigr) e_1 \notag\\
&\quad\quad + \sum_{k = 0}^K  \PP^{\pi_{\theta_{k+1}}} (I - \gamma \PP^{\pi_{\theta_{k+1}}})^{-1} \sum_{\ell = 1}^k \gamma^{k-\ell+1}\Bigl(\prod_{s = \ell+1}^k  \PP^{\pi_{\theta_s}} \Bigr) \bigl(\beta\gamma \PP \epsilon_{\ell+1}^{\rm b} - (I - \gamma \PP^{\pi_{\theta_{\ell}}})\epsilon_{\ell}^{\rm c}\bigr) \biggr)  \biggr](s,a), \notag
\#
for any $(s,a)\in\cS\times \cA$.  Here $\epsilon^{\rm a}_{i+1}$, $\epsilon_{\ell+1}^{\rm b}$, $\epsilon^{\rm c}_{i+1}$, and $e_1$ are defined in \eqref{eq:def-pi-error-lin}, \eqref{eq:eek-b-lin}, \eqref{eq:error-Q-lin}, and \eqref{eq:def-ek-lin}, respectively. 

Now, it remains to upper bound each term on the RHS of \eqref{eq:f-bd}. 
We introduce the following error propagation lemma. 

\begin{lemma}\label{lemma:error-prop}
Suppose that
\#\label{eq:error-kl1}
& \EE_{\rho_{k}}\bigl[\bigl(f_{\theta_{k+1}}(s, a) - \tau_{k+1}\cdot(\beta^{-1}Q_{\omega_k}(s, a) - \tau^{-1}_kf_{\theta_k}(s, a))\bigr)^2\bigr]^{1/2} \leq \varepsilon_{k+1, f}. 
\#
Then, we have
\$
& \EE_{\nu^*} \bigl[ | \epsilon_{k+1}^{\rm a}(s) | \bigr] \leq \sqrt{2} \tau_{k+1}^{-1} \cdot \varepsilon_{k+1, f} \cdot (\phi^*_k + \psi^*_k), \quad   \EE_{\nu^*} \bigl[ | \epsilon_{k+1}^{\rm b}(s) | \bigr] \leq \sqrt{2} \tau_{k+1}^{-1} \cdot \varepsilon_{k+1, f} \cdot ( 1+ \psi^*_k ),
\$
where $\epsilon_{k+1}^{\rm a}$ and $\epsilon_{k+1}^{\rm b}$ are defined in \eqref{eq:def-pi-error-lin} and \eqref{eq:eek-b-lin}, respectively, $\phi^*_{k}$ and $\psi^*_{k}$ are defined in Assumption \ref{assum:cc}. 
\end{lemma}
\begin{proof}
See \S\ref{prf:lemma:error-prop} for a detailed proof. 
\end{proof}


Following from Lemma \ref{lemma:bd-node-diff}, with probability at least $1 - O(H_{\rm c})\exp(-\Omega(H_{\rm c}^{-1}m_{\rm c}))$, we have  $|Q_{\omega_0}| \leq 2$. Also, from the fact that $|\cR(s,a)| \leq \cR_{\max}$, we know that $|Q^*| \leq \cR_{\max}$. Therefore, for any measure $\rho$, we have
\#\label{eq:f-bd-1}
& \Bigl| \EE_{\rho}\Bigl[ (I - \gamma \PP^{\pi^*})^{-1} \sum_{k = 0}^K (\gamma \PP^{\pi^*})^{k+1}(Q^* - Q_{\omega_0})  \Bigr] \Bigr| \notag\\
& \qquad \leq \EE_{\rho}\Bigl[ (I - \gamma \PP^{\pi^*})^{-1} \sum_{k = 0}^K (\gamma \PP^{\pi^*})^{k+1}|Q^* - Q_{\omega_0}|  \Bigr]  \notag\\
& \qquad \leq  \cR_{\max} (1 - \gamma )^{-1} \sum_{k = 0}^K \gamma^{k+1}  \leq \cR_{\max} (1 - \gamma)^{-2}. 
\#
Also, by changing the index of summation, we have
\#\label{eq:mon5}
& \Bigl| \EE_{\rho}\Bigl[ (I - \gamma \PP^{\pi^*})^{-1} \sum_{k = 0}^K \sum_{i = 0}^k (\gamma \PP^{\pi^*})^{k-i}  \gamma \beta \PP \epsilon_{i+1}^{\rm a}  \Bigr] \Bigr|\notag\\
& \qquad = \Bigl| \EE_{\rho}\Bigl[  \sum_{k = 0}^K \sum_{i = 0}^k \sum_{j = 0}^\infty (\gamma \PP^{\pi^*})^{k-i+j}  \gamma \beta \PP \epsilon_{i+1}^{\rm a}  \Bigr] \Bigr|\notag\\
& \qquad = \Bigl| \EE_{\rho}\Bigl[  \sum_{k = 0}^K \sum_{i = 0}^k \sum_{t = k-i}^\infty (\gamma \PP^{\pi^*})^{t}  \gamma \beta \PP \epsilon_{i+1}^{\rm a}  \Bigr] \Bigr|\notag\\
& \qquad \leq   \sum_{k = 0}^K \sum_{i = 0}^k \sum_{t = k-i}^\infty \bigl| \EE_{\rho}\bigl[ (\gamma \PP^{\pi^*})^{t}  \gamma \beta \PP \epsilon_{i+1}^{\rm a}  \bigr] \bigr|, 
\#
where we expand $(I - \gamma \PP^{\pi^*})^{-1}$ into an infinite sum in the first equality. Further, by changing the measure of the expectation on the RHS of \eqref{eq:mon5}, we have
\#\label{eq:mon6}
   \sum_{k = 0}^K \sum_{i = 0}^k \sum_{t = k-i}^\infty \bigl| \EE_{\rho}\bigl[ (\gamma \PP^{\pi^*})^{t}  \gamma \beta \PP \epsilon_{i+1}^{\rm a}  \bigr] \bigr| \leq   \sum_{k = 0}^K \sum_{i = 0}^k \sum_{t = k-i}^\infty \beta\gamma^{t+1} c(t)\cdot  \EE_{\nu^*}[ |\epsilon_{i+1}^{\rm A }| ], 
\#
where $c(t)$ is defined in Assumption \ref{assum:cc}. Further, by Lemma \ref{lemma:error-prop} and interchanging the summation on the RHS of \eqref{eq:mon6}, we have
\#\label{eq:f-bd-2}
& \Bigl| \EE_{\rho}\Bigl[ (I - \gamma \PP^{\pi^*})^{-1} \sum_{k = 0}^K \sum_{i = 0}^k (\gamma \PP^{\pi^*})^{k-i}  \gamma \beta \PP \epsilon_{i+1}^{\rm a}  \Bigr] \Bigr|\notag\\
& \qquad \leq 2 \sum_{k = 0}^K \sum_{t = 0}^\infty \sum_{i = \max\{0, k-t\}}^k \beta\gamma^{t+1} c(t)\cdot \tau_{i+1}^{-1}\varepsilon_f (\phi_i^*+\psi_i^*)\notag\\
& \qquad \leq \sum_{k = 0}^K \sum_{t = 0}^\infty 4 kt  \gamma^{t+1} c(t)\cdot \varepsilon_f (\phi^* + \psi^*)\notag\\
& \qquad \leq \gamma \sum_{k = 0}^K 4   C_{\rho, \rho^*} \cdot  \varepsilon_f (\phi^* + \psi^*) \leq 2\gamma K^2 C_{\rho, \rho^*}  (\phi^* + \psi^*)\cdot \varepsilon_f , 
\#
where $\varepsilon_f = \max_{i} \EE_{\rho_{i}}[(f_{\theta_{i+1}}(s, a) - \tau_{i+1}\cdot(\beta^{-1}Q_{\omega_i}(s, a) - \tau^{-1}_i f_{\theta_i}(s, a)))^2]^{1/2}$, and $C_{\rho, \rho^*}$ is defined in Assumption \ref{assum:cc}. Here in the second inequality, we use the fact that $\tau_{i+1}^{-1} = (i+1)\cdot \beta^{-1}$, and $\phi_i^* \leq \phi^*$ and $\psi_i^* \leq \psi^*$ by Assumption \ref{assum:cc}.

By similar arguments in the derivation of \eqref{eq:f-bd-2}, we have
\#\label{eq:f-bd-3}
& \Bigl| \EE_{\rho}\Bigl[ (I - \gamma \PP^{\pi^*})^{-1} \sum_{k = 0}^K \sum_{i = 0}^{k-1} (\gamma \PP^{\pi^*})^{k-i}  \epsilon_{i+1}^{\rm c}  \Bigr] \Bigr| \leq  2(K+1) C_{\rho, \rho^*} \phi^*\cdot  \varepsilon_Q,  \\
& \Bigl| \EE_{\rho}\Bigl[ (I - \gamma \PP^{\pi^*})^{-1} \sum_{i = 0}^K (\gamma \PP^{\pi^*})^{i}\gamma \beta \PP \kl (\pi^* \,\|\, \pi_{\theta_0})  \Bigr] \Bigr| \leq \log |\cA|\cdot K^{1/2} (1 - \gamma)^{-2}, \notag\\
& \EE_{\rho}\Bigl[ (I - \gamma \PP^{\pi^*})^{-1} \sum_{k = 0}^K \gamma^{k+1} \PP^{\pi_{\theta_{k+1}}} (I - \gamma \PP^{\pi_{\theta_{k+1}}})^{-1}  \Bigl(\prod_{s = 1}^k  \PP^{\pi_{\theta_s}}\Bigr) e_1 \Bigr] \leq  (2 + \cR_{\max})\cdot (1 - \gamma)^{-3}, \notag
\#
where $\varepsilon_Q = \max_i \EE_{\rho^*}[|\epsilon_{i+1}^{\rm c}|]$. And we use the fact that $\beta = K^{1/2}$. 

Now, it remains to upper bound the last term on the RHS of \eqref{eq:f-bd}. We first consider the terms involving $\epsilon_{\ell+1}^{\rm b}$. We have
\#\label{eq:mon1}
& \EE_{\rho}\biggl[ (I - \gamma \PP^{\pi^*})^{-1} \sum_{k = 0}^K  \PP^{\pi_{\theta_{k+1}}} (I - \gamma \PP^{\pi_{\theta_{k+1}}})^{-1} \sum_{\ell = 1}^k \gamma^{k-\ell+1}\Bigl(\prod_{s = \ell+1}^k  \PP^{\pi_{\theta_s}} \Bigr) \beta\gamma \PP \epsilon_{\ell+1}^{\rm b}  \biggr]\notag\\
& \qquad =  \sum_{j = 0}^\infty  \sum_{i = 0}^\infty  \sum_{k = 0}^K \sum_{\ell = 1}^k  \EE_{\rho}\biggl[ (\gamma \PP^{\pi^*})^j  (\gamma \PP^{\pi_{\theta_{k+1}}})^{i+1} \gamma^{k-\ell}\Bigl(\prod_{s = \ell+1}^k  \PP^{\pi_{\theta_s}} \Bigr) \beta\gamma \PP \epsilon_{\ell+1}^{\rm b}   \biggr]\notag\\
& \qquad \leq  \beta\gamma  \sum_{k = 0}^K \sum_{\ell = 1}^k \sum_{j = 0}^\infty  \sum_{i = 0}^\infty  \gamma^{i+j+k-\ell+1} \cdot  \EE_{\rho^*}[   | \PP \epsilon_{\ell+1}^{\rm b} | ]\cdot c(i+j+k-\ell+1)\notag\\
& \qquad \leq 2 \gamma  \sum_{k = 0}^K \sum_{\ell = 1}^k \sum_{j = 0}^\infty  \sum_{i = 0}^\infty  \gamma^{i+j+k-\ell+1} \cdot  (\ell + 1)\varepsilon_f \cdot (1+\psi_\ell^*)  \cdot c(i+j+k-\ell+1), 
\#
where we expand $(I - \gamma \PP^{\pi^*})^{-1}$ and $(I - \gamma \PP^{\pi_{\theta_{k+1}}})^{-1}$ to infinite sums in the first equality,     change the measure of the expectation in the first inequality, and   use Lemma \ref{lemma:error-prop} in the last inequality. Now, by changing the index of the summation, we have 
\#\label{eq:mon2}
& \gamma  \sum_{k = 0}^K \sum_{\ell = 1}^k \sum_{j = 0}^\infty  \sum_{i = 0}^\infty  \gamma^{i+j+k-\ell+1} \cdot  (\ell + 1)\varepsilon_f \cdot(1+\psi_\ell^*)  \cdot c(i+j+k-\ell+1)\notag\\
& \qquad = \gamma  \sum_{k = 0}^K \sum_{\ell = 1}^k \sum_{j = 0}^\infty  \sum_{t = j+k-\ell+1}^\infty  \gamma^{t} \cdot  (\ell + 1)\varepsilon_f \cdot (1+\psi_\ell^*)  \cdot c(t)\notag\\
& \qquad \leq  \gamma  \sum_{k = 0}^K \sum_{j = 0}^\infty \sum_{t = j+1}^\infty  \sum_{\ell = \max\{0, j+k-t+1\}}^k   \gamma^{t} \cdot  (\ell + 1)\varepsilon_f \cdot (1+\psi^*)  \cdot c(t), 
\#
where we use the fact that $\psi^*_\ell \leq \psi^*$ from Assumption \ref{assum:cc} in the last inequality.  By further manipulating the order of summations of the RHS of \eqref{eq:mon2}, we have
\#\label{eq:mon3}
& \gamma  \sum_{k = 0}^K \sum_{j = 0}^\infty \sum_{t = j+1}^\infty  \sum_{\ell = \max\{0, j+k-t+1\}}^k   \gamma^{t} \cdot  (\ell + 1)\varepsilon_f (1+\psi^*)  \cdot c(t)\notag\\
& \qquad \leq \gamma  \sum_{k = 0}^K \sum_{j = 0}^\infty \Bigl( \sum_{t = j+1}^{j+k+1} (t-j)(2k + j-k+1) \cdot \gamma^t c(t)   + \sum_{t = j+k+2}^\infty k^2 \cdot \gamma^t c(t) \Bigr) \cdot \varepsilon_f (1+\psi^*)  \notag\\
& \qquad  =  \gamma  \sum_{k = 0}^K  \Bigl (  \sum_{t = 1}^\infty \sum_{j = \max\{0, t-k-1\}}^{t-1} (t-j)(2k + j-k+1) \cdot \gamma^t c(t) \notag\\
& \qquad \qquad\qquad + \sum_{t = k+2}^\infty \sum_{j = 1}^{t-k-2} k^2 \cdot \gamma^t c(t)   \Bigr) \cdot \varepsilon_f (1+\psi^*) \notag\\
& \qquad \leq 20 \gamma  \sum_{k = 0}^K  \Bigl (  \sum_{t = 1}^\infty k^2  \cdot t \gamma^t c(t) + \sum_{t = 1}^\infty  k^2 \cdot t \gamma^t c(t)    \Bigr) \cdot \varepsilon_f (1+\psi^*)\notag\\
& \qquad \leq 20 \gamma K  \cdot C_{\rho, \rho^*} \cdot \varepsilon_f (1+\psi^*),
\#
where we use the definition of $C_{\rho, \rho^*}$ from Assumption \ref{assum:cc} in the last inequality.  Now, combining \eqref{eq:mon1}, \eqref{eq:mon2}, and \eqref{eq:mon3}, we have
\#\label{eq:f-bd-4}
& \EE_{\rho}\biggl[ (I - \gamma \PP^{\pi^*})^{-1} \sum_{k = 0}^K  \PP^{\pi_{\theta_{k+1}}} (I - \gamma \PP^{\pi_{\theta_{k+1}}})^{-1} \sum_{\ell = 1}^k \gamma^{k-\ell+1}\Bigl(\prod_{s = \ell+1}^k  \PP^{\pi_{\theta_s}} \Bigr) \beta\gamma \PP \epsilon_{\ell+1}^{\rm b}  \biggr]\notag\\
& \qquad \leq 20 \gamma K  \cdot C_{\rho, \rho^*} \cdot \varepsilon_f \cdot (1+\psi^*). 
\#

Following from similar arguments when deriving \eqref{eq:f-bd-4}, we have
\#\label{eq:f-bd-5}
& \EE_{\rho}\biggl[ (I - \gamma \PP^{\pi^*})^{-1} \sum_{k = 0}^K  \PP^{\pi_{\theta_{k+1}}} (I - \gamma \PP^{\pi_{\theta_{k+1}}})^{-1} \sum_{\ell = 1}^k \gamma^{k-\ell+1}\Bigl(\prod_{s = \ell+1}^k  \PP^{\pi_{\theta_s}} \Bigr) (I- \gamma \PP^{\pi_{\theta_\ell}}) \epsilon_\ell^{\rm c}\biggr] \notag\\
& \qquad \leq 20 K  \cdot C_{\rho, \rho^*} \phi^* \cdot \varepsilon_Q,
\#

Now, by plugging \eqref{eq:f-bd-1}, \eqref{eq:f-bd-2}, \eqref{eq:f-bd-3}, \eqref{eq:f-bd-4}, and \eqref{eq:f-bd-5} into \eqref{eq:f-bd}, with probability at least $1 - O(H_{\rm c})\exp(-\Omega(H_{\rm c}^{-1}m_{\rm c}))$,  we have
\#\label{eq:f-bd-aa}
& \EE_{\rho} \Bigl [\sum_{k = 0}^K Q^*(s,a) - Q^{\pi_{\theta_{k+1}}}(s,a) \Bigr] \\
& \qquad \leq 2\log |\cA|\cdot  K^{1/2} (1 - \gamma)^{-3} + 60   K^2  C_{\rho, \rho^*}  (\phi^* + \psi^*+1)\cdot \varepsilon_f + 50 K  C_{\rho, \rho^*} \phi^*  \cdot \varepsilon_Q. \notag
\#
Meanwhile, following from Propositions \ref{prop:ac-bd}  and \ref{prop:acc-bd}, it holds with probability at least $1 - 1/K$ that
\#\label{eq:laji-bd}
& \varepsilon_f = O\bigl( R_{\rm a} N_{\rm a}^{-1/4} + R_{\rm a}^{4/3} m_{\rm a}^{-1/12} H_{\rm a}^{7/2} (\log m_{\rm a})^{1/2} ), \notag\\
& \varepsilon_Q = O\bigl( R_{\rm c} N_{\rm c}^{-1/4} + R_{\rm c}^{4/3} m_{\rm c}^{-1/12} H_{\rm c}^{7/2} (\log m_{\rm c})^{1/2} ). 
\#
Combining \eqref{eq:f-bd-aa}, \eqref{eq:laji-bd}, and the choices of parameters stated in the theorem, 
it holds with probability at least $1 - 1/K$ that 
\$
\EE_{\rho} \Bigl [\sum_{k = 0}^K Q^*(s,a) - Q^{\pi_{\theta_{k+1}}}(s,a) \Bigr] \leq \bigl(2(1-\gamma)^{-3} \log |\cA| + O(1)\bigr)\cdot K^{1/2},
\$
which concludes the proof of Theorem \ref{thm:regret}.

\section{Supporting Results} \label{sec:supp-res}

In this section, we provide some supporting results in the proof of Theorems \ref{thm:regret-lin} and \ref{thm:regret}.  We introduce Lemma \ref{lemma:3pt-lm}, which applies to both Algorithms \ref{algo:l-ac} and \ref{algo:n-ac}.  To introduce Lemma \ref{lemma:3pt-lm}, for any policy $\pi$ and action-value function $Q$, we define $\tilde \pi(a\given s) \propto \exp(\beta^{-1} Q(s,a)) \cdot \pi(a\given s)$. 

\begin{lemma}\label{lemma:3pt-lm}
For any $s\in\cS$ and $\pi^\dagger$, we have
\$
\beta^{-1} \cdot  \la Q(s,\cdot), \pi^\dagger(\cdot\given s) - \tilde \pi(\cdot\given s) \ra &   \leq \kl \bigl(\pi^\dagger(\cdot\given s) \,\|\, \pi (\cdot\given s) \bigr) -  \kl \bigl(\pi^\dagger(\cdot\given s) \,\|\, \tilde \pi(\cdot\given s) \bigr) \notag\\
&\quad  + \bigl\la \log \bigl( \tilde \pi(\cdot\,|\, s)/\pi(\cdot\,|\, s)\bigr) - \beta^{-1} \cdot Q(s,\cdot), \pi^\dagger(\cdot\given s) - \tilde \pi(\cdot\given s)  \bigr\ra.
\$
\end{lemma}
\begin{proof}
By calculation, it suffices to show that
\$
& \bigl \langle \log ( \tilde \pi(\cdot\,|\, s)/\pi(\cdot\,|\, s)),  \pi^\dagger(\cdot\,|\, s) - \tilde \pi(\cdot\given s) \bigr \rangle   \notag\\
& \qquad \leq  \kl(\pi^\dagger(\cdot\given  s)\,\|\, \pi(\cdot\given  s)) - \kl(\pi^\dagger(\cdot\given  s) \,\|\, \tilde \pi(\cdot \given  s)).
\$
By the definition of the KL divergence, it holds for any $s \in \cS$ that
\#\label{eq::NPG_err_pf1}
&\kl(\pi^\dagger(\cdot\,|\, s)\,\|\, \pi(\cdot\,|\, s)) - \kl(\pi^\dagger(\cdot\,|\, s) \,\|\, \tilde \pi (\cdot\,|\, s)) \notag\\
& \qquad  = \bigl \langle \log(\tilde \pi(\cdot\,|\, s)/\pi (\cdot\,|\, s)),  \pi^\dagger(\cdot\,|\, s) \bigr  \rangle .
\#
Meanwhile, for the term on the RHS of \eqref{eq::NPG_err_pf1}, we have
\#\label{eq::NPG_err_pf2}
&\bigl\langle \log(\tilde \pi(\cdot\,|\, s)/\pi_{\theta_{k}}(\cdot\,|\, s)),  \pi^\dagger(\cdot\,|\, s) \bigr  \rangle  \notag\\
&\qquad=\bigl \langle \log ( \tilde \pi(\cdot\,|\, s)/\pi(\cdot\,|\, s) ) ,  \pi^\dagger(\cdot\,|\, s) - \tilde \pi(\cdot\given s) \bigr \rangle  \notag\\
& \qquad \qquad + \bigl \langle \log(\tilde \pi(\cdot\,|\, s)/\pi(\cdot\,|\, s)),  \tilde \pi(\cdot\given s) \bigr\rangle \notag\\
&\qquad=  \bigl \langle \log ( \tilde \pi(\cdot\,|\, s)/\pi(\cdot\,|\, s)) ,  \pi^\dagger(\cdot\,|\, s) - \tilde \pi(\cdot\given s)\bigr \rangle  + \kl(\tilde \pi(\cdot \,|\, s) \,\|\, \pi(\cdot \,|\, s))\notag\\
&\qquad \geq \bigl \langle \log ( \tilde \pi(\cdot\,|\, s)/\pi(\cdot\,|\, s)) ,  \pi^\dagger(\cdot\,|\, s) - \tilde \pi(\cdot\given s) \bigr \rangle . 
\#
Combining \eqref{eq::NPG_err_pf1} and \eqref{eq::NPG_err_pf2}, we obtain that
\$
& \bigl \langle \log ( \tilde \pi(\cdot\,|\, s)/\pi(\cdot\,|\, s)),  \pi^\dagger(\cdot\,|\, s) - \tilde \pi(\cdot\given s) \bigr \rangle  \notag\\
& \qquad  \leq  \kl(\pi^\dagger(\cdot\,|\, s) \,\|\, \pi(\cdot\,|\, s)) - \kl(\pi^\dagger(\cdot\,|\, s) \,\|\, \tilde \pi(\cdot\,|\, s)),
\$
which concludes the proof of Lemma \ref{lemma:3pt-lm}.
\end{proof}

\subsection{Local Linearization of DNNs}

In the proofs of Propositions \ref{prop:ac-bd} and \ref{prop:acc-bd} in \S\ref{prf:prop:ac-bd} and \S\ref{prf:prop:acc-bd}, respectively, we utilize the linearization of DNNs. We introduce some related auxiliary results here. 
First, we define the linearization $\bar u_{\theta}$ of the DNN $u_{\theta}\in \cU(w, H, R)$ as follows,  
\$
\bar u_{\theta}(\cdot) = u_{\theta_0}(\cdot) + (\theta - \theta_0)^\top \nabla_{\theta_0} u_{\theta}(\cdot), 
\$ 
where $\theta_0$ is the initialization of $u_{\theta}$.   The following lemmas characterize the linearization error.

\begin{lemma}\label{lemma:bd-deriv}
Suppose that $H = O(m^{1/12} R^{-1/6}(\log m)^{-1/2})$ and $m = \Omega(d^{3/2}R^{-1} H ^{-3/2} \cdot \log(m^{1/2}/R)^{3/2})$. Then with probability at least $1 - \exp(-\Omega(R^{2/3} m^{2/3} H ))$ over the random initialization $\theta_0$, it holds for any $\theta\in\cB(\theta_0, R)$ and any  $(s,a)\in\cS\times \cA$ that
\$
\| \nabla_\theta u_\theta(s,a) - \nabla_\theta u_{\theta_0}(s,a) \|_2 = O\bigl(R^{1/3} m^{-1/6} H^{5/2}(\log m)^{1/2}\bigr)
\$
and 
\$
\| \nabla_\theta u_\theta(s,a) \|_2 = O(H). 
\$
\end{lemma}
\begin{proof}
See the proof of Lemma A.5 in \cite{gao2019convergence} for a detailed proof. 
\end{proof}

\begin{lemma}\label{lemma:bd-func}
Suppose that $H = O(m^{1/12} R^{-1/6}(\log m)^{-1/2})$ and $m = \Omega(d^{3/2}R^{-1} H ^{-3/2} \cdot \log(m^{1/2}/R)^{3/2} )$. Then with probability at least $1 - \exp(-\Omega(R^{2/3} m^{2/3} H ))$ over the random initialization $\theta_0$, it holds for any $\theta\in\cB(\theta_0, R)$ and any $(s,a)\in\cS\times \cA$ that
\$
  | u_\theta(s,a) -  \bar u_{\theta}(s,a) | = O\bigl(R^{4/3} m^{-1/6} H^{5/2}(\log m)^{1/2}\bigr). 
\$ 
\end{lemma}
\begin{proof}
Recall that 
\$
\bar u_{\theta}(s,a) = u_{\theta_0}(s,a) + (\theta - \theta_0)^\top \nabla_\theta u_{\theta_0}(s,a).
\$
By mean value theorem, there exists $t\in[0,1]$, which depends on $\theta$ and $(s,a)$, such that
\$
u_\theta(s,a) -  \bar u_{\theta}(s,a) = (\theta-\theta_0)^\top \bigl( \nabla_\theta u_{\theta_0 + t (\theta-\theta_0)}(s,a) - \nabla_\theta u_{\theta_0}(s,a) \bigr). 
\$
Further by Lemma \ref{lemma:bd-deriv}, we have
\$
| u_\theta(s,a) -  \bar u_{\theta}(s,a)  | & \leq  \|\theta-\theta_0\|_2 \cdot  \bigl\| \nabla_\theta u_{\theta_0 + t\cdot (\theta-\theta_0)}(s,a) - \nabla_\theta u_{\theta_0}(s,a) \bigr\|_2 \notag\\
& = O\bigl(R^{4/3} m^{-1/6} H^{5/2}(\log m)^{1/2}\bigr),
\$
where we use Cauchy-Schwarz inequality in the first inequality. This concludes the proof of Lemma \ref{lemma:bd-func}. 
\end{proof}

We denote by $x^{(h)}$ the output of the $h$-th layer of the DNN $u_\theta\in \cU(m, H, R)$, and $x^{(h), 0}$ the output of the $h$-th layer of the DNN $u_{\theta_0}\in \cU(m, H, R)$.
The following lemma upper bounds  the distance between $x^{(h)}$ and $x^{(h), 0}$. 

\begin{lemma}\label{lemma:bd-node-diff}  
With probability at least $1 - \exp(-\Omega(  R^{2/3} m^{2/3} H  ))$ over the random initialization $\theta_0$, for any $\theta\in\cB(\theta_0, R)$ and any $h\in[H]$, we have
\$
\|x^{(h)} - x^{(h),0}\|_2 = O\bigl(  R H^{5/2} m^{-1/2} (\log m)^{1/2} \bigr). 
\$
Also, with probability at least $1 - O(H) \exp(-\Omega(  H^{-1}m  ))$ over the random initialization $\theta_0$, for any $\theta\in\cB(\theta_0, R)$ and any $h\in[H]$, it holds that
\$
2/3 \leq  \|x^{(h)} \|_2 \leq 4/3. 
\$
\end{lemma}
\begin{proof}
The first inequality follows from Lemma A.5 in \cite{gao2019convergence}, and the second inequality follows from Lemma 7.1 in \cite{allen2018convergence}. 
\end{proof}

\section{Proofs of Propositions}

\subsection{Proof of Proposition \ref{prop:energy-based}}\label{prf:prop:energy-based}
The proof follows the proof of Proposition 3.1 in \cite{liu2019neural}.  First, we write the update $\tilde \pi_{{k+1}} \gets \argmax_{\pi}  \EE_{\nu_k}  [ \la Q_{\omega_k}(s, \cdot), \pi(\cdot \given s) \ra   -  \beta  \cdot \kl(\pi(\cdot \given s)  \,\|\, \pi_{\theta_{k}}(\cdot \given s))  ]$ as a constrained optimization problem in the following way, 
\$
&\max_{\pi} ~\EE_{\nu_k}\bigl[\la \pi(\cdot\given s), Q_{\omega_k}(s, \cdot)\ra - \beta \cdot  \kl(\pi(\cdot\given s)\,\|\,\pi_{\theta_k}(\cdot\given s))\bigr] \\
&\text{ s.t.}~\sum_{a\in \cA}\pi(a\given s) = 1,\qquad \text{for any}~ s \in \cS.
\$
We consider the Lagrangian of the above program, 
\$
 \int_{s\in\cS}\Bigl(\la \pi(\cdot\given s), Q_{\omega_k}(s,\cdot)\ra - \beta \cdot \kl \bigl(\pi(\cdot\given s)\,\|\,\pi_{\theta_k}(\cdot\given s)\bigr)\Bigr)\ud \nu_k(s) + \int_{s\in \cS} \Bigl(\sum_{a\in \cA}\pi(a\,|\, s) - 1\Bigr)\ud \lambda(s),
\$
where $\lambda(\cdot)$ is the dual parameter, which is a function on $\cS$. Now, by plugging in 
\$
\pi_{\theta_k}(a\given s) = \frac{ \exp(\tau_k^{-1} f_{\theta_k}(s,a))}{\sum_{a'\in\cA}\exp(\tau_k^{-1} f_{\theta_k}(s,a'))},
\$ 
we have the following optimality condition, 
\$
Q_{\omega_k}(s,a) + \beta \tau_k^{-1} f_{\theta_k}(s,a) - \beta \cdot\Big(\log\Bigl(\sum_{a'\in\cA}\exp(\tau_k^{-1}f_{\theta_k}(s,a'))\Bigr) + \log\pi(a\,|s) + 1\Bigr) + \frac{\lambda(s)}{\nu_k(s)} = 0,
\$
for any $(s,a)\in\cS\times\cA$. Note that $\log(\sum_{a'\in\cA}\exp(\tau_k^{-1} f_{\theta_k}(s,a')))$ is only a function of $s$. Thus, we have 
\$
\hat{\pi}_{k+1}(a\given s) \propto \exp(\beta^{-1}Q_{\omega_k}(s,a) + \tau_k^{-1} f_{\theta_k}(s,a))
\$
for any $(s,a)\in\cS\times\cA$, which concludes the proof of Proposition \ref{prop:energy-based}.

\subsection{Proof of Proposition \ref{prop:ac-bd}}\label{prf:prop:ac-bd}
We define the local linearization of $f_{\theta}$ as follows, 
\#\label{eq:lin-actor}
\bar f_{\theta} = f_{\theta_0} + (\theta - \theta_0)^\top \nabla_{\theta_0} f_{\theta}. 
\#
Meanwhile, we denote by 
\#\label{eq:def-gs}
& g_n = \bigl( f_{\theta(n)} - \tilde \tau\cdot (\beta^{-1}Q_\omega + \tau^{-1}f_\theta) \bigr)\cdot \nabla_\theta f_{\theta(n)}, && g_n^e = \EE_{\rho_{\pi_\theta}}[ g_n ], \notag\\
& \bar g_n = \bigl( \bar f_{\theta(n)} - \tilde \tau\cdot (\beta^{-1}Q_\omega + \tau^{-1}f_\theta) \bigr)\cdot \nabla_\theta f_{\theta_0}, && \bar g_n^e = \EE_{\rho_{\pi_\theta}}[ \bar g_n ],\notag\\
& g_* = \bigl( f_{\theta_*} - \tilde \tau\cdot (\beta^{-1}Q_\omega + \tau^{-1}f_\theta) \bigr)\cdot \nabla_\theta f_{\theta_*}, && g_*^e = \EE_{\rho_{\pi_\theta}}[ g_* ],\notag\\
& \bar g_* = \bigl( \bar f_{\theta_*} - \tilde \tau\cdot (\beta^{-1}Q_\omega + \tau^{-1}f_\theta) \bigr)\cdot \nabla_\theta f_{\theta_0}, && \bar g_*^e = \EE_{\rho_{\pi_\theta}}[ \bar g_* ],
\#
where $\theta_*$ satisfies that
\#\label{eq:def-theta-star}
\theta_* = \Gamma_{\cB(\theta_0, R_{\rm a} )}(\theta_* - \alpha \cdot \bar g_*^e ). 
\#
By Algorithm \ref{algo:actor}, we know that
\#\label{eq:nop1}
\theta(n+1) = \Gamma_{\cB(\theta_0, R_{\rm a} )}(\theta(n) - \alpha \cdot  g_n ). 
\#
By \eqref{eq:def-theta-star} and \eqref{eq:nop1}, we have
\#\label{eq:ac-bd1}
& \EE_{\rho_{\pi_\theta}}\bigl[ \| \theta(n+1) - \theta_* \|_2^2 \given \theta(n) \bigr]\notag\\
& \qquad = \EE_{\rho_{\pi_\theta}}\bigl[ \| \Gamma_{\cB(\theta_0, R_{\rm a})}(\theta(n) - \alpha \cdot  g_n ) - \Gamma_{\cB(\theta_0, R_{\rm a})}(\theta_* - \alpha \cdot \bar g_*^e ) \|_2^2 \given \theta(n) \bigr]\notag\\
& \qquad \leq \EE_{\rho_{\pi_\theta}}\bigl[ \| (\theta(n) - \alpha \cdot  g_n ) - (\theta_* - \alpha \cdot \bar g_*^e ) \|_2^2 \given \theta(n) \bigr]\notag\\
& \qquad = \|\theta(n) - \theta_*\|_2^2  + 2\alpha \cdot \underbrace{ \la \theta_* - \theta(n), g_n^e - \bar g_*^e \ra }_{\textstyle\rm (i)}+ \alpha^2 \cdot \underbrace{\EE_{\rho_{\pi_\theta}} \bigl[\|g_n - \bar g_*^e\|_2^2\given \theta(n)\bigr]}_{\textstyle\text{ (ii)}}, 
\#
where we use the fact that $\Gamma_{\cB(\theta_0, R_{\rm a})}$ is a contraction mapping in the first inequality.  We upper bound term (i) and term (ii) on the RHS of \eqref{eq:ac-bd1} in the sequel. 

\vskip5pt
\noindent\textbf{Upper Bound of Term (i).} By Cauchy–Schwarz inequality, it holds that
\#\label{eq:ac-bd-t1-a}
 \la \theta_* - \theta(n), g_n^e - \bar g_*^e \ra & = \la  \theta_* - \theta(n), g_n^e - \bar g_n^e \ra + \la  \theta_* - \theta(n), \bar g_n^e - \bar g_*^e \ra \notag\\
& \leq \| \theta_* - \theta(n)\|_2\cdot \| g_n^e - \bar g_n^e \|_2 + \la  \theta_* - \theta(n), \bar g_n^e - \bar g_*^e  \ra\notag\\
& \leq 2R_{\rm a} \cdot \|g_n^e - \bar g_n^e\|_2 + \la  \theta_* - \theta(n), \bar g_n^e - \bar g_*^e \ra,
\#
where we use the fact that $\theta(n), \theta_* \in \cB(\theta_0, R_{\rm a})$ in the last inequality. Further, by the definitions in \eqref{eq:def-gs}, it holds that
\#\label{eq:ac-bd-t1-b}
 \la \theta_* - \theta(n), \bar g_n^e - \bar g_*^e \ra &  = \EE_{\rho_{\pi_\theta}} \bigl[ (\bar f_{\theta(n)} - \bar f_{\theta_*}) \cdot \la \theta_* - \theta(n), \nabla_\theta f_{\theta_0} \ra \bigr]\notag\\
&  = \EE_{\rho_{\pi_\theta}} \bigl[ (\bar f_{\theta(n)} - \bar f_{\theta_*}) \cdot (\bar f_{\theta_*} - \bar f_{\theta(n)} ) \bigr]\notag\\
&  = -\EE_{\rho_{\pi_\theta}} \bigl[ (\bar f_{\theta(n)} - \bar f_{\theta_*})^2 \bigr], 
\#
where we use \eqref{eq:lin-actor} in the second equality.  Combining \eqref{eq:ac-bd-t1-a} and \eqref{eq:ac-bd-t1-b}, we obtain the following upper bound of term (i), 
\#\label{eq:ac-bd-t1}
& \la \theta_* - \theta(n), g_n^e - \bar g_*^e \ra \leq 2R_{\rm a} \cdot \|g_n^e - \bar g_n^e\|_2 - \EE_{\rho_{\pi_\theta}} \bigl[ (\bar f_{\theta(n)} - \bar f_{\theta_*})^2 \bigr]. 
\#

\vskip5pt
\noindent\textbf{Upper Bound of Term (ii).} 
We now upper bound term (ii) on the RHS of \eqref{eq:ac-bd1}.  It holds by Cauchy-Schwarz inequality that
\#\label{eq:bd-term-2-haha}
 \EE_{\rho_{\pi_\theta}} \bigl[\|g_n - \bar g_*^e\|_2^2\given \theta(n)\bigr] & \leq 2\EE_{\rho_{\pi_\theta}} \bigl[\|g_n -  g_n^e\|_2^2\given \theta(n)\bigr] + 2 \|g_n^e - \bar g_*^e\|_2^2\notag\\
& \leq 2\underbrace{\EE_{\rho_{\pi_\theta}} \bigl[\|g_n -  g_n^e\|_2^2\given \theta(n)\bigr]}_{\textstyle\text{(ii).a}} + 4 \underbrace{\|g_n^e - \bar g_n^e\|_2^2}_{\textstyle\text{(ii).b}} + 4 \underbrace{\|\bar g_n^e - \bar g_*^e\|_2^2}_{\textstyle\text{(ii).c}}. 
\#
We upper bound term (ii).a, term (ii).b, and term (ii).c in the sequel. 

\vskip5pt
\noindent\textbf{Upper Bound of Term (ii).a.} Note that
\#\label{eq:suibian4}
\EE_{\rho_{\pi_\theta}} \bigl[\|g_n -  g_n^e\|_2^2\given \theta(n)\bigr] = \EE_{\rho_{\pi_\theta}} \bigl[\|g_n\|_2^2 -  \|g_n^e\|_2^2\given \theta(n)\bigr] \leq \EE_{\rho_{\pi_\theta}} \bigl[\|g_n\|_2^2 \given \theta(n)\bigr]. 
\#
Meanwhile, by the definition of $g_n$ in \eqref{eq:def-gs}, it holds that 
\#\label{eq:suibian3}
\|g_n\|_2^2 & = \bigl( f_{\theta(n)} - \tilde \tau\cdot (\beta^{-1}Q_\omega + \tau^{-1}f_\theta) \bigr)^2\cdot \|\nabla_\theta f_{\theta(n)}\|_2^2. 
\#
We first upper bound $ f_{\theta}$ as follows, 
\$
 f_{\theta}^2 = x^{(H_{\rm a})\top} bb^\top x^{(H_{\rm a})} = x^{(H_{\rm a})\top} x^{(H_{\rm a})} = \|x^{(H_{\rm a})}\|_2^2,
\$
where $x^{(H_{\rm a})}$ is the output of the $H_{\rm a}$-th layer of the DNN $f_{\theta}$. Further combining Lemma \ref{lemma:bd-node-diff}, it holds with probability at least $1 - O(H_{\rm a}) \exp(-\Omega(  H_{\rm a}^{-1} m_{\rm a}  ))$ that 
\#\label{eq:suibian1}
| f_{\theta}| \leq 2. 
\#
Following from similar arguments, with probability at least $1 - O(H_{\rm a}) \exp(-\Omega(  H_{\rm a}^{-1} m_{\rm a}  ))$, we have
\#\label{eq:suibian2}
| Q_{\omega}| \leq 2, \qquad | f_{\theta(n)}| \leq 2. 
\#
Combining Lemma \ref{lemma:bd-deriv},  \eqref{eq:suibian4}, \eqref{eq:suibian3}, \eqref{eq:suibian1}, and \eqref{eq:suibian2}, it holds with probability at least $1 - \exp(  - \Omega(  R_{\rm a}^{2/3} m_{\rm a}^{2/3} H_{\rm a}   ))$ that
\#\label{eq:bd-term-2-aa}
\EE_{\rho_{\pi_\theta}} \bigl[\|g_n -  g_n^e\|_2^2\given \theta(n)\bigr]  =  O(H_{\rm a}^2 ), 
\#
which establishes an upper bound of term (ii).a.

\vskip5pt
\noindent\textbf{Upper Bound of Term (ii).b.} It holds that
\#\label{eq:bd-term-2-b}
\|g_n^e - \bar g_n^e\|_2 & = \bigl\| \EE_{\rho_{\pi_\theta}} \bigl[ \bigl( f_{\theta(n)} - \tilde \tau\cdot (\beta^{-1}Q_\omega + \tau^{-1}f_\theta) \bigr)\cdot \nabla_\theta f_{\theta(n)} \notag \\
& \qquad\qquad -  \bigl( \bar f_{\theta(n)} - \tilde \tau\cdot (\beta^{-1}Q_\omega + \tau^{-1}f_\theta) \bigr)\cdot \nabla_\theta f_{\theta_0}\bigr]\bigr\|_2\notag\\
& \leq \EE_{\rho_{\pi_\theta}} \bigl[ \| f_{\theta(n)}\nabla_\theta f_{\theta(n)} - \bar f_{\theta(n)} \nabla_\theta f_{\theta_0} \|_2\bigr] \notag\\
& \qquad + \tilde \tau\cdot \EE_{\rho_{\pi_\theta}} \bigl[ \|(\beta^{-1}Q_\omega + \tau^{-1}f_\theta) \cdot (\nabla_\theta f_{\theta_0} - \nabla_\theta f_{\theta(n)})\|_2\bigr]\notag\\
& \leq \EE_{\rho_{\pi_\theta}} \bigl[ \| f_{\theta(n)}\nabla_\theta f_{\theta_0} - \bar f_{\theta(n)} \nabla_\theta f_{\theta_0}\|_2 \bigr] + \EE_{\rho_{\pi_\theta}} \bigl[ \| f_{\theta(n)}\nabla_\theta f_{\theta(n)} - f_{\theta(n)}\nabla_\theta f_{\theta_0} \|_2\bigr] \\
&\qquad +  \EE_{\rho_{\pi_\theta}} \bigl[ \|\tilde \tau\cdot(\beta^{-1}Q_\omega + \tau^{-1}f_\theta) \cdot (\nabla_\theta f_{\theta_0} - \nabla_\theta f_{\theta(n)})\|_2\bigr]. \notag
\#
We upper bound the three terms on the RHS of \eqref{eq:bd-term-2-b} in the sequel, respectively. 

For the term $ \| f_{\theta(n)}\nabla_\theta f_{\theta_0} - \bar f_{\theta(n)} \nabla_\theta f_{\theta_0}\|_2 $ on the RHS of \eqref{eq:bd-term-2-b}, following from Lemmas \ref{lemma:bd-deriv} and \ref{lemma:bd-func}, it holds with probability at least $1 - \exp(-\Omega(R_{\rm a}^{2/3} m_{\rm a}^{2/3} H_{\rm a} ))$ that 
\#\label{eq:bd-term-2-b-1}
\| f_{\theta(n)}\nabla_\theta f_{\theta_0} - \bar f_{\theta(n)} \nabla_\theta f_{\theta_0}\|_2 = O\bigl( R_{\rm a}^{4/3} m_{\rm a}^{-1/6} H_{\rm a}^{7/2} (\log m_{\rm a})^{1/2}  \bigr). 
\#

For the term $ \| f_{\theta(n)}\nabla_\theta f_{\theta(n)} - f_{\theta(n)}\nabla_\theta f_{\theta_0} \|_2 $ on the RHS of \eqref{eq:bd-term-2-b}, following from \eqref{eq:suibian2} and Lemma \ref{lemma:bd-deriv}, with probability at least $1 - \exp(-\Omega(R_{\rm a}^{2/3} m_{\rm a}^{2/3} H_{\rm a} ))$, we have 
\#\label{eq:bd-term-2-b-3}
\| f_{\theta(n)}\nabla_\theta f_{\theta(n)} - f_{\theta(n)}\nabla_\theta f_{\theta_0} \|_2   = O\bigl( R_{\rm a}^{1/3} m_{\rm a}^{-1/6} H_{\rm a}^{5/2} (\log m_{\rm a})^{1/2}  \bigr). 
\#

For the term $ \|\tilde \tau\cdot(\beta^{-1}Q_\omega + \tau^{-1}f_\theta)\cdot (\nabla_\theta f_{\theta_0} - \nabla_\theta f_{\theta(n)})\|_2 $ on the RHS of \eqref{eq:bd-term-2-b},  we first upper bound $\tilde \tau\cdot(\beta^{-1}Q_\omega + \tau^{-1}f_\theta)$ as follows,
\$
 |\tilde \tau\cdot(\beta^{-1}Q_\omega + \tau^{-1}f_\theta)|  \leq 2,
\$
where we use \eqref{eq:suibian1}, \eqref{eq:suibian2}, and the fact that $\tilde \tau^{-1} = \beta^{-1} + \tau^{-1}$. Further combining Lemma \ref{lemma:bd-deriv}, it holds with probability at least $1 - \exp(-\Omega(R_{\rm a}^{2/3} m_{\rm a}^{2/3} H_{\rm a} ))$ that
\#\label{eq:bd-term-2-b-4}
 \|\tilde \tau\cdot(\beta^{-1}Q_\omega + \tau^{-1}f_\theta)\cdot (\nabla_\theta f_{\theta_0} - \nabla_\theta f_{\theta(n)})\|_2 = O\bigl( R_{\rm a}^{1/3} m_{\rm a}^{-1/6} H_{\rm a}^{5/2} (\log m_{\rm a})^{1/2}  \bigr). 
\#

Now, combining \eqref{eq:bd-term-2-b}, \eqref{eq:bd-term-2-b-1}, \eqref{eq:bd-term-2-b-3}, and \eqref{eq:bd-term-2-b-4}, it holds with probability at least $1 - \exp(-\Omega(R_{\rm a}^{2/3} m_{\rm a}^{2/3} H_{\rm a} ))$ that
\#\label{eq:bd-term-2-bb}
\|g_n^e - \bar g_n^e\|_2^2 = O\bigl( R_{\rm a}^{8/3} m_{\rm a}^{-1/3} H_{\rm a}^{7} \log m_{\rm a}  \bigr), 
\#
which establishes an upper bound of term (ii).b.

\vskip5pt
\noindent\textbf{Upper Bound of Term (ii).c.}  It holds that
\$
\|\bar g_n^e - \bar g_*^e\|_2^2 = \bigl\|\EE_{\rho_{\pi_\theta}}[  (\bar f_{\theta(n)} - \bar f_{\theta_*}) \nabla_\theta f_{\theta_0}  ]\bigr\|_2^2 \leq \EE_{\rho_{\pi_\theta}}\bigl [  (\bar f_{\theta(n)} - \bar f_{\theta_*})^2 \cdot \|\nabla_\theta f_{\theta_0}\|_2^2  \bigr]. 
\$
Further combining Lemma \ref{lemma:bd-deriv}, it holds with probability at least $1 - \exp(-\Omega(R_{\rm a}^{2/3} m_{\rm a}^{2/3} H_{\rm a} ))$ that
\#\label{eq:bd-term-2-cc}
\|\bar g_n^e - \bar g_*^e\|_2^2  \leq O(H_{\rm a}^2) \cdot \EE_{\rho_{\pi_\theta}} \bigl [  (\bar f_{\theta(n)} - \bar f_{\theta_*})^2   \bigr], 
\#
which establishes an upper bound of term (ii).c.

\vskip5pt
Now, combining \eqref{eq:bd-term-2-haha}, \eqref{eq:bd-term-2-aa}, \eqref{eq:bd-term-2-bb}, and \eqref{eq:bd-term-2-cc}, we have
\#\label{eq:ac-bd-t2}
\EE_{\rho_{\pi_\theta}} \bigl[ \|g_n - \bar g_*^e\|_2^2 \given \theta(n) \bigr] & \leq O\bigl( R_{\rm a}^{8/3} m_{\rm a}^{-1/3} H_{\rm a}^{7} \log m_{\rm a}  \bigr) + O(H_{\rm a}^2) \cdot \EE_{\rho_{\pi_\theta}} \bigl [  (\bar f_{\theta(n)} - \bar f_{\theta_*})^2   \bigr], 
\#
which is an upper bound of term (ii) on the RHS of \eqref{eq:ac-bd1}. 

\vskip5pt
By plugging the upper bound of term (i) in \eqref{eq:ac-bd-t1} and the upper bound of term (ii) in \eqref{eq:ac-bd-t2} into \eqref{eq:ac-bd1}, combining \eqref{eq:bd-term-2-bb}, with probability at least $1 - \exp(-\Omega(R_{\rm a}^{2/3} m_{\rm a}^{2/3} H_{\rm a} ))$, we have
\#\label{eq:ac-bd-t}
& \EE_{\rho_{\pi_\theta}} \bigl[ \| \theta(n+1) - \theta_* \|_2^2 \given \theta(n) \bigr] \notag\\
& \qquad \leq \| \theta(n) - \theta_* \|_2^2 + 2\alpha\cdot \Bigl( O\bigl( R_{\rm a}^{7/3} m_{\rm a}^{-1/6} H_{\rm a}^{7/2} (\log m_{\rm a})^{1/2}  \bigr) - \EE_{\rho_{\pi_\theta}} \bigl[ (\bar f_{\theta(n)} - \bar f_{\theta_*})^2 \bigr] \Bigr)\\
& \qquad\qquad +  \alpha^2 \cdot \Bigl(O\bigl( R_{\rm a}^{8/3} m_{\rm a}^{-1/3} H_{\rm a}^{7} \log m_{\rm a}  \bigr) + O(H_{\rm a}^2) \cdot \EE_{\rho_{\pi_\theta}} \bigl [  (\bar f_{\theta(n)} - \bar f_{\theta_*})^2   \bigr] \Bigr) \notag. 
\#
Rearranging terms in \eqref{eq:ac-bd-t}, it holds with probability at least $1 - \exp(-\Omega(R_{\rm a}^{2/3} m_{\rm a}^{2/3} H_{\rm a} ))$ that 
\#\label{eq:ff10}
&(2\alpha - \alpha^2\cdot O(H_{\rm a}^2) )\cdot\EE_{\rho_{\pi_\theta}} \bigl [  (\bar f_{\theta(n)} - \bar f_{\theta_*})^2   \bigr] \notag\\
&\qquad \leq   \| \theta(n) - \theta_* \|_2^2  - \EE_{\rho_{\pi_\theta}} \bigl[ \| \theta(n+1) - \theta_* \|_2^2 \given \theta(n) \bigr]   + \alpha\cdot O\bigl( R_{\rm a}^{8/3} m_{\rm a}^{-1/6} H_{\rm a}^{7} \log m_{\rm a}  \bigr). 
\#
By telescoping the sum and using Jensen's inequality in \eqref{eq:ff10}, we have
\$
& \EE_{\rho_{\pi_\theta}} \bigl [  ( \bar f_{\bar \theta} - \bar f_{\theta_*})^2   \bigr]  \leq \frac{1}{N_{\rm a}}\cdot \sum_{n = 0}^{N_{\rm a} - 1} \EE_{\rho_{\pi_\theta}} \bigl [  (\bar f_{\theta(n)} - \bar f_{\theta_*})^2   \bigr]\notag\\
& \qquad \leq 1/N_{\rm a}\cdot \bigl(2\alpha - \alpha^2 \cdot O(H_{\rm a}^2) \bigr)^{-1}\cdot \bigl( \| \theta_{0} - \theta_* \|_2^2 + \alpha N_{\rm a}\cdot O( R_{\rm a}^{8/3} m_{\rm a}^{-1/6} H_{\rm a}^{7} \log m_{\rm a}  )  \bigr)\notag\\
& \qquad \leq N_{\rm a}^{-1/2}\cdot  \| \theta_{0} - \theta_* \|_2^2  +  O( R_{\rm a}^{8/3} m_{\rm a}^{-1/6} H_{\rm a}^{7} \log m_{\rm a}  ), 
\$
where the last line comes from the choices that $\alpha = N_{\rm a}^{-1/2}$ and $H_{\rm a} = O(N_{\rm a}^{1/4})$.  Further combining Lemma \ref{lemma:bd-func} and using triangle inequality, we have
\#\label{eq:ff11}
\EE_{\rho_{\pi_\theta}} \bigl [  ( f_{\bar \theta} - \bar f_{\theta_*})^2   \bigr] &  = O( R_{\rm a}^2 N_{\rm a}^{-1/2} + R_{\rm a}^{8/3} m_{\rm a}^{-1/6} H_{\rm a}^{7} \log m_{\rm a}  ). 
\#
By the definition of $\theta_*$ in \eqref{eq:def-theta-star}, we know that
\#\label{eq:aa1}
\la \bar g_*^e, \theta - \theta_*\ra \geq 0, \qquad \text{for any } \theta\in \cB(\theta_0, R_{\rm a}). 
\#
By plugging the definition of $\bar g_*^e$ into \eqref{eq:aa1}, we have
\$
\EE_{\rho_{\pi_\theta}} \bigl[ \la \bar f_{\theta_*} - \tilde \tau \cdot (\beta^{-1}Q_\omega + \tau^{-1} f_\theta), \bar f_{\theta^\dagger} - \bar f_{\theta_*}  \ra \bigr] \geq 0, \qquad \text{for any } \theta^\dagger\in \cB(\theta_0, R_{\rm a}),
\$
which is equivalent to
\#\label{eq:aa2}
\theta_* = \argmin_{\theta^\dagger \in \cB(\theta_0, R_{\rm a})} \EE_{\rho_{\pi_\theta}}\bigl[ \bigl( \bar f_{\theta^\dagger} - \tilde \tau \cdot ( \beta^{-1} Q_{\omega} +  \tau^{-1} f_{\theta}   ) \bigr)^2\bigr]. 
\#
Meanwhile, by the fact that $\theta_0 = \omega_0$, we have
\$
\tilde \tau \cdot ( \beta^{-1} \bar Q_{\omega}+  \tau^{-1} \bar f_{\theta}   ) & = \tilde \tau \cdot \bigl( \beta^{-1}\cdot (Q_{\omega_0} + (\omega - \omega_0)^\top \nabla_\omega Q_{\omega_0}) +  \tau^{-1} \cdot (f_{\theta_0} + (\theta - \theta_0)^\top \nabla_\theta f_{\theta_0})   \bigr) \notag\\
& = f_{\theta_0} + \bigl(\tilde \tau \cdot ( \beta^{-1} \omega + \tau^{-1} \theta)  - \theta_0\bigr)^\top \nabla_\theta f_{\theta_0},
\$
where the second line comes from $\tilde \tau^{-1} = \beta^{-1} + \tau^{-1}$. Note that $\theta\in\cB(\theta_0, R_{\rm a})$, $\omega\in\cB(\omega_0, R_{\rm c})$, $\theta_0 = \omega_0$, and $R_{\rm a} = R_{\rm c}$, we know that $\tilde \tau \cdot ( \beta^{-1} \omega + \tau^{-1} \theta) \in \cB(\theta_0, R_{\rm a})$. Therefore, with probability at least $1 - \exp( - \Omega(R_{\rm a}^{2/3} m_{\rm a}^{2/3} H_{\rm a} ))$ we have
\#\label{eq:ff15}
& \EE_{\rho_{\pi_\theta}}\bigl[ \bigl( \bar f_{\theta_*} - \tilde \tau \cdot ( \beta^{-1} Q_{\omega}+  \tau^{-1} f_{\theta}   ) \bigr)^2\bigr]\notag\\
& \qquad \leq \EE_{\rho_{\pi_\theta}}\bigl[ \bigl( \tilde \tau \cdot ( \beta^{-1} \bar Q_{\omega} +  \tau^{-1} \bar f_{\theta}   ) - \tilde \tau \cdot ( \beta^{-1} Q_{\omega} +  \tau^{-1} f_{\theta}   ) \bigr)^2\bigr] \notag\\
& \qquad \leq \tilde \tau^2 \cdot\beta^{-2} \cdot \EE_{\rho_{\pi_\theta}}[ (  \bar Q_{\omega}  -  Q_{\omega}  )^2 ] +  \tilde \tau^2\cdot \tau^{-2} \cdot \EE_{\rho_{\pi_\theta}}[ (  \bar f_{\theta}  -  f_{\theta}  )^2 ]\notag\\
& \qquad = O(R_{\rm a}^{8/3} m_{\rm a}^{-1/3} H_{\rm a}^{5}\log m_{\rm a} ), 
\#
where the first inequality comes from \eqref{eq:aa2}, and the last inequality comes from Lemma \ref{lemma:bd-func} and the fact that $R_{\rm c} = R_{\rm a}$, $m_{\rm c} = m_{\rm a}$, and $H_{\rm c} = H_{\rm a}$.  Combining \eqref{eq:ff11} and \eqref{eq:ff15}, by triangle inequality, we have
\$
& \EE_{\rho_{\pi_\theta}}\bigl[ \bigl( f_{\overline\theta}(s,a) - \tilde \tau \cdot ( \beta^{-1} Q_{\omega}(s,a) +  \tau^{-1} f_{\theta}(s,a)   ) \bigr)^2\bigr] = O( R_{\rm a}^2 N_{\rm a}^{-1/2} + R_{\rm a}^{8/3} m_{\rm a}^{-1/6} H_{\rm a}^{7} \log m_{\rm a} ), 
\$
which finishes the proof of Proposition \ref{prop:ac-bd}.

\subsection{Proof of Proposition \ref{prop:acc-bd}}\label{prf:prop:acc-bd}
The proof is similar to that of Proposition \ref{prop:ac-bd} in \S \ref{prf:prop:ac-bd}.  For the completeness of the paper, we present it here.  
We define the local linearization of $Q_{\omega}$ as follows, 
\#\label{eq:lin-critic}
\bar Q_{\omega} = Q_{\omega_0} + (\omega - \omega_0)^\top \nabla_{\omega_0} Q_{\omega}. 
\#
We denote by 
\#\label{eq:def-gs-c}
& g_n = \bigl( Q_{\omega(n)}(s_0, a_0) -  \gamma \cdot Q_\omega(s_1, a_1) -  (1-\gamma)\cdot r_0  \bigr)\cdot \nabla_\omega Q_{\omega(n)}(s_0, a_0), && g_n^e = \EE_{\pi_\theta}[ g_n ], \notag\\
& \bar g_n = \bigl( \bar Q_{\omega(n)}(s_0, a_0) -  \gamma\cdot Q_\omega(s_1, a_1) -  (1-\gamma) \cdot r_0  \bigr)\cdot \nabla_\omega Q_{\omega_0}(s_0, a_0), && \bar g_n^e = \EE_{\pi_\theta}[ \bar g_n ],\notag\\
& g_* = \bigl( Q_{\omega_*}(s_0, a_0) -  \gamma\cdot Q_\omega(s_1, a_1) - (1-\gamma) \cdot r_0  \bigr)\cdot \nabla_\omega Q_{\omega_*}(s_0, a_0), && g_*^e = \EE_{\pi_\theta}[ g_* ],\notag\\
& \bar g_* = \bigl( \bar Q_{\omega_*}(s_0, a_0) -  \gamma\cdot Q_\omega(s_1, a_1) - (1-\gamma) \cdot r_0  \bigr)\cdot \nabla_\omega Q_{\omega_0}(s_0, a_0), && \bar g_*^e = \EE_{\pi_\theta}[ \bar g_* ],
\#
where $\omega_*$ satisfies that
\#\label{eq:def-omega-star}
\omega_* = \Gamma_{\cB(\omega_0, R_{\rm c} )}(\omega_* - \alpha \cdot \bar g_*^e ). 
\#
Here the expectation $\EE_{\pi_\theta}[\cdot]$ is taken following $(s_0, a_0)\sim \rho_{\pi_\theta}(\cdot)$, $s_1\sim P(\cdot\given s_0, a_0)$, $a_1\sim \pi_{\theta}(\cdot\given s_1)$, and $r_0 = \cR(s_0, a_0)$.  By Algorithm \ref{algo:critic}, we know that
\$
\omega(n+1) = \Gamma_{\cB(\omega_0, R_{\rm c} )}(\omega(n) - \eta \cdot  g_n ). 
\$
Note that
\#\label{eq:ac-bd1-c}
& \EE_{\pi_\theta}\bigl[ \| \omega(n+1) - \omega_* \|_2^2 \given \omega(n) \bigr]\notag\\
& \qquad = \EE_{\pi_\theta}\bigl[ \| \Gamma_{\cB(\omega_0, R_{\rm c})}(\omega(n) - \eta \cdot  g_n ) - \Gamma_{\cB(\omega_0, R_{\rm c})}(\omega_* - \eta \cdot \bar g_*^e ) \|_2^2 \given \omega(n) \bigr]\notag\\
& \qquad \leq \EE_{\pi_\theta}\bigl[ \| (\omega(n) - \eta \cdot  g_n ) - (\omega_* - \eta \cdot \bar g_*^e ) \|_2^2 \given \omega(n) \bigr]\notag\\
& \qquad = \|\omega(n) - \omega_*\|_2^2  + 2\eta \cdot \underbrace{ \la \omega_* - \omega(n), g_n^e - \bar g_*^e \ra }_{\rm (iii)}+ \eta^2 \cdot \underbrace{\EE_{\pi_\theta} \bigl[\|g_n - \bar g_*^e\|_2^2\given \omega(n)\bigr]}_{\rm (iv)}. 
\#
We upper bound term (iii) and term (iv) on the RHS of \eqref{eq:ac-bd1-c} in the sequel. 

\vskip5pt
\noindent\textbf{Upper Bound of Term (iii).} By H\"older's inequality, it holds that
\#\label{eq:ac-bd-t1-a-c}
& \la \omega_* - \omega(n), g_n^e - \bar g_*^e \ra   \notag\\
& \qquad = \la  \omega_* - \omega(n), g_n^e - \bar g_n^e \ra + \la  \omega_* - \omega(n), \bar g_n^e - \bar g_*^e \ra \notag\\
& \qquad \leq \| \omega_* - \omega(n)\|_2\cdot \| g_n^e - \bar g_n^e \|_2 + \la  \omega_* - \omega(n), \bar g_n^e - \bar g_*^e  \ra\notag\\
& \qquad \leq 2R_{\rm c} \cdot \|g_n^e - \bar g_n^e\|_2 + \la  \omega_* -\omega(n), \bar g_n^e - \bar g_*^e \ra,
\#
where we use the fact that $\omega(n), \omega_* \in \cB(\omega_0, R_{\rm c})$ in the last line. Further, by the definitions in \eqref{eq:def-gs-c}, it holds that
\#\label{eq:ac-bd-t1-b-c}
& \la \omega_* - \omega(n), \bar g_n^e - \bar g_*^e \ra \notag\\
& \qquad = \EE_{\pi_\theta} \bigl[ (\bar Q_{\omega(n)}(s_0, a_0) - \bar Q_{\omega_*}(s_0, a_0)) \cdot \la \omega_* - \omega(n), \nabla_\omega Q_{\omega_0}(s_0, a_0) \ra \bigr]\notag\\
& \qquad = \EE_{\pi_\theta} \bigl[ (\bar Q_{\omega(n)}(s_0, a_0) - \bar Q_{\omega_*}(s_0, a_0)) \cdot (\bar Q_{\omega_*}(s_0, a_0) - \bar Q_{\omega(n)}(s_0, a_0)) \bigr]\notag\\
& \qquad = -\EE_{\pi_\theta} \bigl[ (\bar Q_{\omega(n)}(s_0, a_0) - \bar Q_{\omega_*}(s_0, a_0))^2 \bigr] = -\EE_{\rho_{\pi_\theta}} \bigl[ (\bar Q_{\omega(n)} - \bar Q_{\omega_*})^2 \bigr], 
\#
where the second equality comes from \eqref{eq:lin-critic}, and the last equality comes from the fact that the expectation is only taken to the state-action pair $(s_0, a_0)$.  Combining \eqref{eq:ac-bd-t1-a-c} and \eqref{eq:ac-bd-t1-b-c}, we obtain the following upper bound of term (i), 
\#\label{eq:ac-bd-t1-c}
& \la \omega_* - \omega(n), g_n^e - \bar g_*^e \ra \leq 2R_{\rm c} \cdot \|g_n^e - \bar g_n^e\|_2 - \EE_{\rho_{\pi_\theta}} \bigl[ (\bar Q_{\omega(n)} - \bar Q_{\omega_*})^2 \bigr]. 
\#

\vskip5pt
\noindent\textbf{Upper Bound of Term (iv).} We now upper bound term (iv) on the RHS of \eqref{eq:ac-bd1-c}.  It holds by Cauchy-Schwarz inequality that
\#\label{eq:bd-term-2-haha-c}
& \EE_{\pi_\theta} \bigl[\|g_n - \bar g_*^e\|_2^2\given \omega(n)\bigr]\notag\\
& \qquad \leq 2\EE_{\pi_\theta} \bigl[\|g_n -  g_n^e\|_2^2\given \omega(n)\bigr] + 2 \|g_n^e - \bar g_*^e\|_2^2\notag\\
& \qquad \leq 2\underbrace{\EE_{\pi_\theta} \bigl[\|g_n -  g_n^e\|_2^2\given \omega(n)\bigr]}_{\rm (iv).a} + 4 \underbrace{\|g_n^e - \bar g_n^e\|_2^2}_{\rm (iv).b} + 4 \underbrace{\|\bar g_n^e - \bar g_*^e\|_2^2}_{\rm (iv).c}. 
\#
We upper bound term (iv).a, term (iv).b, and term (iv).c in the sequel. 

\vskip5pt
\noindent\textbf{Upper Bound of Term (iv).a.}  We now upper bound term (iv).a on the RHS of \eqref{eq:bd-term-2-haha-c}.  
By expanding the square, we have
\#\label{eq:suibian4-c}
\EE_{\pi_\theta} \bigl[\|g_n -  g_n^e\|_2^2\given \omega(n)\bigr] = \EE_{\pi_\theta} \bigl[\|g_n\|_2^2 -  \|g_n^e\|_2^2\given \omega(n)\bigr] \leq \EE_{\pi_\theta} \bigl[\|g_n\|_2^2 \given \omega(n)\bigr]. 
\#
Meanwhile, by the definition of $g_n$ in \eqref{eq:def-gs-c}, it holds that 
\#\label{eq:suibian3-c}
\|g_n\|_2^2 & = \bigl( Q_{\omega(n)}(s_0, a_0) -  \gamma\cdot Q_\omega(s_1, a_1) -  (1-\gamma) \cdot r_0  \bigr)^2 \cdot \| \nabla_\omega Q_{\omega(n)}(s_0, a_0)\|_2^2. 
\#
We first upper bound $Q_\omega$ as follows, 
\$
Q_\omega^2 = x^{(H_{\rm c})\top} bb^\top x^{(H_{\rm c})} = x^{(H_{\rm c})\top} x^{(H_{\rm c})} = \|x^{(H_{\rm c})}\|_2^2,
\$
where $x^{(H_{\rm c})}$ is the output of the $H_{\rm c}$-th layer of the DNN $Q_\omega$.  Further combining Lemma \ref{lemma:bd-node-diff}, it holds that 
\#\label{eq:suibian1-c}
| Q_\omega |  \leq 2. 
\#
Similarly, we have
\#\label{eq:suibian2-c}
| Q_{\omega(n)}|  \leq 2. 
\#
Combining Lemma \ref{lemma:bd-deriv}, \eqref{eq:suibian4-c}, \eqref{eq:suibian3-c}, \eqref{eq:suibian1-c}, and \eqref{eq:suibian2-c}, we have
\#\label{eq:bd-term-2-aa-c}
\EE_{\pi_\theta}\bigl[\|g_n -  g_n^e\|_2^2\given \omega(n)\bigr] = O(H_{\rm c}^2). 
\#

\vskip5pt
\noindent\textbf{Upper Bound of Term (iv).b.} We now upper bound term (iv).b on the RHS of \eqref{eq:bd-term-2-haha-c}.  It holds that
\#\label{eq:bd-term-2-b-c}
 & \|g_n^e - \bar g_n^e\|_2 \notag\\
 & \qquad  = \bigl\| \EE_{\pi_\theta} \bigl[ \bigl( Q_{\omega(n)}(s_0, a_0) -  \gamma\cdot Q_\omega(s_1, a_1) -  (1-\gamma) \cdot r_0  \bigr)\cdot \nabla_\omega Q_{\omega(n)}(s_0, a_0) \notag\\
& \qquad\qquad\qquad  - \bigl( \bar Q_{\omega(n)}(s_0, a_0) -  \gamma\cdot Q_\omega(s_1, a_1) - (1-\gamma) \cdot r_0  \bigr)\cdot \nabla_\omega Q_{\omega_0}(s_0, a_0)   \bigr]\bigr\|_2\notag\\
&\qquad  \leq \EE_{\pi_\theta} \bigl[ \bigl\|\bigl(\gamma\cdot Q_\omega(s_1, a_1) +  (1-\gamma) \cdot r_t\bigr) \cdot (\nabla_\omega Q_{\omega_0}(s_0, a_0) - \nabla_\omega Q_{\omega(n)}(s_0, a_0))\bigr\|_2\bigr]\notag\\
&\qquad  \qquad  + \EE_{\rho_{\pi_\theta}} \bigl[ \| Q_{\omega(n)}\nabla_\omega Q_{\omega(n)} - \bar Q_{\omega(n)} \nabla_\omega Q_{\omega_0} \|_2\bigr] \notag\\
&\qquad  \leq \EE_{\pi_\theta} \bigl[ \bigl\|\bigl(\gamma\cdot Q_\omega(s_1, a_1) +  (1-\gamma) \cdot r_0\bigr) \cdot (\nabla_\omega Q_{\omega_0}(s_0, a_0) - \nabla_\omega Q_{\omega(n)}(s_0, a_0))\bigr\|_2\bigr] \\
&\qquad  \qquad +  \EE_{\rho_{\pi_\theta}} \bigl[ \| (Q_{\omega(n)} - \bar Q_{\omega(n)}) \cdot \nabla_\omega Q_{\omega_0}\|_2 \bigr] + \EE_{\rho_{\pi_\theta}} \bigl[ \| Q_{\omega(n)} \cdot ( \nabla_\omega Q_{\omega(n)} - \nabla_\omega Q_{\omega_0}) \|_2\bigr]. \notag
\#
We now upper bound the three terms on the RHS of \eqref{eq:bd-term-2-b-c} in the sequel, respectively. 

For the term $\EE_{\rho_{\pi_\theta}} [ \| (Q_{\omega(n)} - \bar Q_{\omega(n)}) \cdot \nabla_\omega Q_{\omega_0}\|_2 ]$ on the RHS of \eqref{eq:bd-term-2-b-c}, following from Lemmas \ref{lemma:bd-deriv} and \ref{lemma:bd-func}, it holds with probability at least $1 - \exp(-\Omega(R_{\rm c}^{2/3}m_{\rm c}^{2/3} H_{\rm c} ))$ that 
\#\label{eq:bd-term-2-b-1-c}
\EE_{\rho_{\pi_\theta}} \bigl[ \| (Q_{\omega(n)} - \bar Q_{\omega(n)}) \cdot \nabla_\omega Q_{\omega_0}\|_2 \bigr]  = O\bigl( R_{\rm c}^{4/3} m_{\rm c}^{-1/6} H_{\rm c}^{7/2} (\log m_{\rm c})^{1/2} \bigr). 
\#

For the term $\EE_{\rho_{\pi_\theta}} [ \| Q_{\omega(n)} \cdot ( \nabla_\omega Q_{\omega(n)} - \nabla_\omega Q_{\omega_0}) \|_2 ]$ on the RHS of \eqref{eq:bd-term-2-b-c}, following from \eqref{eq:suibian2-c} and Lemma \ref{lemma:bd-deriv}, with probability at least $1 - \exp(-\Omega(R_{\rm c}^{2/3}m_{\rm c}^{2/3} H_{\rm c} ))$, we have 
\#\label{eq:bd-term-2-b-3-c}
\EE_{\rho_{\pi_\theta}} \bigl[  \| Q_{\omega(n)} \cdot ( \nabla_\omega Q_{\omega(n)} - \nabla_\omega Q_{\omega_0}) \|_2 \bigr]  =  O\bigl( R_{\rm c}^{1/3} m_{\rm c}^{-1/6} H_{\rm c}^{5/2} (\log m_{\rm c})^{1/2} \bigr). 
\#

For the term $\EE_{\pi_\theta} [ \|(\gamma\cdot Q_\omega(s_1, a_1) +  (1-\gamma) \cdot r_0) \cdot (\nabla_\omega Q_{\omega_0}(s_0, a_0) - \nabla_\omega Q_{\omega(n)}(s_0, a_0))\|_2]$ on the RHS of \eqref{eq:bd-term-2-b-c}, we first upper bound $|\gamma\cdot Q_\omega(s_1, a_1) +  (1-\gamma) \cdot r_0 |$ as follows,
\$
|\gamma\cdot Q_\omega(s_1, a_1) +  (1-\gamma) \cdot r_0 | \leq 2 +  \cR_{\max},
\$
where we use \eqref{eq:suibian1-c} and the fact that $|\cR(s, a)|\leq \cR_{\max}$ for any $(s,a)\in\cS\times \cA$. Further combining Lemma \ref{lemma:bd-deriv}, with probability at least $1 - \exp(-\Omega(R_{\rm c}^{2/3}m_{\rm c}^{2/3} H_{\rm c} ))$, we have
\#\label{eq:bd-term-2-b-4-c}
& \EE_{\pi_\theta} \bigl[ \bigl\|\bigl(\gamma\cdot Q_\omega(s_1, a_1) + (1-\gamma) \cdot r_0 \bigr) \cdot (\nabla_\omega Q_{\omega_0}(s_0, a_0) - \nabla_\omega Q_{\omega(n)}(s_0, a_0))\bigr\|_2\bigr]\notag\\
& \qquad  = O\bigl( R_{\rm c}^{1/3} m_{\rm c}^{-1/6} H_{\rm c}^{5/2} (\log m_{\rm c})^{1/2} \bigr).
\#

Now, combining \eqref{eq:bd-term-2-b-c}, \eqref{eq:bd-term-2-b-1-c}, \eqref{eq:bd-term-2-b-3-c}, and \eqref{eq:bd-term-2-b-4-c}, it holds with probability at least $1 - \exp(-\Omega(R_{\rm c}^{2/3}m_{\rm c}^{2/3} H_{\rm c} ))$ that 
\#\label{eq:bd-term-2-bb-c}
\|g_n^e - \bar g_n^e\|_2^2 = O(R_{\rm c}^{8/3}m_{\rm c}^{-1/3} H_{\rm c}^7 \log m_{\rm c}). 
\#

\vskip5pt
\noindent\textbf{Upper Bound of Term (iv).c.} We now upper bound term (iv).c on the RHS of \eqref{eq:bd-term-2-haha-c}.   It holds that
\$
\|\bar g_n^e - \bar g_*^e\|_2^2 = \bigl\|\EE_{\rho_{\pi_\theta}}[  (\bar Q_{\omega(n)} - \bar Q_{\omega_*}) \nabla_\omega Q_{\omega_0}  ]\bigr\|_2^2 \leq \EE_{\rho_{\pi_\theta}}\bigl [  (\bar Q_{\omega(n)} - \bar Q_{\omega_*})^2 \cdot \|\nabla_\omega Q_{\omega_0}\|_2^2  \bigr]. 
\$
Further combining Lemma \ref{lemma:bd-deriv}, it holds that
\#\label{eq:bd-term-2-cc-c}
\EE_{\pi_{\theta}}\bigl[\|\bar g_n^e - \bar g_*^e\|_2^2\given \omega(n) \bigr] \leq O(H_{\rm c}^2) \cdot \EE_{\rho_{\pi_\theta}} \bigl [  (\bar Q_{\omega(n)} - \bar Q_{\omega_*})^2   \bigr]. 
\#

\vskip5pt
Combining \eqref{eq:bd-term-2-haha-c}, \eqref{eq:bd-term-2-aa-c}, \eqref{eq:bd-term-2-bb-c}, and \eqref{eq:bd-term-2-cc-c}, we obtain the following upper bound for term (iv) on the RHS of \eqref{eq:ac-bd1-c}, 
\#\label{eq:ac-bd-t2-c}
\EE_{\pi_\theta}\bigl[\|g_n - \bar g_*^e\|_2^2 \given \omega(n) \bigr] & \leq O(R_{\rm c}^{8/3} m_{\rm c}^{-1/3} H_{\rm c}^7 \log m_{\rm c} )  +   O(H_{\rm c}^2) \cdot \EE_{\rho_{\pi_\theta}} \bigl [  (\bar Q_{\omega(n)} - \bar Q_{\omega_*})^2   \bigr]. 
\#

\vskip5pt
We continue upper bounding \eqref{eq:ac-bd1-c}.  By plugging \eqref{eq:ac-bd-t1-c} and \eqref{eq:ac-bd-t2-c} into \eqref{eq:ac-bd1-c}, it holds with probability at least $1 - \exp(-\Omega(R_{\rm c}^{2/3} m_{\rm c}^{2/3}  H_{\rm c} ))$ that 
\#\label{eq:ac-bd-t-c}
& \EE_{\pi_\theta} \bigl[ \| \omega(n+1) - \omega_* \|_2^2 \given \omega(n) \bigr] \notag\\
& \qquad \leq \| \omega(n) - \omega_* \|_2^2 + 2\eta\cdot \Bigl( O\bigl( R_{\rm c}^{7/3} m_{\rm c}^{-1/6} H_{\rm c}^{7/2} (\log m_{\rm c})^{1/2} \bigr) - \EE_{\rho_{\pi_\theta}} \bigl[ (\bar Q_{\omega(n)} - \bar Q_{\omega_*})^2 \bigr] \Bigr)\notag \\
& \qquad\qquad +  \eta^2\cdot \Bigl( O\bigl( R_{\rm c}^{8/3} m_{\rm c}^{-1/3} H_{\rm c}^7 \log m_{\rm c} \bigr)  + O(H_{\rm c}^2) \cdot \EE_{\rho_{\pi_\theta}} \bigl [  (\bar Q_{\omega(n)} - \bar Q_{\omega_*})^2   \bigr] \Bigr) . 
\#
Rearranging terms in \eqref{eq:ac-bd-t-c}, it holds with probability at least $1 - \exp(-\Omega(R_{\rm c}^{2/3} m_{\rm c}^{2/3}  H_{\rm c}))$ that 
\#\label{eq:ff10-q}
&(2\eta - \eta^2 \cdot O(H_{\rm c}^2)) \cdot \EE_{\rho_{\pi_\theta}} \bigl [  (\bar Q_{\omega(n)} - \bar Q_{\omega_*})^2   \bigr] \notag\\
&\qquad \leq  \| \omega(n) - \omega_* \|_2^2  - \EE_{\rho_{\pi_\theta}} [ \| \omega(n+1) - \omega_* \|_2^2 \given \omega(n) ]  + \eta\cdot O(R_{\rm c}^{8/3}m_{\rm c}^{-1/3} H_{\rm c}^7 \log m_{\rm c}). 
\#
By telescoping the sum and using Jensen's inequality in \eqref{eq:ff10-q}, we have
\$
& \EE_{\rho_{\pi_\theta}} \bigl [  ( \bar Q_{\bar \omega} - \bar Q_{\omega_*})^2   \bigr]  \leq \frac{1}{N_{\rm c}}\cdot \sum_{n = 0}^{N_{\rm c} - 1} \EE_{\rho_{\pi_\theta}} \bigl [  (\bar Q_{\omega(n)} - \bar Q_{\omega_*})^2   \bigr]\notag\\
& \qquad \leq 1/N_{\rm c}\cdot \bigl(2\eta - \eta^2\cdot O(H_{\rm c}^2) \bigr)^{-1}\cdot \bigl( \| \omega_0 - \omega_* \|_2^2  + \eta N_{\rm c}\cdot O(R_{\rm c}^{8/3} m_{\rm c}^{-1/6} H_{\rm c}^7 \log m_{\rm c}) \bigr)\notag\\
& \qquad \leq N_{\rm c}^{-1/2} \cdot \|\theta_0 - \theta_*\|_2^2 + O(R_{\rm c}^{8/3}m_{\rm c}^{-1/6} H_{\rm c}^7 \log m_{\rm c}), 
\$
where the last line comes from the choices that $\eta = N_{\rm c}^{-1/2}$ and $H_{\rm c} = O(N_{\rm c}^{1/4})$.  Further combining Lemma \ref{lemma:bd-func} and using triangle inequality, we have
\#\label{eq:ff12-c}
\EE_{\rho_{\pi_\theta}} \bigl [  ( Q_{\bar \omega} - \bar Q_{\omega_*})^2   \bigr]  = O ( R_{\rm c}^2 N_{\rm c}^{-1/2} + R_{\rm c}^{8/3} m_{\rm c}^{-1/6} H_{\rm c}^7 \log m_{\rm c}  ). 
\#

To establish the upper bound of $\EE_{\rho_{\pi_\theta}}[ (\bar Q_{\omega_*} - \tilde Q)^2 ]$, we upper bound $\EE_{\rho_{\pi_\theta}}[ (\bar Q_{\omega_*} - \tilde Q)^2 ]$ in the sequel. By the definition of $\omega_*$ in \eqref{eq:def-omega-star}, following a similar argument to derive \eqref{eq:aa2}, we have
\#\label{eq:bb1}
\omega_* = \argmin_{\omega^\dagger \in \cB(\omega_0, R_{\rm c})}  \EE_{\rho_{\pi_\theta}} \bigl[  ( \bar Q_{\omega^\dagger}(s_0, a_0) - \tilde Q(s_0, a_0) )^2   \bigr]. 
\#
From the fact that $\tilde Q\in\cU(m_{\rm c}, H_{\rm c}, R_{\rm c})$ by Assumption \ref{assum:close-bellman}, we know that $\tilde Q = Q_{\tilde \omega}$ for some $\tilde \omega\in \cB(\omega_0, R_{\rm c})$. Therefore, by \eqref{eq:bb1}, with probability at least $1 - \exp(-\Omega(R_{\rm c}^{2/3} m_{\rm c}^{2/3}  H_{\rm c}))$, we have
\#\label{eq:bb2}
\EE_{\rho_{\pi_\theta}} \bigl[ (\bar Q_{\omega_*} - \tilde Q)^2 \bigr] \leq \EE_{\rho_{\pi_\theta}} \bigl[ (\bar Q_{\tilde \omega} - \tilde Q)^2 \bigr] = O(R_{\rm c}^{8/3} m_{\rm c}^{-1/3} H_{\rm c}^5 \log m_{\rm c}),
\#
where we use Lemma \ref{lemma:bd-func} in the last inequality.  Now, combining \eqref{eq:ff12-c} and \eqref{eq:bb2}, by triangle inequality, with probability at least $1 - \exp(-\Omega(R_{\rm c}^{2/3} m_{\rm c}^{2/3}  H_{\rm c}))$, we have
\$
 \EE_{\rho_{\pi_\theta}} \bigl [  ( Q_{\bar \omega} - \tilde Q )^2   \bigr] & \leq 2\EE_{\rho_{\pi_\theta}} \bigl [  ( Q_{\bar \omega} - \bar Q_{\omega_*} )^2   \bigr] + 2\EE_{\rho_{\pi_\theta}} \bigl [  ( \bar Q_{\omega_*} - \tilde Q )^2   \bigr] \notag\\
 &  = O ( R_{\rm c}^2 N_{\rm c}^{-1/2} + R_{\rm c}^{8/3} m_{\rm c}^{-1/6} H_{\rm c}^7 \log m_{\rm c}  ), 
\$
which concludes the proof of Proposition \ref{prop:acc-bd}.

\section{Proofs of Lemmas}

\subsection{Proof of Lemma \ref{lemma:l-critic}}\label{prf:lemma:l-critic}
W denote by $\tilde Q  = \cT^{\pi_{\theta_{k}}} Q_{\omega_k}$. 
In the sequel, we upper bound $ \EE_{\rho_{k+1}}[ ( Q_{\omega_{k+1}}  - Q_{\bar \omega_{k+1}} )^2] $, where $\bar \omega_{k+1} = \Gamma_R( \tilde \omega_{k+1})$ and $\tilde \omega_{k+1}$ is defined in \eqref{eq:pop-omega}.   Note that by the fact that $\|\varphi(s,a)\|_2 \leq 1$ uniformly, it suffices to upper bound $\|\omega_{k+1} - \tilde \omega_{k+1}\|_2$.     By the definitions of $\omega_{k+1}$ and  $\tilde \omega_{k+1}$ in \eqref{eq:omega-update} and \eqref{eq:pop-omega}, respectively,  we have
\#\label{eq:lin-c-1}
& \|\omega_{k+1} - \bar \omega_{k+1}\|_2 \leq \| \hat \Phi \hat  v - \Phi v \|_2  \leq \| \Phi\|_2 \cdot \|\hat v - v\|_2 + \|\hat \Phi - \Phi\|_2 \cdot \|\hat v\|_2. 
\#
Here, we use the fact that the projection $\Gamma_R(\cdot )$ is a contraction in the first inequality, and triangle inequality in the second inequality. Also,  for notational convenience, we denote by $\hat \Phi$, $\Phi$, $\hat v$, and $v$ in \eqref{eq:lin-c-1} as follows, 
\$
& \hat \Phi = \Bigl(  \frac{1}{N}  \sum_{\ell = 1}^{N} \varphi(s_{\ell,1}, a_{\ell,1}) \varphi(s_{\ell,1}, a_{\ell,1})^\top  \Bigr)^{-1}, \quad \Phi = \bigl( \EE_{\rho_{k+1}} [ \varphi(s,a) \varphi(s,a)^\top  ] \bigr)^{-1}, \notag \\
& \hat v  =  \frac{1}{N} \sum_{\ell = 1}^{N} \bigl( (1-\gamma) r_{\ell,2} + \gamma  Q_{\omega_k}(s_{\ell,2}', a_{\ell,2}')   \bigr) \cdot \varphi(s_{\ell,2}, a_{\ell,2}), \notag\\
& v  =  \EE_{\rho_{k+1}} \bigl [  \bigl( (1-\gamma) \cR + \gamma \PP^{\pi_{\theta_{k+1}}} Q_{\omega_k}   \bigr)(s,a) \cdot \varphi(s,a) \bigr]  .  
\$
By the fact that $\|\varphi(s,a)\|_2 \leq 1$, $|\cR(s,a)| \leq \cR_{\max}$, and $\|\omega_k\|_2 \leq R$ we have 
\#\label{eq:lin-c-2}
\| \Phi\|_2 \leq 1/\sigma^*, \qquad \|\hat v\|_2 \leq \cR_{\max} + R. 
\#
Now, following from matrix Bernstein inequality \citep{tropp2015introduction} and Assumption \ref{assum:matrix}, with probability at least $1 - p/2$, we have
\#\label{eq:lin-c-3}
\|\hat \Phi - \Phi \|_2  \leq  \frac{ 4}{\sqrt N (\sigma^*)^2} \cdot \log ( N/p + d/p),  
\#
where $\sigma^*$ is defined in Assumption \ref{assum:matrix}. Similarly, with probability at least $1 - p/2$, we have
\#\label{eq:lin-c-4}
\|\hat v - v\|_2  \leq 4(\cR_{\max} + R) / \sqrt N \cdot \log (N/p+d/p). 
\#
Now, combining \eqref{eq:lin-c-1}, \eqref{eq:lin-c-2}, \eqref{eq:lin-c-3},  and \eqref{eq:lin-c-4},  we have
\$
\|\omega_{k+1} - \bar \omega_{k+1}\|_2  \leq \frac{ 16(\cR_{\max} + R) }{\sqrt N (\sigma^*)^2} \cdot \log (N/p+d/p). 
\$
Therefore, it holds with probability at least $1 - p$  that
\#\label{eq:lin-c-5}
( Q_{\omega_{k+1}}  - Q_{\bar \omega_{k+1}} )^2  \leq \frac{ 32(\cR_{\max} + R)^2 }{ N (\sigma^*)^2} \cdot \log^2 (N/p+d/p). 
\#
Meanwhile, by Assumption \ref{assum:error} and the definition of $\bar \omega_{k+1}$, we have 
\#\label{eq:lin-c-6}
\tilde Q(s,a) = Q_{\bar \omega_{k+1}}(s,a) 
\#
for any $(s,a)\in \cS\times\cA$. 
Combining \eqref{eq:lin-c-5} and \eqref{eq:lin-c-6} and a union bound argument, with probability at least $1 - \delta$, it holds  for any $k \in \{0, 1, \ldots, K\}$ that
\$
\EE_{\rho_{k+1}}  \bigl[ ( Q_{\omega_{k+1}}(s,a)  - \tilde Q (s,a)  )^2 \bigr] \leq \frac{ 32 (\cR_{\max} + R)^2 }{ N (\sigma^*)^4} \cdot \log^2 (NK/p+dK/p), 
\$
which concludes the proof of Lemma  \ref{lemma:l-critic}.

\subsection{Proof of Lemma \ref{lemma:a1}}\label{prf:lemma:a1}
  Following from the definitions of $\PP^\pi$ and $\PP$ in \eqref{eq:def-ops}, we have
\#\label{eq:term2-bound1-lin}
A_{1,k}(s,a) =  \bigl[\gamma ( \PP^{\pi^*} - \PP^{\pi_{\theta_{k+1}}})Q_{\omega_k }   \bigr] (s,a)  = \bigl [\gamma  \PP \la Q_{\omega_k}, \pi^* - \pi_{\theta_{k+1}} \ra  \bigr] (s,a). 
\#
By invoking Lemma \ref{lemma:3pt-lm} and combining \eqref{eq:term2-bound1-lin}, it holds for any $(s,a)\in\cS\times\cA$ that
\$
A_{1,k}(s,a) =  \bigl[\gamma( \PP^{\pi^*} - \PP^{\pi_{\theta_{k+1}}})Q_{\omega_k }   \bigr] (s,a) \leq  \bigl[\gamma \beta \cdot \PP(\vartheta_k + \epsilon^{\rm a}_{k+1})\bigr] (s,a), 
\$
where $\vartheta_k$ and  $\epsilon^{\rm a}_{k+1}$ are defined in \eqref{eq:def-deltak-lin} and \eqref{eq:def-pi-error-lin} of Lemma \ref{lemma:a1}, respectively. 
We conclude the proof of Lemma \ref{lemma:a1}.

\subsection{Proof of Lemma \ref{lemma:a2}} \label{prf:lemma:a2}
By the definition that $Q^*$ is the action-value function of an optimal policy $\pi^*$,  we know that $Q^*(s,a) \geq Q^\pi(s,a)$ for any policy $\pi$ and state-action pair $(s,a)\in \cS\times\cA$.  Therefore, for any $(s,a)\in\cS\times\cA$, we have
\#\label{eq:term1-bound1-lin}
A_{2,k}(s,a) = \bigl[\gamma \PP^{\pi^*} (Q^{\pi_{\theta_{k+1}}} - Q_{\omega_k })  \bigr](s,a)  \leq  \bigl[\gamma \PP^{\pi^*} (Q^{*} - Q_{\omega_k })  \bigr] (s,a).
\#
In the sequel, we upper bound $Q^*(s,a) - Q_{\omega_k}(s, a)$ for any $(s,a)\in\cS\times\cA$.   We define
\$
\tilde Q_{k+1} = (1-\gamma)\cdot  \cR + \gamma \cdot \PP^{\pi_{\theta_{k+1}}} Q_{\omega_{k}}.
\$
By its definition, we know that $\tilde Q_{k+1} =  \cT^{\pi_{\theta_{k+1}}} Q_{\omega_{k}}$. 
It holds for any $(s,a)\in\cS\times\cA$ that   
\#\label{eq:term1-bound2-lin}
& Q^*(s,a) - Q_{\omega_{k+1}}(s, a)\notag\\
&\qquad = Q^*(s,a) -  \tilde Q_{k+1}(s, a) + \tilde Q_{k+1}(s, a) - Q_{\omega_{k+1}}(s, a)\notag\\
&\qquad =  \bigl[  \bigl((1-\gamma)\cdot  \cR + \gamma \cdot \PP^{\pi^*} Q^* \bigr) -   \bigl((1-\gamma) \cdot\cR + \gamma\cdot \PP^{\pi_{\theta_{k+1}}} Q_{\omega_{k}} \bigr) \bigr](s, a) + \epsilon^{\rm c}_{k+1}(s, a)  \notag\\
&\qquad =  \gamma \cdot [ \PP^{\pi^*} Q^* - \PP^{\pi_{\theta_{k+1}}} Q_{\omega_{k}} ](s, a) + \epsilon^{\rm c}_{k+1}(s, a)  \notag\\
&\qquad =  \gamma \cdot [ \PP^{\pi^*} Q^* - \PP^{\pi^*} Q_{\omega_{k}} ](s, a) + \gamma \cdot [ \PP^{\pi^*} Q_{\omega_k} - \PP^{\pi_{\theta_{k+1}}} Q_{\omega_{k}} ](s, a) + \epsilon^{\rm c}_{k+1}(s, a)  \notag\\
& \qquad = \gamma\cdot \bigl[ \PP^{\pi^*}(Q^* - Q_{\omega_k})\bigr](s,a) + A_{1,k}(s,a) +  \epsilon^{\rm c}_{k+1}(s, a) \notag\\
& \qquad \leq \gamma\cdot \bigl[ \PP^{\pi^*}(Q^* - Q_{\omega_k})\bigr](s,a) + \gamma \beta \cdot \bigl[  \PP(\vartheta_k + \epsilon^{\rm a}_{k+1})\bigr](s,a) +  \epsilon^{\rm c}_{k+1}(s,a),
\#
where  $\epsilon^{\rm c}_{k+1}$ and $A_{1,k}$ are defined in \eqref{eq:error-Q-lin} and \eqref{eq:jiaoshaa}, respectively. 
Here, we use Lemma \ref{lemma:a1} to upper bound $A_{1,k}$ in the last line. 
We remark that \eqref{eq:term1-bound2-lin} upper bounds $Q^* - Q_{\omega_{k+1}}$ using $Q^* - Q_{\omega_k}$. 
By recursively applying a similar argument as in \eqref{eq:term1-bound2-lin}, we have
\#\label{eq:term1-bound3-lin}
&Q^*(s,a) - Q_{\omega_{k}}(s, a)\notag\\
&\qquad\leq \bigl[ (\gamma\PP^{\pi^*})^k (Q^* - Q_{\omega_{0}})  \bigr] (s,a) + \gamma \beta\cdot  \sum_{i = 0}^{k-1} \bigl[ (\gamma \PP^{\pi^*}  )^{k-i-1}   \PP(\vartheta_i +  \epsilon^{\rm a}_{i+1} )  \bigr](s,a)\\
& \qquad\qquad  +  \sum_{i = 0}^{k-1} \bigl[ (\gamma \PP^{\pi^*}  )^{k-i-1}  \epsilon^{\rm c}_{i+1}  \bigr] (s,a). \notag
\#
Combining \eqref{eq:term1-bound1-lin} and \eqref{eq:term1-bound3-lin}, it holds for any $(s, a)\in\cS\times \cA$ that 
\$
A_{2,k}(s,a )& \leq \bigl[\gamma \PP^{\pi^*} (Q^{*} - Q_{\omega_k })  \bigr](s,a)\notag\\
& \leq \bigl[ (\gamma\PP^{\pi^*})^{k+1} (Q^* - Q_{\omega_{0}})  \bigr] (s,a) + \gamma \beta \cdot \sum_{i = 0}^{k-1} \bigl[ (\gamma \PP^{\pi^*}  )^{k-i}   \PP(\vartheta_i +  \epsilon^{\rm a}_{i+1} )  \bigr](s,a)\\
& \qquad  +  \sum_{i = 0}^{k-1} \bigl[ (\gamma \PP^{\pi^*}  )^{k-i}  \epsilon^{\rm c}_{i+1}  \bigr] (s,a),  \notag
\$
where $\vartheta_i$, $\epsilon_{i+1}^{\rm a}$, and $\epsilon_{i+1}^{\rm c}$ are defined in \eqref{eq:def-deltak-lin} of Lemma \ref{lemma:a1}, \eqref{eq:def-pi-error-lin} of Lemma \ref{lemma:a1}, and \eqref{eq:error-Q-lin} of Lemma \ref{lemma:a2}, respectively. 
We conclude the proof of Lemma \ref{lemma:a2}.

\subsection{Proof of Lemma \ref{lemma:a3}} \label{prf:lemma:a3}
 Note that for any $(s,a)\in\cS\times\cA$, we have
\$
A_{3,k}(s,a) & = [ \cT^{\pi_{\theta_{k+1}}} Q_{\omega_{k}}  - Q^{\pi_{\theta_{k+1}}} ](s,a)\notag \\
& =  \Bigl[\bigl( (1-\gamma)\cdot  \cR + \gamma  \PP^{\pi_{\theta_{k+1}}} Q_{\omega_k }  \bigr) - Q^{\pi_{\theta_{k+1}}} \Bigr](s,a) \notag\\
&  = \Bigl[\bigl( (1-\gamma)\cdot  \cR + \gamma  \PP^{\pi_{\theta_{k+1}}} Q_{\omega_k }  \bigr) -   \sum_{t = 0}^{\infty}(1-\gamma)  (\gamma \PP^{\pi_{\theta_{k+1}}})^t \cR  \Bigr](s,a)\notag\\
& = \biggl[ \sum_{t = 1}^\infty    \bigl((\gamma \PP^{\pi_{\theta_{k+1}}})^{t} Q_{\omega_k } - (\gamma \PP^{\pi_{\theta_{k+1}}})^{t+1} Q_{\omega_k }  \bigr) -   \sum_{t = 1}^{\infty}(1-\gamma)  (\gamma \PP^{\pi_{\theta_{k+1}}})^t \cR  \biggr](s,a)     \notag\\
&  = \sum_{t = 1}^\infty \Bigl[ (\gamma \PP^{\pi_{\theta_{k+1}}})^t \bigl(Q_{\omega_k} - \gamma   \PP^{\pi_{\theta_{k+1}}}Q_{\omega_k} - (1-\gamma)\cdot \cR\bigr)\Bigr] (s,a) \notag\\
& = \sum_{t = 1}^\infty \Bigl[ (\gamma \PP^{\pi_{\theta_{k+1}}})^t \bigl(Q_{\omega_k} -  \cT^{\pi_{\theta_{k+1}}}Q_{\omega_k} \bigr)\Bigr] (s,a) \notag\\
& = \sum_{t = 1}^\infty \bigl[ (\gamma \PP^{\pi_{\theta_{k+1}}})^t e_{k+1} \bigr] (s,a) = \bigl[ \gamma \PP^{\pi_{\theta_{k+1}}} (I - \gamma \PP^{\pi_{\theta_{k+1}}})^{-1} e_{k+1} \bigr](s,a),
\$
where the term $e_{k+1}$ in the last line is defined in \eqref{eq:def-ek-lin}. 
We conclude the proof of Lemma \ref{lemma:a3}.

\subsection{Proof of Lemma \ref{lemma:bound-ek}}\label{prf:lemma:bound-ek}
We invoke Lemma \ref{lemma:3pt-lm} in \S\ref{sec:supp-res}, which gives
\#\label{eq:mmp1}
& \beta^{-1} \cdot  \la Q_{\omega_k}(s,\cdot), \pi_{\theta_k}(\cdot\given s) - \pi_{\theta_{k+1}}(\cdot\given s) \ra  \notag\\
& \qquad\leq \bigl \la \log ( \pi_{\theta_{k+1}}(\cdot\,|\, s)/\pi_{\theta_k}(\cdot\,|\, s)) - \beta^{-1} \cdot Q_{\omega_k}(s,\cdot), \pi_{\theta_k}(\cdot\given s) - \pi_{\theta_{k+1}}(\cdot\given s) \bigr \ra  \notag\\
& \qquad \qquad -  \kl (\pi_{\theta_k}(\cdot\given s) \,\|\, \pi_{\theta_{k+1}}(\cdot\given s))  \notag\\
& \qquad \leq  \bigl \la \log ( \pi_{\theta_{k+1}}(\cdot\,|\, s)/\pi_{\theta_k}(\cdot\,|\, s)) - \beta^{-1} \cdot Q_{\omega_k}(s,\cdot), \pi_{\theta_k}(\cdot\given s) - \pi_{\theta_{k+1}}(\cdot\given s) \bigr \ra  = \epsilon_{k+1}^{\rm b}(s). 
\#
Combining \eqref{eq:mmp1} and the definition of $\PP^\pi$ in \eqref{eq:def-ops}, we have
\#\label{eq:bound-ek-1}
[\PP^{\pi_{\theta_{k}}}Q_{\omega_{k}} - \PP^{\pi_{\theta_{k+1}}}Q_{\omega_{k}}](s,a) \leq \beta [\PP \epsilon_{k+1}^{\rm b} ] (s). 
\#
By the definition of $e_{k+1}$ in \eqref{eq:def-ek-lin}, we have
\#\label{eq:bound-ek-2}
e_{k+1}(s,a) & = \bigl[Q_{\omega_{k}} - \gamma\cdot \PP^{\pi_{\theta_{k+1}}}Q_{\omega_{k}}  -(1-\gamma)\cdot  \cR \bigr](s, a)\notag\\
& \leq \bigl[Q_{\omega_{k}} - \gamma\cdot \PP^{\pi_{\theta_{k}}}Q_{\omega_{k}}  -(1-\gamma)\cdot \cR  \bigr](s, a) + \beta\gamma\cdot [\PP \epsilon_{k+1}^{\rm b}](s,a) \\
& = \bigl[\tilde Q_{k} - \gamma\cdot \PP^{\pi_{\theta_{k}}}\tilde Q_{k}  -(1-\gamma)\cdot \cR  \bigr](s, a) + \bigl[\beta\gamma \PP \epsilon_{k+1}^{\rm b} - (I-\gamma \PP^{\pi_{\theta_k}})\epsilon_{k}^{\rm c}\bigr](s,a),\notag
\#
where we use \eqref{eq:bound-ek-1} in the first inequality, and 
\#\label{eq:popu-update}
\tilde Q_k = (1-\gamma)\cdot \cR + \gamma\cdot \PP^{\pi_{\theta_{k}}} Q_{\omega_{k-1}}.
\#  
For the first term on the RHS of \eqref{eq:bound-ek-2}, by \eqref{eq:popu-update}, it holds that
\#\label{eq:bound-ek-3}
& \tilde Q_{k} - \gamma\cdot \PP^{\pi_{\theta_{k}}}\tilde Q_{k}  -(1-\gamma)\cdot \cR\notag\\
&\qquad = (1-\gamma)\cdot \cR + \gamma\cdot \PP^{\pi_{\theta_{k}}} Q_{\omega_{k-1}} - \gamma(1-\gamma)\cdot \PP^{\pi_{\theta_{k}}} \cR  - (\gamma \PP^{\pi_{\theta_{k}}})^2 Q_{\omega_{k-1}} - (1-\gamma)\cdot \cR\notag\\
& \qquad = \gamma\cdot \PP^{\pi_{\theta_k}} \bigl(Q_{\omega_{k-1}} - \gamma\PP^{\pi_{\theta_k}}Q_{\omega_{k-1}} - (1-\gamma)\cR\bigr) = \gamma\cdot \PP^{\pi_{\theta_k}} e_k. 
\#
Combining \eqref{eq:bound-ek-2} and \eqref{eq:bound-ek-3}, we have for any $(s,a)\in\cS\times\cA$ that
\#\label{eq:bound-ek-4}
e_{k+1}(s,a)\leq  [\gamma\PP^{\pi_{\theta_k}} e_k ](s,a) + \bigl[\beta\gamma \PP \epsilon_{k+1}^{\rm b} - (I-\gamma \PP^{\pi_{\theta_k}})\epsilon_{k}^{\rm c}\bigr](s,a). 
\#
By telescoping \eqref{eq:bound-ek-4}, it holds that
\$
e_{k+1}(s,a)\leq \biggl[ \Bigl( \prod_{s = 1}^k \gamma \PP^{\pi_{\theta_s}} \Bigr) e_1 + \sum_{i = 1}^k \gamma^{k-i} \Bigl( \prod_{s = i+1}^k  \PP^{\pi_{\theta_s}} \Bigr)\bigl(\beta\gamma \PP \epsilon_{i+1}^{\rm b} - (I - \gamma \PP^{\pi_{\theta_{i}}})\epsilon_{i}^{\rm c}\bigr)  \biggr](s,a). 
\$
This finishes the proof of the lemma.

\subsection{Proof of Lemma \ref{lemma:meiyou}}\label{prf:lemma:meiyou}
Note that  $\|\omega_0\|_2 \leq R$ and $|r(s,a)| \leq r_{\max}$ for any $(s,a) \in \cS\times \cA$, which implies that $|Q_{\omega_0}(s,a)| \leq R$ and $|Q^*(s,a)| \leq r_{\max}$ by their definitions. 
Thus, for $M_1$, we have
\#\label{eq:f-bd-1-lin}
& |M_1|  \leq \EE_{\rho}\Bigl[ (I - \gamma \PP^{\pi^*})^{-1} \sum_{k = 0}^K (\gamma \PP^{\pi^*})^{k+1}|Q^* - Q_{\omega_0}|  \Bigr]  \notag\\
& \qquad \leq  4 (1 - \gamma )^{-1} \sum_{k = 0}^K \gamma^{k+1} \cdot (\cR_{\max} + R)  \leq 4 (1 - \gamma)^{-2} \cdot (\cR_{\max} + R). 
\#
For $M_2$,   by the definition of $e_1$ in \eqref{eq:def-ek-lin}, $|\omega_k| \leq R$,   $|\phi(s,a)| \leq 1$, and $|r(s,a)| \leq r_{\max}$, we have
\#\label{eq:shabi3}
|e_1(s,a)| & = \bigl| [Q_{\omega_{k}} - \cT^{\pi_{\theta_{k+1}}} Q_{\omega_{k}}  ](s,a) \bigr|\notag\\
& = \bigl|  \omega_{k}^\top \phi(s,a) - \gamma \cdot \omega_{k}^\top  [\PP^{\pi_{\theta_{k+1}}} \phi]  (s,a)   -(1-\gamma) \cdot r(s,a)   \bigr|\notag\\
& \leq 2R + r_{\max}
\#
for any $(s,a) \in \cS\times\cA$.   Therefore, we have
\#\label{eq:shabi4}
|M_2| \leq  (1 - \gamma)^{-3} \cdot (2R + \cR_{\max}). 
\#
Meanwhile,   by the initialization $\tau_0 = \infty$ in Algorithm \ref{algo:l-ac}, the initial policy  $\pi_{\theta_0}(\cdot \given s)$ is a uniform distribution over $\cA$.  Therefore, it holds for any $s\in \cS$ that 
\#\label{eq:shabi1}
\kl \bigl (\pi^*(\cdot \given s) \,\|\, \pi_{\theta_0}(\cdot \given s) \bigr) & = \int_\cA  \pi^*(a \given s) \log \frac{\pi^*(a \given s)}{\pi_{\theta_0}(a \given s)}  \ud a\notag\\
& = \int_\cA  \pi^*(a \given s) \log \pi^*(a \given s)   \ud a - \int_\cA  \pi^*(a \given s) \log \pi_{\theta_0}(a \given s)   \ud a \notag\\
& \leq - \int_\cA  \pi^*(a \given s) \log \pi_{\theta_0}(a \given s)   \ud a \notag\\
& =  \int_\cA  \pi^*(a \given s) \log |\cA|   \ud a = \log |\cA|. 
\#
Therefore, by \eqref{eq:shabi1}, we have
\#\label{eq:shabi2}
&  M_3 \leq (1 - \gamma)^{-2} \cdot \log |\cA|\cdot K^{1/2} , 
\#
where we use $\beta  = K^{1/2}$.  
We see that \eqref{eq:f-bd-1-lin}, \eqref{eq:shabi4}, and \eqref{eq:shabi2} upper bound $M_1$, $M_2$, and $M_3$, respectively. We conclude the proof of Lemma \ref{lemma:meiyou}.

\subsection{Proof of Lemma \ref{lemma:you}}\label{prf:lemma:you}
For $M_4$, by changing the index of summation, we have
\#\label{eq:mon5-lin}
|M_4| & = \Bigl| \EE_{\rho}\Bigl[  \sum_{k = 0}^K \sum_{i = 0}^k \sum_{j = 0}^\infty (\gamma \PP^{\pi^*})^{k-i+j}  \epsilon_{i+1}^{\rm c}  \Bigr] \Bigr|\notag\\
& = \Bigl| \EE_{\rho}\Bigl[  \sum_{k = 0}^K \sum_{i = 0}^k \sum_{t = k-i}^\infty (\gamma \PP^{\pi^*})^{t}  \epsilon_{i+1}^{\rm c}  \Bigr] \Bigr| \notag \\
&  \leq   \sum_{k = 0}^K \sum_{i = 0}^k \sum_{t = k-i}^\infty \bigl| \EE_{\rho}\bigl[ (\gamma \PP^{\pi^*})^{t}  \epsilon_{i+1}^{\rm c}  \bigr] \bigr|,
\#
where we expand $(I - \gamma \PP^{\pi^*})^{-1}$ into an infinite sum in the first equality. Further, by changing the measure of the expectation from $\rho$ to $\rho^*$ on the RHS of \eqref{eq:mon5-lin}, we have
\#\label{eq:mon6-lin}
&    \sum_{k = 0}^K \sum_{i = 0}^k \sum_{t = k-i}^\infty \bigl| \EE_{\rho}\bigl[ (\gamma \PP^{\pi^*})^{t}  \epsilon_{i+1}^{\rm c}  \bigr] \bigr|   \leq   \sum_{k = 0}^K \sum_{i = 0}^k \sum_{t = k-i}^\infty \gamma^{t} c(t)\cdot  \EE_{\rho^*}[ |\epsilon_{i+1}^{\rm c}| ], 
\#
where $c(t)$ is defined in Assumption \ref{assum:cc-lin}. Further, by   changing the index of summation on the RHS of \eqref{eq:mon6-lin}, combining  \eqref{eq:mon5-lin},  we have
\#\label{eq:f-bd-3-lin}
|M_4| & \leq   \sum_{k = 0}^K \sum_{t = 0}^\infty \sum_{i = \max\{0, k-t\}}^k \gamma^{t} c(t)\cdot \varepsilon_Q  \notag\\
&  \leq \sum_{k = 0}^K \sum_{t = 0}^\infty 2  t  \gamma^{t } c(t)\cdot \varepsilon_Q \notag\\
&  \leq \gamma \sum_{k = 0}^K 2   C_{\rho, \rho^*} \cdot  \varepsilon_Q   \leq 3 K C_{\rho, \rho^*}  \cdot \varepsilon_Q , 
\#
where $\varepsilon_Q = \max_i \EE_{\rho^*}[|\epsilon_{i+1}^{\rm c}|]$, and $C_{\rho, \rho^*}$ is defined in Assumption \ref{assum:cc-lin}.  

Now, for $M_5$, by a similar argument as in the derivation of \eqref{eq:f-bd-3-lin},  we have
\#\label{eq:f-bd-5-lin}
M_5 & \leq   \sum_{i=0}^\infty \sum_{k=0}^K \sum_{j=0}^\infty \sum_{\ell=1}^k \gamma^{i+j+k-\ell+1} c(i+j+k-\ell+1)  \cdot \varepsilon_Q \notag \\
&  = \sum_{i=0}^\infty \sum_{k=0}^K \sum_{j=0}^\infty \sum_{t = i+j+1}^{i+j+k} \gamma^{t} c(t) \cdot \varepsilon_Q   \leq \sum_{k=0}^K \sum_{t=1}^\infty t^2 \gamma^{t} c(t) \cdot \varepsilon_Q  \leq  K   C_{\rho, \rho^*} \cdot \varepsilon_Q. 
\#
We see that \eqref{eq:f-bd-3-lin} and \eqref{eq:f-bd-5-lin} upper bound $M_4$ and $M_5$, respectively.  We conclude the proof of Lemma \ref{lemma:you}.

\subsection{Proof of Lemma \ref{lemma:error-prop}}\label{prf:lemma:error-prop}
\noindent\textbf{Part 1.} We first show that the first inequality holds.  Note that 
\$
\pi_{\theta_k}(a\given s) = \exp(\tau_k^{-1}f_{\theta_k}(s,a)) / Z_{\theta_k}(s), \qquad \pi_{\theta_{k+1}}(a \given  s) = \exp(\tau_{k+1}^{-1}f_{\theta_{k+1}}(s,a)) / Z_{\theta_{k+1}}(s), 
\$
Here $Z_{\theta_k}(s), Z_{\theta_{k+1}}(s)\in \RR$ are normalization factors, which are defined as
\$
Z_{\theta_k}(s) = \sum_{a'\in\cA} \exp( \tau_k^{-1}f_{\theta_k}(s,a') ) , \qquad Z_{\theta_{k+1}}(s) =\sum_{a'\in\cA} \exp( \tau_{k+1}^{-1}f_{\theta_{k+1}}(s,a')). 
\$
Thus, we have 
\# \label{eq:l04}
&  \la   \log  (\pi_{\theta_{k+1}}(\cdot \given s) / \pi_{\theta_{k}}(\cdot \given s)) - \beta^{-1} Q_{\omega_k}(s,\cdot) , \pi^*(\cdot\given s) - \pi_{\theta_{k+1}}(\cdot\given s)   \ra    \notag\\
&   \qquad =    \la  \tau_{k+1}^{-1}f_{\theta_{k+1}}(s,\cdot)- (\beta^{-1}Q_{\omega_k}(s,\cdot) + \tau_k^{-1} f_{\theta_k}(s, \cdot)), \pi^*(\cdot\given s)-\pi_{\theta_k}(\cdot\given s) \ra ,
\#
where we use the fact that
\$
&\la \log Z_{\theta_{k+1}}(s) - \log Z_{\theta_k}(s),  \pi^*(\cdot \given s)-\pi_{\theta_{k+1}}(\cdot \given s)\ra  \notag\\
&\qquad = (\log Z_{\theta_{k+1}}(s) - \log Z_{\theta_k}(s)) \cdot  \sum_{a'\in \cA} ( \pi^*(a' \given s) - \pi_{\theta_{k+1}}(a' \given s) ) = 0.
\$
Thus, it remains to upper bound the right-hand side of \eqref{eq:l04}. We have
\#\label{eq:l01}
& \la \tau_{k+1}^{-1}f_{\theta_{k+1}}(s,\cdot)- (\beta_k^{-1}Q_{\omega_k}(s,\cdot) + \tau_k^{-1}f_{\theta_k}(s,\cdot)), \pi^*(\cdot\given s)-\pi_{\theta_{k+1}}(\cdot\given s) \ra  \\
&\qquad =\biggl\la \tau_{k+1}^{-1}f_{\theta_{k+1}}(s,\cdot)- (\beta_k^{-1}Q_{\omega_k}(s,\cdot) + \tau_k^{-1}f_{\theta_k}(s,\cdot)), \pi_{\theta_k} (\cdot\given s)\cdot\biggr(\frac{\pi^*(\cdot\given s)}{\pi_{\theta_k}(\cdot\given s)}-\frac{\pi_{\theta_{k+1}}(\cdot\given s)}{\pi_{\theta_k}(\cdot\given s)}\biggl) \biggr\ra  \notag.
\#
Taking expectation with respect to $s\sim\nu^*$ on the both sides of \eqref{eq:l01} and using the Cauchy-Schwarz inequality, we obatin
\$
&\EE_{\nu^*}\bigl[\bigl | \bigl \la \tau_{k+1}^{-1}f_{\theta_{k+1}}(s,\cdot)- (\beta_k^{-1}Q_{\omega_k}(s,\cdot) + \tau_k^{-1}f_{\theta_k}(s,\cdot)), \pi^*(\cdot\given s)-\pi_{\theta_{k+1}}(\cdot\given s) \bigr\ra \bigr |   \bigr]\bigr|\notag\\
&\quad = \int_\cS \biggl| \biggl\la \tau_{k+1}^{-1}f_{\theta_{k+1}}(s,\cdot)- (\beta_k^{-1}Q_{\omega_k}(s,\cdot) + \tau_k^{-1}f_{\theta_k}(s,\cdot)), \notag\\
& \qquad\qquad\qquad  \pi_{\theta_k}(\cdot\given s)\cdot\nu_k(s)\cdot\biggr(\frac{\pi^*(\cdot\given s)}{\pi_{\theta_k}(\cdot\given s)}-\frac{\pi_{\theta_{k+1}}(\cdot\given s)}{\pi_{\theta_k}(\cdot\given s)}\biggl) \biggr\ra  \biggr |\cdot\Bigl | \frac{\nu^*(s)}{\nu_k(s)} \Bigr | \ud s \notag\\
& \quad= \int_{\cS\times\cA}\bigl | \tau_{k+1}^{-1}f_{\theta_{k+1}}(s, a)- (\beta_k^{-1}Q_{\omega_k}(s,a) + \tau_k^{-1}f_{\theta_k}(s,a))\bigr | \notag \\
& \qquad\qquad\qquad \cdot\biggl | \frac{\rho^*(a\given s)}{\rho_k(a\given s)}-\frac{\pi_{\theta_{k+1}}(a\given s)\cdot\nu^*(s)}{\rho_k(a\given s)}\biggr | \ud \rho_k(s,a) \notag\\
&\quad \leq \EE_{\rho_k}\bigl[\bigl(\tau_{k+1}^{-1}f_{\theta_{k+1}}(s, a)- (\beta_k^{-1}Q_{\omega_k}(s, a) + \tau_k^{-1}f_{\theta_k}(s, a))\bigr)^2\bigr]^{1/2}     \cdot     \EE_{\rho_k}\biggl[\biggl|\frac{\ud\rho^*}{\ud\rho_{k}}-\frac{\ud(\pi_{\theta_{k+1}}\nu^*)}{\ud\rho_k}\biggr|^2\biggr]^{1/2}\notag\\
&\quad \leq \sqrt{2} \tau_{k+1}^{-1} \cdot \varepsilon_{k+1, f} \cdot (\phi^*_{k} + \psi^*_{k}),
\$
where in the last inequality we use the error bound in \eqref{eq:error-kl1} and the definition of $\phi^*_k$ and $\psi^*_k$ in Assumption \ref{assum:cc}. This finishes the proof of the first inequality. 

\vskip5pt
\noindent\textbf{Part 2.} The proof of the second inequality follows from a similar argument as above. We have 
\# \label{eq:l0422}
&  \la   \log  (\pi_{\theta_{k+1}}(\cdot \given s) / \pi_{\theta_{k}}(\cdot \given s)) - \beta^{-1} Q_{\omega_k}(s,\cdot) , \pi_{\theta_k}(\cdot\given s) - \pi_{\theta_{k+1}}(\cdot\given s)   \ra    \notag\\
&   \qquad =    \la  \tau_{k+1}^{-1}f_{\theta_{k+1}}(s,\cdot)- (\beta^{-1}Q_{\omega_k}(s,\cdot) + \tau_k^{-1} f_{\theta_k}(s, \cdot)), \pi_{\theta_k}(\cdot\given s) - \pi_{\theta_{k+1}}(\cdot\given s) \ra ,
\#
where we use the fact that
\$
&\la \log Z_{\theta_{k+1}}(s) - \log Z_{\theta_k}(s),  \pi_{\theta_k}(\cdot \given s)-\pi_{\theta_{k+1}}(\cdot \given s)\ra  \notag\\
&\qquad = (\log Z_{\theta_{k+1}}(s) - \log Z_{\theta_k}(s)) \cdot  \sum_{a'\in \cA} ( \pi_{\theta_k}(a' \given s) - \pi_{\theta_{k+1}}(a' \given s) ) = 0.
\$
Thus, it remains to upper bound the right-hand side of \eqref{eq:l0422}. We have
\#\label{eq:l0122}
& \la \tau_{k+1}^{-1}f_{\theta_{k+1}}(s,\cdot)- (\beta_k^{-1}Q_{\omega_k}(s,\cdot) + \tau_k^{-1}f_{\theta_k}(s,\cdot)), \pi_{\theta_k}(\cdot\given s)-\pi_{\theta_{k+1}}(\cdot\given s) \ra  \\
&\qquad =\biggl\la \tau_{k+1}^{-1}f_{\theta_{k+1}}(s,\cdot)- (\beta_k^{-1}Q_{\omega_k}(s,\cdot) + \tau_k^{-1}f_{\theta_k}(s,\cdot)), \pi_{\theta_k} (\cdot\given s)\cdot\biggr(1 -   \frac{\pi_{\theta_{k+1}}(\cdot\given s)}{\pi_{\theta_k}(\cdot\given s)}\biggl) \biggr\ra  \notag.
\#
Taking expectation with respect to $s\sim\nu^*$ on the both sides of \eqref{eq:l0122} and using the Cauchy-Schwarz inequality, we obatin
\$
&\EE_{\nu^*}\bigl[ \bigl| \bigl\la \tau_{k+1}^{-1}f_{\theta_{k+1}}(s,\cdot)- (\beta_k^{-1}Q_{\omega_k}(s,\cdot) + \tau_k^{-1}f_{\theta_k}(s,\cdot)), \pi_{\theta_k}(\cdot\given s)  -  \pi_{\theta_{k+1}}(\cdot\given s) \bigr\ra  \bigr| \bigr]\notag\\
&\quad  = \int_\cS\biggl| \biggl\la \tau_{k+1}^{-1}f_{\theta_{k+1}}(s,\cdot)- (\beta_k^{-1}Q_{\omega_k}(s,\cdot) + \tau_k^{-1}f_{\theta_k}(s,\cdot)), \pi_{\theta_k}(\cdot\given s)\cdot\nu_k(s)\cdot\biggr(1 - \frac{\pi_{\theta_{k+1}}(\cdot\given s)}{\pi_{\theta_k}(\cdot\given s)}\biggl) \biggr\ra \biggr| \notag \\
& \qquad \qquad \qquad  \cdot \Bigl | \frac{\nu^*(s)}{\nu_k(s)} \Bigr | \ud s\notag\\
& \quad = \int_{\cS\times\cA}\bigl|\tau_{k+1}^{-1}f_{\theta_{k+1}}(s, a)- (\beta_k^{-1}Q_{\omega_k}(s,a) + \tau_k^{-1}f_{\theta_k}(s,a))\bigr|\cdot\biggl|1 - \frac{\pi_{\theta_{k+1}}(a\given s)\cdot\nu^*(s)}{\rho_k(a\given s)}\biggr| \ud \rho_k(s,a) \notag\\
&\quad  \leq \EE_{\rho_k}\bigl[\bigl(\tau_{k+1}^{-1}f_{\theta_{k+1}}(s, a)- (\beta_k^{-1}Q_{\omega_k}(s, a) + \tau_k^{-1}f_{\theta_k}(s, a))\bigr)^2\bigr]^{1/2}     \cdot     \EE_{\rho_k}\biggl[\biggl|1 - \frac{\ud(\pi_{\theta_{k+1}}\nu^*)}{\ud\rho_k}\biggr|^2\biggr]^{1/2}\notag\\
&\quad  \leq \sqrt{2} \tau_{k+1}^{-1} \cdot \varepsilon_{k+1, f} \cdot (1+\psi^*_{k}),
\$
where in the last inequality we use the error bound in \eqref{eq:error-kl1} and the definition of $\psi^*_k$ in Assumption \ref{assum:cc}. This finishes the proof of the second inequality.

\end{document}